\documentclass[10pt]{article}

\usepackage{fullpage}
\usepackage{natbib}
\usepackage{algorithm}
\usepackage{algorithmic}
\newcommand{\squeeze}{}
\usepackage{threeparttable}

\usepackage{multirow}
\usepackage{colortbl}
\definecolor{bgcolor}{rgb}{0.8,1,1}
\definecolor{bgcolor2}{rgb}{0.8,1,0.8}

\usepackage{graphicx} %Loading the package
\graphicspath{{Experiments/}}

\newcommand{\eqdef}{\; { := }\;}

\newcommand{\R}{\mathbb{R}}

\newcommand{\ExpBr}[1]{\mathbb{E}\left[#1\right]}
\newcommand{\norm}[1]{\left\|#1\right\|}
\def\<#1,#2>{\langle #1,#2\rangle}

\newcommand{\newalpha}{h}

\usepackage{tcolorbox}
\usepackage{pifont}
\definecolor{mydarkgreen}{RGB}{39,130,67}
\definecolor{mydarkred}{RGB}{192,47,25}
\newcommand{\green}{\color{mydarkgreen}}
\newcommand{\red}{\color{mydarkred}}
\newcommand{\cmark}{\green\ding{51}}%
\newcommand{\xmark}{\red\ding{55}}%

\newcommand{\mA}{\mathbf{A}}
\newcommand{\mB}{\mathbf{B}}

\newcommand{\mH}{\mathbf{H}}
\newcommand{\mI}{\mathbf{I}}
\newcommand{\mU}{\mathbf{U}}

\newcommand{\cC}{{\mathcal{C}}}
\newcommand{\cH}{{\mathcal{H}}}

\usepackage{amsmath,amsfonts,amssymb,amsthm,array}

\usepackage{mdframed} 
\usepackage{thmtools}
\usepackage{textcomp}

\declaretheorem[within=section]{definition}
\declaretheorem[sibling=definition]{theorem}

\declaretheorem[sibling=definition]{assumption}
\declaretheorem[sibling=definition]{corollary}

\declaretheorem[sibling=definition]{lemma}

\usepackage[colorinlistoftodos,bordercolor=orange,backgroundcolor=orange!20,linecolor=orange,textsize=scriptsize]{todonotes}

\usepackage{microtype}
\usepackage{subfigure}
\usepackage{booktabs} % for professional tables

\usepackage{grffile}

\usepackage{hyperref}

\title{\bf Distributed Second Order Methods with Fast Rates  and Compressed Communication}
\author{Rustem Islamov\thanks{King Abdullah University of Science and Technology (KAUST), Thuwal, Saudi Arabia, and Moscow Institute of Physics and Technology (MIPT), Dolgoprudny, Russia. This research was conducted while this author was an intern at KAUST and an undergraduate student at MIPT.} \and Xun Qian\thanks{King Abdullah University of Science and Technology, Thuwal, Saudi Arabia.} \and Peter Richt\'{a}rik\thanks{King Abdullah University of Science and Technology, Thuwal, Saudi Arabia.}}
\date{February 13, 2021}

\begin{document}

\maketitle

\begin{abstract}
We develop several new communication-efficient second-order methods for distributed optimization. Our first method, {\sf NEWTON-STAR}, is a variant of Newton's method from which it inherits its fast local quadratic rate. However, unlike Newton's method, {\sf NEWTON-STAR} enjoys the same per iteration communication cost as gradient descent. While this method is impractical as it relies on the use of certain unknown parameters characterizing the Hessian of the objective function at the optimum,  it serves as the starting point which enables us design practical variants thereof with strong theoretical guarantees. In particular, we design a stochastic sparsification strategy for learning the unknown parameters in an iterative fashion in a communication efficient manner. Applying this strategy to {\sf NEWTON-STAR} leads to our next method, {\sf NEWTON-LEARN}, for which we prove  local linear and superlinear rates independent of the condition number. When applicable, this method can have dramatically superior convergence behavior when compared to state-of-the-art methods. Finally, we develop a globalization strategy using cubic regularization which leads to our next method, {\sf CUBIC-NEWTON-LEARN}, for which we prove global sublinear and linear convergence rates, and a fast superlinear rate. Our results are supported with experimental results on real datasets, and show several orders of magnitude improvement on baseline and state-of-the-art methods in terms of communication complexity. 
\end{abstract}

{
%\footnotesize
\tableofcontents
}

\clearpage
\section{Introduction}

The prevalent paradigm for training modern supervised machine learning models is based on (regularized) empirical risk minimization (ERM) \citep{shai_book}, and the most commonly used optimization methods deployed for solving ERM problems belong to the class of  stochastic first order methods \citep{RobbinsMonro:1951,  Nemirovski-Juditsky-Lan-Shapiro-2009}. Since modern training data sets are very large and are becoming larger every year, it is increasingly harder to get by without relying on modern computing architectures which make efficient use of  distributed computing.  However, in order to develop efficient distributed methods, one has to keep in mind that communication among the different parallel workers (e.g. processors or compute nodes) is typically very slow, and almost invariably forms the main bottleneck in deployed optimization software and systems \citep{bekkerman2011scaling}. For this reason, further advances in the area of communication efficient distributed first order optimization methods for solving ERM problems are highly needed, and research in this area constitutes one  of the most important fundamental endeavors in modern machine learning. Indeed, this research  field is very active, and numerous advances  have been made over the past decade \citep{Seide14, Wen17, Alistarh17, Bernstein18, DIANA, Stich19, Tang19}.
 
  \subsection{Distributed optimization}
 
We consider L2 regularized empirical risk minimization problems of the form
\begin{equation}\label{primal}
\squeeze 
\min \limits_{x\in \mathbb{R}^d} \left[ P(x) \eqdef f(x) + \frac{\lambda}{2}\|x\|^2  \right],
\end{equation}
where $f:\R^d \to \R$ is a smooth\footnote{Function $\phi:\R^d\to \R$ is {\em smooth} if it is differentiable, and has $L_\phi$ Lipschitz gradient: $\|\nabla \phi(x)- \nabla \phi(y)\| \leq L_\phi  \|x-y\|$ for all $x,y\in \R^d$. We say that $L_\phi$ is the {\em smoothness constant} of $\phi$.} convex function of the ``average of averages'' structure
\begin{equation} \label{eq:f_and_f_i} 
 \squeeze f(x) \eqdef\frac{1}{n}\sum \limits_{i=1}^n  f_i(x), \quad f_i(x) \eqdef \frac{1}{m} \sum \limits \limits_{j=1}^{m} f_{ij}(x), \end{equation}
and $\lambda\geq 0$ is a regularization parameter. Here $n$ is the number of parallel workers (nodes), and $m$ is the number of training examples handled by each node\footnote{All our results can be extended in a straightforward way to the more general case when node $i$ contains $m_i$ training examples. We decided to present the results in the special case $m=m_i$ for all $i$ in order to simplify the notation. }.  
The value $f_{ij}(x)$ denotes the loss of the model parameterized by vector $x\in \R^d$ on the $j^{\rm th}$ example owned by the $i^{\rm th}$ node. This example is denoted as $a_{ij} \in \R^d$, and the corresponding loss function is $\varphi_{ij}:\R \to \R$, and hence we have
\begin{equation}\label{eq:f_ij} f_{ij}(x) \eqdef \varphi_{ij}(a_{ij}^\top x).\end{equation}

Thus, $f$ represents the average loss/risk over all $nm$ training datapoints, and problem \eqref{primal} seeks to find the model whose (L2 regularized) empirical risk is minimized. We make the following assumption throughout the paper. 
\begin{assumption}\label{as:general}
Problem (\ref{primal}) has at least one optimal solution $x^*$. For all  $i$ and $j$, the loss function $\varphi_{ij}: \mathbb{R} \to \mathbb{R}$ is $\gamma$-smooth, twice differentiable, and its second derivative $\varphi_{ij}^{\prime\prime} : \R \to \R$ is $\nu$-Lipschitz continuous.
\end{assumption} 

Note that in view of \eqref{eq:f_ij}, the Hessian of $f_{ij}$ at point $x$ is
\begin{equation}
\label{eq:87ybfd0fd}
\mH_{ij}(x) \eqdef 
\nabla^2 f_{ij}(x) = 
\newalpha_{ij}(x) a_{ij}a_{ij}^\top, 
\end{equation}
where
\begin{equation}\label{eq:h_ij-def} \newalpha_{ij}(x) \eqdef \varphi^{\prime\prime}_{ij}(a_{ij}^\top x).\end{equation} In view of Assumption~\ref{as:general}, we have $|\varphi''_{ij}(t)| \leq \gamma$ for all $t\in \R$, and 
\begin{equation}
|\newalpha_{ij}(x) - \newalpha_{ij}(y)| \leq \nu |a_{ij}^\top x - a_{ij}^\top y| \leq \nu \|a_{ij}\|  \|x-y\| \label{eq:alphaijL} 
\end{equation}
for all $x, y \in \R^d$. Let $R\eqdef \max_{ij} \|a_{ij}\|$. The Hessian of $f_i$ is given by
\begin{equation}\label{eq:H_i} \squeeze \mH_i(x) \overset{\eqref{eq:f_and_f_i}  }{=}
%\nabla^2 f_i(x)  
\frac{1}{m}\sum\limits_{j=1}^m \mH_{ij}(x) \overset{\eqref{eq:87ybfd0fd}}{=} \frac{1}{m}\sum\limits\limits_{j=1}^m \newalpha_{ij}(x) a_{ij}a_{ij}^\top,\end{equation}
and the Hessian of $f$ is given by
\begin{equation}\label{eq:H}\squeeze \mH(x) \overset{\eqref{eq:f_and_f_i}  }{=}
%\nabla^2 f_i(x)  
\frac{1}{n}\sum\limits_{i=1}^n \mH_{i}(x) \overset{\eqref{eq:H_i}}{=} \frac{1}{nm}\sum \limits_{i=1}^n   \sum\limits_{j=1}^m \newalpha_{ij}(x) a_{ij}a_{ij}^\top.\end{equation}

\subsection{The curse of the condition number}
 
All first order methods---distributed or not---suffer from a dependence on an appropriately chosen notion of a {\em condition number\footnote{Example: if one wishes to minimize an $L$-smooth $\mu$-strongly convex function  and one cares about the number of gradient type iterations, the appropriate notion of a condition number is $\kappa \eqdef \frac{L}{\mu}$.}}---a number that describes the difficulty of solving the problem by the method at hand. A condition number is a function of the goal we are trying to achieve (e.g., minimize the number of iterations vs minimize the  number of communications), choice of the loss function, structure of the model we are trying to learn, and last but not least, the size and properties of the training data. In fact, most research in this area is motivated by the desire to design methods that would have a {\em reduced} dependence on the condition number. This is the case for many of the tricks heavily studied in the literature, including minibatching~\citep{pegasos2}, importance sampling~\citep{NeedellWard2015, IProx-SDCA}, random reshuffling~\citep{RR}, variance reduction~\citep{schmidt2017minimizing, johnson2013accelerating, proxSVRG, SAGA}, momentum~\citep{SHB-NIPS, SMOMENTUM}, adaptivity \citep{MM2019}, communication compression~\citep{Alistarh17, Bernstein18, DIANA}, and local computation~\citep{COCOA+journal, localSGD-Stich,localSGD-AISTATS2020}.  Research in this area is becoming saturated, and new ideas are needed to make further progress.

\subsection{Newton's method to the rescue?} One of the ideas that undoubtedly crossed everybody's mind is the trivial observation that there {\em is} a very old and simple method which does {\em not} suffer from any conditioning issues: Newton's method. Indeed, when it works, Newton's method has a fast {\em local quadratic convergence rate} which is entirely independent of the condition number of the problem \citep{Beck-book-nonlinear}. While this is a very attractive property, developing  scalable distributed variants of Newton's method that could also {\em provably outperform} gradient based methods remains a largely unsolved problem. To highlight the severity of the issues with extending Newton's method to stochastic and distributed settings common in machine learning, we note that until recently, we did not even have any Newton-type analogue of SGD that could provably work with small minibatch sizes, let alone minibatch size one \citep{SN2019}. In contrast, SGD with minibatch size one is one of the simplest and well understood variants thereof \citep{NeedellWard2015}, and much of modern development in the area of SGD methods is much more sophisticated.  Most variants of Newton's method proposed for deployment in machine learning are heuristics, which is to say that they are not supported with any convergence guarantees, or have convergence guarantees without explicit rates, or suffer from rates that are worse than the rates of first order methods.

\subsection{Contributions summary}

We develop {\em several  new  fundamental Newton-type methods} which we hope make a marked step towards the ultimate goal of developing practically useful and {\em communication efficient distributed second order methods}. Our methods are designed with the explicit goal of supporting {\em efficient communication} in a distributed setting, and in sharp contrast with most recent work, their design was heavily influenced by our desire to equip them with {\em strong convergence guarantees} typical for the classical Newton's method \citep{Wallis1685, Raphson1697} and  cubically regularized Newton's method~\citep{Griewank-cubic-1981, PN2006-cubic}. Our convergence results are summarized in Table~\ref{tbl:rates}.

\begin{table}[t]
	\caption{Summary of algorithms proposed and convergence results proved in this paper.}
	\label{tbl:rates}
	\begin{center}
		  \begin{threeparttable}
		\scriptsize
			\begin{tabular}{|c | ccc | cc|}
				\hline
				                     & \multicolumn{3}{c|}{\bf Convergence} &  &   \\
				  {\bf Method} &  {\bf result} ${}^\dagger$   &  {\bf type}  & {\bf rate}  & \begin{tabular}{c}{ \bf Rate} \\ {\bf independent of the}\\ {\bf condition number?}  \end{tabular} & {\bf Theorem }\\
				\hline
				\begin{tabular}{c} {\sf NEWTON-STAR (NS)} \\ \eqref{eq:Newton-star} \end{tabular} & $r_{k+1}\leq c r_k^2$   & local & quadratic & \cmark &  \ref{th:localquadratic}\\ \hline	
%%%%%%%%%%%%%%%%%%%%					
				\begin{tabular}{c} {\sf MAX-NEWTON (MN)} \\ Algorithm~\ref{alg:maxnewton} \end{tabular} &  $r_{k+1}\leq c r_k^2$      &local & quadratic & \cmark &  \ref{th:maxnewton}		\\	 \hline				
%%%%%%%%%%%%%%%%%%%%			 				
			 \multirow{2}{*}{ \begin{tabular}{c} {\sf NEWTON-LEARN (NL1)}  \\ Algorithm~\ref{alg:NL1} \end{tabular}  } 
			 &  $\Phi^k_1 \leq \theta_1^k \Phi_1^0$ & local & linear & \cmark &  \ref{th:lambda>0}\\ 
			 &  $r_{k+1} \leq c \theta_1^k  r_k$ & local & superlinear & \cmark &  \ref{th:lambda>0}\\ \hline	
%%%%%%%%%%%%%%%%%%%%			 
			 \multirow{2}{*}{ 	\begin{tabular}{c} {\sf NEWTON-LEARN (NL2)} \\ Algorithm~\ref{alg:NL2} \end{tabular}  } 
			 &   $\Phi^k_2 \leq \theta_2^k \Phi_2^0$  & local & linear & \cmark &  \ref{th:general}\\ 
			 &  $r_{k+1} \leq c \theta_2^k  r_k$ & local & superlinear & \cmark &  \ref{th:general}\\  \hline	
%%%%%%%%%%%%%%%%%%%%	
 \multirow{4}{*}{ \begin{tabular}{c} {\sf CUBIC-NEWTON-LEARN (CNL)} \\ Algorithm~\ref{alg:CNL} \end{tabular}  } 
				 & $\Delta_k \leq \frac{c}{k}$   & global & sublinear & \xmark &  \ref{th:concubic}\\ 
				 & $\Delta_k \leq c \exp(-k/c)$ & global & linear & \xmark &  \ref{th:stronglyconcubic}\\ 
				 & $\Phi^k_3 \leq \theta_3^k \Phi^0_3$ & local & linear & \cmark &  \ref{th:cubicsup}\\ 				  
				 & $r_{k+1} \leq c \theta_3^k r_k$ & local & superlinear & \cmark &  \ref{th:cubicsup}\\ 
%%%%%%%%%%%%%%%%%%%%					 																
          \hline
			\end{tabular}
     \end{threeparttable}			
    \begin{tablenotes}
      {\scriptsize      
  \item   Quantities for which we prove convergence:  (i) distance to solution $r_k \eqdef \norm{x^k-x^*}$; (ii) Lyapunov function $\Phi^k_q \eqdef \norm{x^k-x^*}^2 + c_q \sum_{i=1}^n \sum_{j=1}^m ( h_{ij}^k - h_{ij}(x^*) )^2$ for $q=1,2,3$, where $h_{ij}(x^*) = \varphi''_{ij}(a_{ij}^\top x^*)$ (see \eqref{eq:h_ij-def}); (iii) Function value suboptimality  $\Delta_k \eqdef P(x^k) - P(x^*)$
        \item ${}^\dagger$ constant $c$ is possibly different each time it appears in this table. Refer to the precise statements of the theorems for the exact values.
        }
    \end{tablenotes}			
	\end{center}
\end{table}

\begin{itemize}
\item {\bf First new method and its local quadratic convergence.} We first show that if we know the Hessian of the objective function at   the optimal solution, then we can use it instead of the typical Hessian appearing in Newton's method, and the resulting algorithm, which we call {\sf NEWTON-STAR (NS)}, inherits  local quadratic convergence behavior of Newton's method (see Theorem~\ref{th:localquadratic}). In a distributed setting with a central orchestrating sever, each compute node only needs to send the local  gradient  to the server  node, and no matrices need to be sent. While this method is not practically useful, it acts as a stepping stone to our next method, in which these deficiencies are removed. This method is described in Section~\ref{sec:3steps}. A somewhat different method with similar properties, which we call {\sf MAX-NEWTON}\footnote{In fact, this was the first method we developed, in Summer 2020, when we embarked on the research which eventually lead to the results presented in this paper.}, is described in Section~\ref{sec:maxNewton}. 

\item  {\bf Second new method and its local linear and superlinear convergence.} Motivated by the above result, we propose a {\em learning scheme which enables us to learn the Hessian at the optimum iteratively  in a communication efficient manner.} This scheme gives rise to our second new method: {\sf NEWTON-LEARN (NL)}. We analyze this method in two cases: (i) all individual loss functions are convex and $\lambda >0$ (giving rise to the {\sf NL1} method), and (ii) the aggregate loss function $P$ is strongly convex (giving rise to the {\sf NL2} method). Besides the local full gradient, each worker node needs to send additional information to the server node in order to learn the Hessian at the optimum. However, our learning  scheme supports {\em compressed communication} with arbitrary compression level. This level can be chosen so that in each iteration, each node sends an equivalent of a few gradients to the server only. That is, we can achieve $O(d)$ communication complexity in each iteration. In both cases, we prove local linear convergence for a carefully designed Lyapunov function, and  local superlinear convergence for the squared distance to optimum (see Theorems~\ref{th:lambda>0} and \ref{th:general}). Remarkably, all these rates are {\em independent of the condition number.} The {\sf NL1} and {\sf NL2} methods and the associated theory are described in Section~\ref{sec:Newton-learn}.

\item {\bf Third new method and its global convergence.} Next, we equip our learning scheme with a  cubic regularization strategy \citep{Griewank-cubic-1981, PN2006-cubic}, which  leads to a new {\em globally} convergent method: {\sf CUBIC-NEWTON-LEARN (CNL)}. We establish global sublinear and linear convergence (for function values) guarantees for convex and strongly convex problems, respectively. The method can also achieve a fast local linear (for a Lyapunov function) and superlinear (for squared distance to solution) convergence in the strongly convex case.  We describe this method and the associated theory in Section~\ref{sec:CUBIC-NEWTON-LEARN}.

\item {\bf Experiments.} Our theory is corroborated with numerical experiments showing the superiority of our methods to several state-of-the-art benchmarks, including DCGD~\citep{KFJ}, DIANA \citep{DIANA, DIANA2}, ADIANA \citep{ADIANA}, BFGS~\cite{Broyden1967, Fletcher1970, Goldfarb1970, shanno1970conditioning}, and DINGO \citep{DINGO2019}.  Our methods can achieve communication complexity which is {\em several orders of magnitude better} than  competing methods (see Section~\ref{sec:experiments}).
\end{itemize}

\subsection{Related work}

Several distributed Newton-type methods can be found in recent literature. DANE \citep{DANE} is a distributed approximate Newton-type method where each worker node needs to solve a subproblem using the full gradient at each iteration, and the new iterate is the average of these subproblem solutions. The linear convergence of DANE was obtained in the strongly convex case. An inexact DANE method in which the subproblem is solved approximately was proposed and studied by \citet{AIDE}. Moreover, an accelerated version of inexact DANE, called AIDE, was proposed in  \citep{AIDE} by a generic acceleration scheme---catalyst~\citep{lin2015universal}---and an optimal communication complexity can be obtained up to logarithmic factors in specific settings. The DiSCO method, which combines inexact damped Newton method and distributed preconditioned conjugate gradient method, was proposed by \citet{DISCO} and analyzed for self-concordant empirical loss. GIANT \citep{GIANT2018} is a globally improved approximate Newton method which has a better linear convergence rate than first-order methods for quadratic functions, and has local linear-quadratic convergence for strongly convex functions. GIANT and DANE are identical for quadratic programming. The communication cost per iteration of the above methods is $O(d)$. These methods can only achieve linear convergence in the strongly convex case. The comparison of the iteration complexity of the above methods for the ridge regression problem can be found in Table 2 of \citep{GIANT2018}.

\begin{table}[t]
	\caption{Comparison of distributed Newton-type methods. Our methods combine the best of both worlds, and are the only methods we know about which do so: we obtain fast rates independent of the condition number, and allow for $O(d)$ communication per communication round.}
	\label{tablecomparison}
	\begin{center}
		\scriptsize
			  \begin{threeparttable}
			\begin{tabular}{cccccc}
				\toprule
				 {\bf Method}  &  \begin{tabular}{c} {\bf Convergence} \\ {\bf rate} \end{tabular} & \begin{tabular}{c}{ \bf Rate} \\ {\bf independent of the}\\ {\bf condition number?}  \end{tabular} &\begin{tabular}{c}{\bf Communication} \\ {\bf  cost}\\ {\bf per iteration}  \end{tabular}  & \begin{tabular}{c}  {\bf Network} \\ {\bf  structure}  \end{tabular} \\
				\midrule 
				\begin{tabular}{c} DANE \\ \citep{DANE} \end{tabular}   & Linear & \xmark & $O(d)$ & Centralized \\ \hline
				\begin{tabular}{c} DiSCO \\ \citep{DISCO}  \end{tabular}   & Linear &  \xmark & $O(d)$ & Centralized \\ \hline 
				\begin{tabular}{c} AIDE \\ \citep{AIDE}  \end{tabular}   & Linear &  \xmark & $O(d)$ & Centralized \\ \hline 
				\begin{tabular}{c} GIANT \\ \citep{GIANT2018}  \end{tabular}   & Linear &  \xmark & $O(d)$ & Centralized \\ \hline
				\begin{tabular}{c} DINGO \\ \citep{DINGO2019}   \end{tabular}   & Linear &  \xmark & $O(d)$ & Centralized \\ \hline 
				\begin{tabular}{c} DAN \\ \citep{DAN2020}   \end{tabular}   & Local quadratic${}^\dagger$ & \cmark & $O(nd^2)$ & Decentralized \\ \hline 
				\begin{tabular}{c} DAN-LA \\ \citep{DAN2020}   \end{tabular}   & Superlinear & \cmark & $O(nd)$ & Decentralized \\ \hline 
				\rowcolor{bgcolor} 
				\begin{tabular}{c} {\sf NEWTON-STAR} \\ {\bf this work}  \end{tabular}   &Local  quadratic & \cmark & $O(d)$ & Centralized \\ \hline 
				\rowcolor{bgcolor} 
				\begin{tabular}{c} {\sf MAX-NEWTON} \\ {\bf this work}  \end{tabular}   & Local  quadratic & \cmark & $O(d)$ & Centralized \\ \hline 				
				\rowcolor{bgcolor} 
				\begin{tabular}{c} {\sf NEWTON-LEARN} \\ {\bf this work}  \end{tabular}   &Local  superlinear & \cmark & $O(d)$ & Centralized \\ \hline 
				\rowcolor{bgcolor} 
				\begin{tabular}{c} {\sf CUBIC-NEWTON-LEARN} \\ {\bf this work}  \end{tabular}   & Superlinear & \cmark & $O(d)$ & Centralized \\ 
				\bottomrule
			\end{tabular}
			\end{threeparttable}
    \begin{tablenotes}
      {\scriptsize      
  \item   ${}^\dagger$ DAN converges globally, but the quadratic rate is introduced only after $O(L_2/\mu^2)$ steps, where $L_2$ is the Lipschitz constant of the Hessian of $P$, and $\mu$ is the strong convexity parameter of $P$. This is a property it inherits from the recent method of Polyak \citep{L-Newton2019} this method is based on.
        }
    \end{tablenotes}				
	\end{center}
\end{table}

\citet{DINGO2019} proposed a distributed Newton-type method called DINGO   for solving  invex finite-sum problems. Invexity is a special case of non-convexity, which subsumes convexity as a sub-class. A linear convergence rate was obtained for DINGO under certain assumptions using an Armijo-type line search, and at each iteration, several communication rounds are needed assuming two communication rounds for line-search per iteration. The communication cost for each communication round is $O(d)$. The compressed version of DINGO was studied in \citep{Ghosh2020} to reduce the communication cost at each communication round by using the $\delta$-approximate compressor, and the same rate of convergence as DINGO can be obtained by properly choosing the stepsize and hyper-parameters when $\delta$ is a constant.  \citet{DAN2020} proposed two decentralized distributed adaptive Newton methods, called DAN and DAN-LA. DAN combines the distributed selective flooding (DSF) algorithm and Polyak’s adaptive Newton method \citep{polyak2020new}, and enters pure Newton method which has quadratic convergence after about $\frac{2M}{\mu^2} \|\nabla P(x^0)\|$ iterations, where $M$ is the Lipschitz constant of the Hessian of $P$ and $\mu$ is the strongly convex parameter of $P$. DAN-LA, which leverages the low-rank approximation method to reduce the communication cost, has global superlinear convergence. At each iteration, both DAN and DAN-LA need $n-1$ communication rounds, and the communication cost for each communication round is $O(d^2)$ and $O(d)$ respectively. 

We compare the convergence rate and per-iteration communication cost with these Newton-type methods in Table~\ref{tablecomparison}. Note that the first five methods in the table have rates that depend on the condition number of the problem, and as such, do not have the benefits normally attributed to pure Newton's method. Note also that the two prior methods which do have rates independent of the condition number have high cost of communication. Our methods combine the best of both worlds, and are the only methods we know about which do so: we obtain fast rates independent of the condition number, and allow for $O(d)$ communication per communication round.  We were able to achieve this by a complete redesign of how second order methods should work in the distributed setting. Our methods are not simple extensions of existing schemes, and our proofs use novel arguments and techniques.

\section{Three Steps Towards an Efficient Distributed Newton Type Method} \label{sec:3steps}

In order to better explain the  algorithms and results of this paper, we will proceed through several steps in a gradual explanation of the ideas that ultimately lead to our methods. While this is not the process we used to come up with our methods, in retrospect we believe that our methods and results will be understood more easily when seen as having been arrived at in this way. In other words, we have constructed what we believe is a plausible discovery story, one enabling faster and better comprehension. If these ideas seem to follow naturally, it is because we made a conscious effort to make then appear that way. The goal of this paper is to develop communication efficient variants of Newton's method for solving the distributed optimization problem \eqref{primal}.   

%%%%%%%%%%%%%%%%%%%%%%%%%%%%%%%%%%%%%%%%%%%
\subsection{Naive distributed implementation of Newton's method} \label{subsec:Newton}
%%%%%%%%%%%%%%%%%%%%%%%%%%%%%%%%%%%%%%%%%%%

 Newton's method applied to problem \eqref{primal}  performs the iteration
 
 \begin{equation} \label{eq:Newton} x^{k+1} = x^k - \left(\nabla^2 P(x^k)\right)^{-1} \nabla P(x^k) \overset{\eqref{primal}}{=} x^k - \left(\mH(x^k) +\lambda \mI \right)^{-1} \nabla P(x^k).\end{equation}
 
A naive way to implement this method in the {\em parameter server} framework is for each node $i$ to compute the Hessian $\mH_i(x^k)$ and gradient $\nabla f_i(x^k)$ and to communicate these objects to the server. The server then averages the local Hessians $\mH_i(x^k)$ to produce $\mH(x^k)$ via \eqref{eq:H}, and averages the local gradients $\nabla f_i(x^k)$ to produce $\nabla f(x^k)$. The server then adds $\lambda \mI$ to the Hessian, producing $\mH(x^k) + \lambda \mI = \nabla^2 P(x^k)$, adds  $\lambda x^k$ to the gradient, producing $\nabla P(x^k) = \nabla f(x^k) + \lambda x^k$, and subsequently performs the Newton step \eqref{eq:Newton}. The resulting vector $x^{k+1}$ is then broadcasted to the nodes and the process is repeated. 

This implementation mirrors the way GD and many other first order methods are implemented in the parameter server framework. However, unlike in the case of GD, where only $O(d)$ floats need to be sent and received by each node in each iteration, the upstream communication in Newton's method requires $O(d^2)$ floats to be communicated by each worker to the server. Since $d$ is typically very large, this is prohibitive in practice. Moreover, computation of the Newton's step by the parameter server is much more expensive than simple averaging of the gradients performed by gradient type methods. However, in this paper we will not be concerned with the cost of the Newton step itself, as we will assume the server is powerful enough and the network connection is slow enough for this step not to be the main bottleneck of the iteration. Instead, we assume that the communication steps in general, and the $O(d^2)$ communication of the Hessian matrices in particular,  is what forms the bottleneck. The $O(d)$ per node communication cost of the local gradients is negligible,  and so is the $O(d)$ broadcast of the updated model.  

%%%%%%%%%%%%%%%%%%%%%%%%%%%%%%%%%%%%%%%%%%%
\subsection{A better implementation taking advantage of the structure of $\mH_{ij}(x)$} \label{subsec:Naive2}
%%%%%%%%%%%%%%%%%%%%%%%%%%%%%%%%%%%%%%%%%%%

The above naive implementation can be improved in the setting when $m < d^2$ by taking advantage of the explicit structure \eqref{eq:H_i} of the local Hessians as a conic combination of positive semidefinite rank one matrices:
\begin{equation}\label{eq:Newton2}
\squeeze\mH_i(x) = \frac{1}{m}\sum \limits_{j=1}^m \newalpha_{ij}(x) a_{ij}a_{ij}^\top.\end{equation}
Indeed, assuming that the server has direct access to all the training data vectors $a_{ij}\in \R^d$  (these vectors can be sent to the server at the start of the process), node $i$ can send the $m$ coefficients $\newalpha_{i1}(x), \dots, \newalpha_{im}(x)$ to the server instead, and the server is then able to reconstruct the Hessian matrix $\mH_i(x)$ from this information. This way, each node sends $O(m+d)$ floats to the server, which is a substantial improvement on the naive implementation in the regime when $m\ll d^2$.  However, when $m\gg d$, the upstream communication cost is still substantially larger than the $O(d)$ cost of GD.  If the server does not have enough memory to store all vectors $a_{ij}$, this procedure does not work.

%%%%%%%%%%%%%%%%%%%%%%%%%%%%%%%%%%%%%%%%%%%
\subsection{{\sf NEWTON-STAR}: Newton's method with a single Hessian} \label{subsec:Newton-star}
%%%%%%%%%%%%%%%%%%%%%%%%%%%%%%%%%%%%%%%%%%%

We now introduce a simple idea which, surprisingly, enables us to {\em remove the need to iteratively communicate any coefficients altogether.}  Assume, for the sake of argument, that we know the values  $\newalpha_{ij}(x^*)$ for all $i,j$. That is, assume the server has access to coefficients $\newalpha_{ij}(x^*)$ for all $i,j$, and that each node $i$ has access to coefficients $\newalpha_{ij}(x^*)$  for $j=1,\dots,m$, i.e., to the vector
\begin{equation}\label{eq:8f0d8hfd}\newalpha_{i}(x) \eqdef (\newalpha_{i1}(x),\dots, \newalpha_{im}(x)) \in \R^m\end{equation}
for $x=x^*$.  Next, consider the following new Newton-like method which we call {\sf NEWTON-STAR (NS)}, where the ``star'' points to the method's reliance on the knowledge of the optimal solution $x^*$:
\begin{eqnarray}  x^{k+1} &=& x^k - \left(\nabla^2 P(x^*) \right)^{-1} \nabla P(x^k) 
\overset{\eqref{primal}}{=} x^k - \left(\mH(x^*) +\lambda \mI \right)^{-1} \left(\frac{1}{n}\sum_{i=1}^n \nabla f_i(x^k) + \lambda x^k\right).\label{eq:Newton-star}\end{eqnarray}

Since the server knows $\mH(x^*)$, all that the nodes need to communicate are the local gradients $\nabla f_i(x^k)$, which costs $O(d)$ per node. The server then computes $x^{k+1}$, broadcasts it back to the nodes, and the process is repeated. This method has the same per-iteration $O(d)$ communication complexity as GD. However, as we show next, the number of iterations (which is the same as the number of communications) of {\sf NEWTON-STAR} does not depend on the condition number -- a property it borrows from the classical Newton's method. The following theorem says that {\sf NEWTON-STAR} enjoys {\em local quadratic convergence}.

\begin{theorem}[Local quadratic convergence]\label{th:localquadratic}
Let  Assumption \ref{as:general} hold,  and assume that $\mH(x^*) \succeq \mu^* \mI$ for some $\mu^* \geq 0$ (for instance, this holds if $f$ is $\mu^*$-strongly convex) and that $\mu^*+\lambda >0$. Then for any starting point $x^0 \in \R^d$, the iterates of {\sf NEWTON-STAR} for solving problem \eqref{primal}   satisfy the following inequality:
\begin{equation}
\label{eq:NS-rate}
 \squeeze
 \|x^{k+1} - x^*\| \leq \frac{\nu }{2(\mu^*+\lambda)} \cdot \left( \frac{1}{nm} \sum \limits_{i=1}^n \sum \limits_{j=1}^{m} \|a_{ij}\|^3 \right) \cdot \|x^k-x^*\|^2. 
\end{equation}
\end{theorem}
\begin{proof}
{\footnotesize
By the first order optimality conditions, we have 
\begin{equation}\nabla f(x^*) + \lambda x^* = 0.\label{eq:FOC}\end{equation}
Let $\mH_* \eqdef \mH(x^*)$. Since $\mH_* \succeq \mu^* \mI$, we have $\mH_* +\lambda \mI \succeq (\mu^* +\lambda) \mI$, and hence \begin{equation}\label{eq:b98gdf_9898fd_93} \norm{\left(\mH_* + \lambda \mI\right)^{-1} } \leq \frac{1}{\mu^*+\lambda}.\end{equation} 
Using \eqref{eq:FOC} and \eqref{eq:b98gdf_9898fd_93}  and subsequently applying Jensen's inequality  to the function $x\mapsto \norm{x}$, we get 
\begin{eqnarray}
	\|x^{k+1} - x^*\| &=& \left \| x^k - x^*  - \left(\mH_* +\lambda \mI \right)^{-1} \nabla P(x^k)  \right \|  \notag \\
	&\overset{\eqref{eq:FOC}}{=}& \left \| \left(\mH_* +\lambda \mI \right)^{-1} \left[ \left(\mH_* + \lambda \mI \right)(x^k - x^*) - \left(\nabla f(x^k) -\nabla f(x^*)+ \lambda (x^k-x^*)\right)   \right] \right \|  \notag \\ 
	&\overset{\eqref{eq:b98gdf_9898fd_93}}{\leq} &\frac{1}{\mu^*+\lambda} \left \|  \left(\mH_* + \lambda \mI \right) (x^k - x^*) - \left( \nabla f(x^k) - \nabla f(x^*)   \right) - \lambda (x^k -x^*) \right \|  \notag \\ 
	&= &\frac{1}{\mu^*+\lambda} \left \|  \frac{1}{n} \sum_{i=1}^{n}  \mH_i(x^*) (x^k - x^*) -  \frac{1}{n}\sum_{i=1}^n \left( \nabla f_i(x^k) - \nabla f_i(x^*)   \right) \right \|  \notag \\ 
	&\leq & \frac{1}{n (\mu^*+\lambda)} \sum_{i=1}^{n}  \left \|  \mH_i(x^*) (x^k - x^*) -  \left( \nabla f_i(x^k) - \nabla f_i(x^*)   \right) \right \|  \notag \\ 	
	&\overset{\eqref{eq:H_i}+\eqref{eq:f_and_f_i} }{=} &\frac{1}{n (\mu^*+\lambda)} \sum_{i=1}^{n}  \left \| \frac{1}{m} \sum_{j=1}^m \newalpha_{ij}(x^*) a_{ij} a_{ij}^\top (x^k - x^*) -  \frac{1}{m} \sum_{j=1}^m \left( \nabla f_{ij}(x^k) - \nabla f_{ij}(x^*)   \right) \right \|  . \label{eq:rand-opur-9}
	\end{eqnarray}
	
We now use the fundamental theorem of calculus to express difference of gradients $\nabla f_{ij}(x^k) - \nabla f_{ij}(x^*)$ in an integral, obtaining
\begin{equation}\label{eq:difnablaf}
\nabla f_{ij}(x^k) - \nabla f_{ij}(x^*) = \int_{0}^1 \nabla^2 f_{ij}(x^* + \tau (x^k-x^*)) (x^k-x^*) d\tau. 
\end{equation}

Plugging this representation into	\eqref{eq:rand-opur-9} and noting that $ \nabla^2 f_{ij}(x)\equiv \mH_{ij}(x)$ (see \eqref{eq:87ybfd0fd}), we can continue:
\begin{eqnarray}	
	\|x^{k+1} - x^*\| 	& \overset{\eqref{eq:rand-opur-9}+\eqref{eq:difnablaf}}{ \leq} & \frac{1}{n (\mu^*+\lambda)} \sum_{i=1}^n \left\|   \frac{1}{m}  \sum_{j=1}^{m} \left(  \newalpha_{ij}(x^*) a_{ij}a_{ij}^\top (x^k - x^*) -  \int_{0}^1 \mH_{ij} (x^* + \tau(x^k - x^*)) (x^k - x^*)d\tau  \right)   \right\|  \notag \\ 
	&\overset{\eqref{eq:87ybfd0fd}}{=}& \frac{1}{n (\mu^*+\lambda)} \sum_{i=1}^n  \left\| \frac{1}{m} \sum_{j=1}^{m} \left(  \newalpha_{ij}(x^*) a_{ij}a_{ij}^\top (x^k - x^*)  -  \int_{0}^1 \newalpha_{ij} (x^* + \tau(x^k - x^*))a_{ij}a_{ij}^\top (x^k - x^*)d\tau  \right)   \right\| \notag  \\ 
	&=& \frac{1}{n (\mu^*+\lambda)} \sum_{i=1}^n  \left\| \frac{1}{m} \sum_{j=1}^{m}  a_{ij}a_{ij}^\top (x^k - x^*) \left( \newalpha_{ij}(x^*) - \int_{0}^1 \newalpha_{ij}(x^* + \tau(x^k - x^*))  d\tau \right) \right\|  \notag \\ 
	&\leq& \frac{ \|x^k - x^*\|}{ (\mu^*+\lambda)} \frac{1}{nm}\sum_{i=1}^n \sum_{j=1}^{m} \|a_{ij}\|^2 \left|   \int_{0}^1  \newalpha_{ij}(x^*) - \newalpha_{ij}(x^* + \tau(x^k - x^*))  d\tau  \right|. \label{eq:b98f9d8gfd}
\end{eqnarray}
In the last step we have again used Jensen's inequality applied to the function $x\mapsto \norm{x}$, followed by  inequalities of the form $\norm{\mA_{ij} x t_{ij}} \leq \norm{\mA_{ij}} \norm{x} |t_{ij}|$ for $\mA_{ij} = a_{ij}a_{ij}^\top $, $x=x^k-x^*$ and $t_{ij}\in \R$.

From (\ref{eq:alphaijL})  we obtain
$
| \newalpha_{ij}(x^*) - \newalpha_{ij}(x^* + \tau(x^k - x^*))| \leq \nu \tau \|a_{ij}\| \cdot \|x^k - x^*\|, 
$
which implies that 
$$
\left|   \int_{0}^1  \newalpha_{ij}(x^*) - \newalpha_{ij}(x^* + \tau(x^k - x^*))  d\tau  \right| \leq \int_{0}^1 \nu \tau \|a_{ij}\| \cdot \|x^k - x^*\| d\tau = \frac{\nu \|a_{ij}\|}{2} \cdot \|x^k - x^*\|. 
$$

Plugging this into \eqref{eq:b98f9d8gfd}, we finally arrive at \eqref{eq:NS-rate}.
}
\end{proof}

Note that we do not need to assume $f$ to be convex or strongly convex. All we need to assume is positive definiteness of the Hessian at the optimum. This implies local strong convexity, and since our convergence result is local, that is all we need.

{\bf Remark.} Besides  {\sf NEWTON-STAR}, we have designed another new Newton-type method with a local quadratic rate.  This method, which we call {\sf MAX-NEWTON}, is  similar to {\sf NEWTON-STAR} in that it relies on the knowledge of the coefficients $\newalpha_{ij}(x^*)$  for $j=1,\dots,m$. We describe this method in Appendix~\ref{sec:maxNewton}. 

\section{{\sf NEWTON-LEARN}: Learning the Hessian and Local Convergence Theory} \label{sec:Newton-learn}

In Sections~\ref{subsec:Newton}, \ref{subsec:Naive2} and \ref{subsec:Newton-star} we have gone through three steps in our story,  with the first true innovation and contribution of this paper being the {\sf NEWTON-STAR} method and its rate. We have now sufficiently prepared the ground to motivate our first {\em key} contribution: the {\sf NEWTON-LEARN} method. We only outline the basic insights behind this method here; the details are included in Section~\ref{sec:Newton-learn}.

\subsection{The main iteration}
In {\sf NEWTON-LEARN} we maintain a sequence of vectors \begin{equation}\label{eq:h_i^k}h_i^k=(h_{i1}^k,\dots, h_{im}^k) \in \R^m, \nonumber \end{equation} for all $i=1,\dots,n$ throughout the iterations $k\geq 0$  with the goal of {\em learning} the values $\newalpha_{ij}(x^*)$ for all $i,j$. That is, we construct the sequence with the explicit intention to enforce 
\begin{equation}\label{eq:learn} h_{ij}^k \to \newalpha_{ij}(x^*) \qquad \text{as} \qquad k\to \infty.\end{equation}

Using $h_{ij}^k \approx \newalpha_{ij}(x^*)$ we  estimate the Hessian $\mH(x^*)$ via
\begin{equation}\label{eq:Newton2-xx} 
\squeeze \mH(x^*) \approx \mH^k \eqdef  \frac{1}{nm}\sum \limits_{i=1}^n \sum \limits_{j=1}^m h^k_{ij}  a_{ij}a_{ij}^\top,\end{equation}
and then perform a similar iteration to \eqref{eq:Newton-star}:
\begin{equation} \label{eq:Newton-learn} x^{k+1} = x^k - \left(\mH^k + \lambda \mI\right)^{-1} \nabla P(x^k).\end{equation}

\subsection{Learning the coefficients: the idea}
To complete the description of the method, we need to explain how the vectors $h_i^{k+1}$ are updated. This is also the place where we can force the method to be communication efficient. Indeed, if we can design a rule that would enforce the update vectors $h_i^{k+1}-h_i^k$ to be {\em sparse}, say \begin{equation}\label{eq:sparse_update}\|h_i^{k+1}-h_i^k\|_0 \leq s \end{equation} for some $1 \leq s \leq m$ and all $i$ and $k$, then  the upstream communication by each node in each iteration would be of the order $O(s+d)$ only (provided the server has access to all vectors $a_{ij}$)! That is, each node $i$ only needs to communicate $s$ entries of the update vector $h_i^{k+1}-h_i^k$ as the rest are equal to zero, and each node also needs to communicate the $d$ dimensional gradient $\nabla f_i(x^k)$.  Note that $O(s+d)$ can be interpreted as an {\em interpolation} of the $O(m+d)$   per-iteration communication complexity  of the structure-aware implementation of Newton's method from  Section~\ref{subsec:Naive2}, and of the  $O(d)$   per-iteration communication complexity  of {\sf NEWTON-STAR} described in   Section~\ref{subsec:Newton-star}. 

In the more realistic regime when the server does {\em not} have access to the data $\{a_{ij}\}$, we ask each worker $i$ to additionally send the corresponding $s$ vectors $a_{ij}$, which costs extra $O(s d)$ in communication per node. However, when $s=O(1)$, this is the same per-iteration communication effort as that of GD.

We develop two different update rules defining the evolution of the vectors $h_1^k, \dots, h_n^k$. This first rule (see \eqref{eq: big78fd_8h9fd}) applies to the $\lambda>0$ case and leads to our first variant of {\sf NEWTON-LEARN} which we call {\sf NL1} (see Algorithm~\ref{alg:NL1}). This rule and the method are described in Section~\ref{subsec:NL1}. The second rule applies also to the $\lambda=0$ cases and leads to our second variant of {\sf NEWTON-LEARN} which we call {\sf NL2} (see Algorithm~\ref{alg:NL2}). This rule and the method are described in Section~\ref{subsec:NL2}.

\subsection{Outline of fast local convergence theory} We show in Theorem~\ref{th:lambda>0} (covering {\sf NL1}) and Theorem~\ref{th:general} (covering {\sf NL2}) that {\sf NEWTON-LEARN} enjoys a local linear rate wrt a certain Lyapunov function which  involves the term $\|x^k - x^*\|^2$ and also all terms of the form $\|h^k_i - \newalpha_i(x^*)\|^2$. This means that i) the main iteration \eqref{eq:Newton-learn} works, i.e., $x^k$ converges to $x^*$ at a local linear  rate, and that ii) the learning procedure works, and the desired convergence described in \eqref{eq:learn} occurs at a local linear rate. In addition, we also establish a local superlinear rate of $\|x^k - x^*\|^2$. Remarkably,  these rates are {\em independent of any condition number, which is  in sharp contrast with virtually all results on distributed Newton-type methods we are aware of.} 

Moreover, we wish to remark that second order methods are not typically analyzed using a Lyapunov style analysis. Indeed, we only know of a couple works that do so.  First, \citet{SN2019} develop stochastic Newton and cubic Newton methods of a different structure and scope from ours. They do not consider distributed optimization nor communication compression. Second, \citet{RBFGS2020} develop a stochastic BFGS method. Again, their method and scope is very different from ours. Hence, {\em our analysis may be of independent interest as it  adds to the arsenal of theoretical tools which could be used in a more precise analysis of other second order methods.}

\subsection{Compressed learning}

Instead of merely relying on sparse updates for the vectors $h_i^k$ (see \eqref{eq:sparse_update}), we provide a more general communication compression strategy which includes sparsification as a special case \citep{Alistarh17}. We do so via the use of a {\em random compression} operator.  We say that a randomized map $\cC:\R^m\to \R^m$ is  a {\em compression operator (compressor)} if  there exists a constant $\omega \geq 0$ such that the following relations hold for all $x\in \R^m$:
\begin{eqnarray} 
\mathbb{E}[\cC(x) ] &=& x \label{eq:unbiased} \\ 
\mathbb{E}\|\cC(x)\|^2 &\leq &(\omega + 1)\|x\|^2.\label{eq:omega-variance}
\end{eqnarray} 
The identity compressor $\cC(x)\equiv x$ satisfies these relations with $\omega=0$. The larger the variance parameter $\omega$ is allowed to be, the easier it can be to construct a compressor $\cC$ for which the value $\cC(x)$ can be encoded using a small number of bits only. We refer the reader to \citep{biased2020} for a list of several compressors and their properties.

\subsection{{\sf NL1} (learning in the $\lambda > 0$ case)} \label{subsec:NL1}

We now consider  the case where all loss functions $\varphi_{ij}$ are convex and $\lambda >0$.
\begin{assumption}\label{as:learning-1}
Each $\varphi_{ij}$ is convex, $\lambda>0$.
\end{assumption}

When combined with  Assumption~\ref{as:general}, Assumption~\ref{as:learning-1} implies that $\varphi_{ij}''(t)\geq 0$ for all $t$, hence $h_{ij}(x) = \varphi''_{ij}(a_i^\top x)\geq 0$ for all $x\in \R^d$. In particular, $h_{ij}(x^*) \geq 0$ for all $i,j$. Since we wish to construct a sequence of vectors $h_i^k = (h_{i1}^k, \dots, h_{im}^k)\in \R^m$ satisfying $h_{ij}^k \to h_{ij}(x^*)$, it makes sense to try to enforce all vectors in this sequence to have {\em nonnegative entries}: $$h_{ij}^k\geq 0.$$

 Since  $\mH^k$ arises as a linear combination of the rank-one matrices $a_{ij}a_{ij}^\top$ (see \eqref{eq:Newton2-xx}), this makes  $\mH^k$ positive semidefinite, which in turn means that the matrix $\mH^k + \lambda \mI$ appearing in the main iteration \eqref{eq:Newton-learn} of {\sf NEWTON-LEARN} is invertible, and hence the iteration is {\em well defined.}\footnote{Positive definiteness of Hessian estimates is enforced in several popular quasi-Newton methods as well; for instance, in the BFGS method \cite{Broyden1967, Fletcher1970, Goldfarb1970, shanno1970conditioning}. However, quasi-Newton methods operate in a markedly different manner, and the way in which positive definiteness is enforced there is also different.}

\subsubsection{The learning iteration and the {\sf NL1} algorithm}

 In particular, in {\sf NEWTON-LEARN} each node $i$  computes the vector $\newalpha_i(x^k)\in \R^m$  of second derivatives defined in \eqref{eq:8f0d8hfd}, and then performs the update
\begin{equation}\label{eq: big78fd_8h9fd}\boxed{\quad h^{k+1}_i = \left[h^k_i + \eta \cC_i^k(\newalpha_i(x^k) - h^k_i) \right]_+, \quad}\end{equation}
where $\eta>0$ is a learning rate, $\cC_i^k$ is a freshly sampled compressor by node $i$ at iteration $k$. By $[\cdot ]_+$ we denote the positive part function applied element-wise, defined for scalars as follows: $[t ]_+ = t$ if $t\geq 0$ and  $[t ]_+ = 0$ otherwise. 

We remark that it is possible to interpret the learning procedure \eqref{eq: big78fd_8h9fd}  as one step of projected stochastic gradient descent (SGD)  applied to a certain quadratic optimization problem whose unique solution is the vector $\newalpha_i(x^k)$.

The {\sf NL1} algorithm (Algorithm~\ref{alg:NL1}) arises as the combination of the Newton-like update \eqref{eq:Newton-learn} (adjusted to take account of the explicit regularizer) and the learning procedure \eqref{eq: big78fd_8h9fd}. It is easy to see that the update rule for $\mH^k$ in {\sf NL1} is designed to ensure that $\mH^k$ remains of the form  $\mH^k = \frac{1}{n}\sum_{i=1}^n \mH^k_i$, where $\mH^k_i = \frac{1}{m} \sum_{j=1}^{m} h^k_{ij} a_{ij}a_{ij}^\top$. The update rule for $x^k$, performed by the server, is identical to \eqref{eq:Newton-learn},  with an extra provision for the regularizer. The vector $x^{k+1}$ is broadcasted to all workers. Let us comment on how the key communication step is implemented. If the server does not have direct access to the training data vectors $\{a_{ij} \}$, we choose Option~1, otherwise we choose Option~2. A key property of {\sf NL1} is that the server is able to maintain  copies of the learning vectors $h_i^k$ {\em without the need for these vectors to be communicated by the workers to the server.} Indeed, provided the workers and the server agree on the same set of initial vectors $h_1^0, \dots, h_n^0$, update \eqref{eq: big78fd_8h9fd} can be independently computed by  the server as well from its memory state $h_i^k$ and the compressed message $\cC_i^k (\newalpha_i(x^k) - h^k_i)$ received from node $i$. This strategy is reminiscent of the way the key step in the first-order method DIANA \citep{DIANA, DIANA2} is executed. In this sense, {\sf NL1} can be seen as arising from a successful marriage of Newton's method and the DIANA trick.

\begin{algorithm}[tb]
	\caption{{\sf NL1: NEWTON-LEARN} ($\lambda>0$ case)}
	\label{alg:NL1}
\begin{algorithmic}
		\STATE {\bfseries Parameters:} learning rate $\eta>0$ 
		\STATE {\bfseries Initialization:}
		$x^0 \in \R^d$; $h^0_1,\dots, h^0_n \in \R^{m}_{+}$; $\mH^0 =  \frac{1}{nm} \sum \limits_{i=1}^n  \sum\limits_{j=1}^{m} h_{ij}^0 a_{ij}a_{ij}^\top\in \R^{d\times d}$
		\FOR{ $k = 0, 1, 2, \dots$}
		\STATE Broadcast $x^k$ to all workers
		\FOR{each node $i = 1, \dots, n$} 
		\STATE Compute local gradient $\nabla f_i(x^k)$ 
		\STATE $h^{k+1}_i = [h^k_i + \eta \cC_i^k (\newalpha_i(x^k) - h^k_i)]_+$ 
		\STATE Send $\nabla f_i(x^k)$ and $\cC_i^k (\newalpha_i(x^k) - h^k_i)$ to server 
		\STATE {\bf Option 1:} Send $\{a_{ij} : h_{ij}^{k+1} - h_{ij}^k \neq 0\}$ to server
		\STATE {\bf Option 2:} Do nothing if server knows $\{a_{ij} : \forall j\}$
		\ENDFOR
		
		\STATE $x^{k+1} = x^k - \left( \mH^k + \lambda \mI \right)^{-1} \left(  \frac{1}{n} \sum\limits_{i=1}^n \nabla f_i(x^k) + \lambda x^k  \right)$
		\STATE $\mH^{k+1} = \mH^k + \frac{1}{nm} \sum \limits_{i=1}^n  \sum\limits_{j=1}^{m} (h_{ij}^{k+1} - h_{ij}^k) a_{ij}a_{ij}^\top $
		\ENDFOR
\end{algorithmic}
\end{algorithm} 

\subsubsection{Theory}

In our theoretical results we rely on the Lyapunov function 
$$
\squeeze \Phi_1^k \eqdef \|x^{k} - x^*\|^2 + \frac{1}{3mn\eta  \nu^2 R^2} {\cal H}^{k}, \qquad {\cal H}^k \eqdef \sum_{i=1}^n \|h_i^k - \newalpha_i(x^*)\|^2. 
$$ Our main theorem follows.

\begin{theorem}[Convergence of {\sf NL1}]\label{th:lambda>0}
Let  Assumptions~\ref{as:general} and  \ref{as:learning-1} hold. Let $\eta\leq \frac{1}{\omega+1}$ and assume that $\|x^k - x^*\|^2 \leq \frac{\lambda^2}{12\nu^2R^6}$ for all $k\geq 0$. Then for Algorithm \ref{alg:NL1} we have the inequalities 
\begin{eqnarray*}
\squeeze
\mathbb{E}[\Phi_1^k] & \leq & \theta_1^k \Phi_1^0,\\
\squeeze  
\mathbb{E} \left[  \frac{\|x^{k+1} - x^*\|^2}{\|x^k - x^*\|^2 }  \right] & \leq &\theta_1^k  \left(  {6\eta} + \frac{1}{2}  \right) \frac{\nu^2 R^6}{\lambda^2} \Phi_1^0, 
\end{eqnarray*}
where $\theta_1 \eqdef    1 - \min \left\{  \frac{\eta}{2}, \frac{5}{8}  \right\} $.
\end{theorem}

Since the stepsize bound $\eta \leq \frac{1}{\omega+1}$ is independent of the condition number, the linear convergence rates of $\mathbb{E}[\Phi_1^k]$ and $ \mathbb{E} \left[  \frac{\|x^{k+1} - x^*\|^2}{\|x^k - x^*\|^2 }  \right]$ are both {\em independent of the condition number.} Next, we explore under what conditions we can guarantee for all the iterates to stay in a small neighborhood. 

\begin{lemma}\label{lm:initial-1}
Let Assumptions \ref{as:general} and \ref{as:learning-1} hold. Assume $h_{ij}^k$ is a convex combination of $\{  \newalpha_{ij}(x^0), \newalpha_{ij}(x^1), ..., \newalpha_{ij}(x^k)  \}$ for all $i,j$ and $k$. Assume $\|x^0 - x^*\|^2 \leq \frac{\lambda^2}{12\nu^2R^6}$. Then $$
\|x^k - x^*\|^2 \leq \frac{\lambda^2}{12\nu^2R^6} \quad \text{for all} \quad k\geq 0.
$$ 
\end{lemma}

It is easy to verify that if we choose $h_{ij}^0 = \newalpha_{ij}(x^0)$ and use the random sparsification compressor and $\eta \leq \frac{1}{\omega + 1}$, then $h_{ij}^k$ is always a convex combination of $\{  \newalpha_{ij}(x^0), \newalpha_{ij}(x^1), ..., \newalpha_{ij}(x^k)  \}$ for $k\geq 0$. Thus, from Lemma \ref{lm:initial-1} we can guarantee that all the iterates stay in the small neighborhood assumed in Theorem \ref{th:lambda>0} as long as the initial point $x^0$ is in it.

\subsection{{\sf NL2} (learning in the $\lambda \geq 0$ case)} \label{subsec:NL2}

In this subsection, we consider the case where $P$ is $\mu$-strongly convex. Note that we do not require the components $f_{ij}$ to be convex. 
\begin{assumption}\label{as:learning-2}
 $P$ is $\mu$-strongly convex, $|h_{ij}^k| \leq \gamma$ for $k\geq 0$. \end{assumption}

\subsubsection{The learning iteration and the {\sf NL2} algorithm}

As in Algorithm \ref{alg:NL1}, we use a sequence of vectors $\{h_i^k\}_{k\geq 0}$ to learn $\newalpha_i(x^*)$. However, this 
time we rely on a different technique for enforcing positive definiteness of the Hessian estimator.  Since $\lambda$ can be zero, our previous technique aimed at forcing the coefficients $h_{ij}^k$ to be nonnegative will not work. So, we give up on this, 
and instead of \eqref{eq: big78fd_8h9fd} we use the simpler update 
\begin{equation}\label{eq:hikgeneral}
\boxed{\quad h^{k+1}_i = h^k_i + \eta \cC_i^k(\newalpha_i(x^k) - h^k_i). \quad}
\end{equation}
In order to guarantee positive definiteness of  the Hessian estimator $\mH^k + \lambda \mI$ we instead rely on the second part of Assumption~\ref{as:learning-2}. Provided that there exists $\gamma>0$ such that $|h_{ij}^k| \leq \gamma$ for all $i,j$, note that $\frac{\newalpha_{ij}(x^k) + 2\gamma}{h_{ij}^k + 2\gamma}$ is always positive. Noticing that each $a_{ij}a_{ij}^\top$ is positive semidefinite and that $\nabla^2 f(x^k)$ can be expressed in the form
$$
\squeeze \frac{1}{nm}\sum \limits_{i=1}^n \sum\limits_{j=1}^m \left( \frac{\newalpha_{ij}(x^k) + 2\gamma}{h_{ij}^k + 2\gamma} \cdot (h_{ij}^k + 2\gamma) - 2\gamma \right) a_{ij}a_{ij}^\top, 
$$
for  $\beta^k \eqdef \max_{i,j} \frac{\newalpha_{ij}(x^k) + 2\gamma}{h_{ij}^k + 2\gamma}$, we get the inequality
$$
 \frac{1}{nm}\sum_{i=1}^n \sum_{j=1}^m \left[ \beta^k (h_{ij}^k + 2\gamma) - 2\gamma \right] a_{ij}a_{ij}^\top - \nabla^2 f(x^k) 
= \frac{1}{nm}\sum_{i=1}^n \sum_{j=1}^m \left[  \beta^k - \frac{\newalpha_{ij}(x^k) + 2\gamma}{h_{ij}^k + 2\gamma} \right] (h_{ij}^k + 2\gamma) a_{ij}a_{ij}^\top 
\succeq \mathbf{0},
$$
where $\mathbf{0}$ is the $d\times d$ zero matrix, and $\mA \succeq \mB$ means $\mA - \mB$ is positive semidefinite. Thus, if we can maintain the Hessian estimator in the form $$\squeeze \mH^k \eqdef \frac{1}{nm}\sum \limits_{i=1}^n \sum \limits_{j=1}^m \left[ \beta^k (h_{ij}^k + 2\gamma) - 2\gamma \right] a_{ij}a_{ij}^\top,$$ then 
$
\mH^k + \lambda \mI \succeq \nabla^2 f(x^k) + \lambda \mI = \nabla^2P(x^k) \succeq \mu \mI,
$
where the last inequality follows from Assumption~\ref{as:learning-2}. To achieve this goal, we use an auxiliary matrix $\mA^k$, and maintain $\mA^k = \frac{1}{nm} \sum_{i=1}^n \sum_{j=1}^{m}(h_{ij}^k + 2\gamma) a_{ij}a_{ij}^\top$, and $\mH^k = \beta^k \mA^k - 2\gamma \cdot\frac{1}{nm} \sum_{i=1}^n  \sum_{j=1}^{m}a_{ij}a_{ij}^\top$. The rest of Algorithm \ref{alg:NL2} is the same as Algorithm \ref{alg:NL1}.

\begin{algorithm}[tb]
	\caption{{\sf NL2: NEWTON-LEARN} (general case)}
	\label{alg:NL2}
\begin{algorithmic}
		\STATE {\bfseries Parameters:} $\eta>0$; $\gamma>0$ 
		\STATE {\bfseries Initialization:}
		$x^0 \in \R^d$; $h^0_i \in \R^{m}_{+}$; $\mA^0 = \frac{1}{nm} \sum \limits_{i=1}^n \sum\limits_{j=1}^{m} (h_{ij}^0 + 2\gamma)a_{ij}a_{ij}^\top \in \R^{d\times d}$
		\FOR{ $k = 0, 1, 2, \dots$}
		\STATE broadcast $x^k$ to all workers 
		\FOR{ $i = 1, \dots, n$} 
		\STATE  Compute local gradient $\nabla f_i(x^k)$ 
		\STATE $h^{k+1}_i = h^k_i + \eta \cC_i^k(\newalpha_i(x^k) - h^k_i)$ 
		\STATE $\beta_i^k = \max_{j\in [m]} \frac{\newalpha_{ij}(x^k) + 2\gamma}{h_{ij}^k + 2\gamma}$
		\STATE Send $\nabla f_i(x^k)$, $\beta_i^k$, and $\eta \cC_i^k(\newalpha_i(x^k) - h^k_i)$ to server 
		\STATE {\bf Option 1:} Send $\{a_{ij} : h_{ij}^{k+1} - h_{ij}^k \neq 0\}$ to server
		\STATE {\bf Option 2:} Do nothing if server knows $\{a_{ij} : \forall j\}$
		\ENDFOR
		\STATE $\beta^k = \max_{i} \{  \beta_i^k  \}$
		\STATE $\mH^k = \beta^k \mA^k - 2\gamma \cdot\frac{1}{nm} \sum \limits_{i=1}^n  \sum \limits_{j=1}^{m}a_{ij}a_{ij}^\top \in \R^{d\times d}$
		\STATE $x^{k+1} = x^k - \left( \mH^k + \lambda \mI \right)^{-1} \left(  \frac{1}{n} \sum \limits_{i=1}^n \nabla f_i(x^k) + \lambda x^k  \right)$
		\STATE $\mA^{k+1} = \mA^k + \frac{1}{nm} \sum \limits_{i=1}^n  \sum \limits_{j=1}^{m} (\eta \cC_i^k(\newalpha_i(x^k) - h^k_i))_j a_{ij}a_{ij}^\top $
		\ENDFOR
	\end{algorithmic}
\end{algorithm} 

\subsubsection{Theory}

Our analysis of {\sf NL2} relies on the  Lyapunov function 
$$\squeeze
\Phi_2^k \eqdef \|x^{k} - x^*\|^2 + \frac{1}{3mn\eta  \nu^2 R^2} {\cal H}^{k}, \qquad {\cal H}^k \eqdef  \sum_{i=1}^n \|h_i^k - \newalpha_i(x^*)\|^2.
$$
We now present our main convergence result for {\sf NL2}. 

\begin{theorem}[Convergence of {\sf NL2}]\label{th:general}
	Let Assumptions \ref{as:general} and \ref{as:learning-2} hold. Assume $\eta\leq \frac{1}{\omega+1}$ and $\|x^k - x^*\|^2 \leq \frac{\mu^2}{432mn \nu^2R^6}$ for all $k\geq 0$. Then for Algorithm \ref{alg:NL2} we have the inequalities 
\begin{eqnarray*}
\squeeze
	\mathbb{E}[\Phi_2^k] &\leq &  \theta_2^k \Phi_2^0 ,\notag\\ 
	\squeeze  \mathbb{E} \left[  \frac{\|x^{k+1} - x^*\|^2}{\|x^k - x^*\|^2 }  \right] 
		 &\leq & \theta_2^k  \left(  {3mn\eta} + 1  \right) \frac{72\nu^2 R^6}{\mu^2} \Phi_2^0, 
\end{eqnarray*}
where $\theta_2 \eqdef  1 - \min \left\{  \frac{\eta}{2}, \frac{1}{2}  \right\}  $.
\end{theorem}

As before, we give sufficient conditions guaranteeing that the iterates stay in a small neighborhood of the optimum.
\begin{lemma}\label{lm:initial-2}
	Let Assumptions \ref{as:general} and \ref{as:learning-2} hold. Assume $h_{ij}^k$ is a convex combination of $\{  \newalpha_{ij}(x^0), \newalpha_{ij}(x^1), ..., \newalpha_{ij}(x^k)  \}$ for all $i,j$ and $k$. Assume $\|x^0 - x^*\|^2 \leq \frac{\mu^2}{432mn \nu^2R^6}$. Then 	$$
	\|x^k - x^*\|^2 \leq \frac{\mu^2}{432mn \nu^2R^6} \quad \text{for all} \quad k\geq 0. $$
\end{lemma}

If we choose $h_{ij}^0 = \newalpha_{ij}(x^0)$, use a random  compressor with variance $\omega$, and choose stepsize $\eta \leq \frac{1}{\omega + 1}$, then $h_{ij}^k$ is a convex combination of $\{  \newalpha_{ij}(x^0), \newalpha_{ij}(x^1), ..., \newalpha_{ij}(x^k)  \}$ for all $k\geq 0$. Thus, via Lemma \ref{lm:initial-2} we can guarantee all the iterates to be in the small neighborhood required by Theorem \ref{th:general} as long as the initial point $x^0$ is in it.

\section{{\sf CUBIC-NEWTON-LEARN}: Global Convergence Theory via Cubic Regularization} \label{sec:CUBIC-NEWTON-LEARN}

In this section, we develop the {\sf CUBIC-NEWTON-LEARN} method which can obtain {\em global} convergence and superlinear convergence.  We make the following assumption throughout this section.

\begin{assumption}\label{as:cubic}
	Assume $|h_{ij}^k| \leq \gamma$ for $k\geq 0$, and $\|a_{ij}\| \leq R$ for all $i,j$. Assume $R_x \eqdef \sup_{x\in \R^d} \{  \|x-x^*\| : P(x) \leq P(x^0)  \} < + \infty$. 
\end{assumption}

Recall that $\mH(x)$ is the Hessian of $f$ at point $x$. For any $x, y \in \R^d$, we have 
\begin{eqnarray*}
\|\mH(x) - \mH(y)\|	&= & \left\| \frac{1}{n} \sum \limits_{i=1}^n \frac{1}{m} \sum\limits_{j=1}^{m}( \newalpha_{ij}(x) - \newalpha_{ij}(y) )a_{ij}a_{ij}^\top  \right\| \\ 
&\leq & \frac{1}{n} \sum\limits_{i=1}^n \frac{1}{m} \sum\limits_{j=1}^{m} \|a_{ij}\|^2 |\newalpha_{ij}(x) - \newalpha_{ij}(y)| \\ 
&\overset{(\ref{eq:alphaijL})}{\leq} & \frac{1}{n} \sum\limits_{i=1}^n \frac{1}{m} \sum\limits_{j=1}^{m} \nu \|a_{ij}\|^3 \|x-y\| \quad \leq \quad \nu R^3 \|x-y\|. 
\end{eqnarray*}

Let $M \eqdef \nu R^3$. Then $f$ has $M$-Lipschitz continuous Hessian, and so does $P$. 

\subsection{{\sf CNL}: the algorithm}

  At a high level, our {\sf CNL} method (Algorithm \ref{alg:CNL}) can be seen as the ``combination'' of  Algorithm \ref{alg:NL2} ({\sf NL2}) and the cubic Newton method of \citet{Griewank-cubic-1981, PN2006-cubic}. We use the same learning procedure for $h_i^k$, and the same construction for obtaining a Hessian estimator $\mH^k$ for $f$ as those used in Algorithm \ref{alg:NL2}. In particular, since $|h_{ij}^k| \leq \gamma$, same as Algorithm \ref{alg:NL2}, we have 
$$
\mA^k = \frac{1}{nm} \sum_{i=1}^n \sum_{j=1}^{m}(h_{ij}^k + 2\gamma) a_{ij}a_{ij}^\top
$$ 
and  $$\mH^k = \max_{ij} \frac{\newalpha_{ij}(x^k) + 2\gamma}{h_{ij}^k + 2\gamma} \frac{1}{n} \sum_{i=1}^n \frac{1}{m}\sum_{j=1}^{m}(h_{ij}^k + 2\gamma) a_{ij}a_{ij}^\top - 2\gamma \cdot \frac{1}{n} \sum_{i=1}^n \frac{1}{m} \sum_{j=1}^{m}a_{ij}a_{ij}^\top \succeq \nabla^2 f(x^k).$$ Thus, $\mH^k + \lambda \mI \succeq \nabla^2 f(x^k) + \lambda \mI = \nabla^2 P(x^k)$. If we let $$\mH_i^k \eqdef \beta^k \frac{1}{m}\sum_{j=1}^{m}(h_{ij}^k + 2\gamma) a_{ij}a_{ij}^\top - 2\gamma \cdot \frac{1}{m} \sum_{j=1}^{m}a_{ij}a_{ij}^\top,$$ then $\mH^k = \frac{1}{n} \sum_{i=1}^n \mH^k_i$.  Since $P$ has $M$-Lipschitz continuous Hessian, we have $$P(x^k+s) \leq P(x^k) + T(x^k, s), \qquad \text{for all} \qquad x\in \R^d,$$ where 
$$
T(x^k, s) \eqdef \langle \nabla f(x^k) + \lambda x^k, s \rangle + \frac{1}{2} \langle (\mH^k + \lambda \mI)s, s \rangle + \frac{M}{6}\|s\|^3. 
$$
Finally, the search direction $s^k$ is obtained by minimizing $T(x^k, s)$. This subproblem can be solved by computing the eigenvalue decomposition and then solving a one-dimensional nonlinear equation \citep{hanzely2020stochastic,gould2010solving}. The details can be found in Section \ref{sec:subproblem}.

\begin{algorithm}[tb]
	\caption{{\sf CNL: CUBIC-NEWTON-LEARN}}
	\label{alg:CNL}
	\begin{algorithmic}
		\STATE {\bfseries Parameters:} $\eta>0$; $\gamma>0$
		\STATE {\bfseries Initialization:}
		$x^0 \in \R^d$; $h^0_i \in \R^{m}$; $\mA^0 = \frac{1}{nm} \sum \limits_{i=1}^n  \sum\limits_{j=1}^{m} (h_{ij}^0 + 2\gamma)a_{ij}a_{ij}^\top \in \R^{d\times d}$
		\FOR{ $k = 0, 1, 2, \dots$}
		\STATE broadcast $x^k$ to all workers 
		\FOR{ $i = 1, \dots, n$} 
		\STATE Compute local gradient $\nabla f_i(x^k)$ 
		\STATE Update the vector of coefficients $h^{k+1}_i = h^k_i + \eta \cC_i^k(\newalpha_i(x^k) - h^k_i)$ 
		\STATE $\beta_i^k = \max_{j\in [m]} \frac{\newalpha_{ij}(x^k) + 2\gamma}{h_{ij}^k + 2\gamma}$
		\STATE Send $\nabla f_i(x^k)$, $\beta_i^k$, and $\cC_i^k(\newalpha_i(x^k) - h^k_i)$ to server 
		\STATE {\bf Option 1:} Send $\{a_{ij} : h_{ij}^{k+1} - h_{ij}^k \neq 0\}$ to server
		\STATE {\bf Option 2:} Do nothing if server knows $\{a_{ij} : \forall j\}$
		%\STATE {\bf Option:} Send $a_{ij}$ to the server node for all $j\in [m]$ such that $(\cC(\newalpha_i(x^k) - h^k_i))_j \neq 0$
		\ENDFOR
		\STATE $\beta^k = \max_{i} \{  \beta_i^k  \}$
		\STATE $\mH^k = \beta^k \mA^k - 2\gamma \cdot \frac{1}{nm} \sum\limits_{i=1}^n  \sum\limits_{j=1}^{m}a_{ij}a_{ij}^\top \in \R^{d\times d}$
		\STATE $s^k = {\rm argmin}_{s\in \R^d} T(x^k, s)$
		\STATE $x^{k+1} = x^k + s^k$
		\STATE $\mA^{k+1} = \mA^k + \frac{1}{nm} \sum \limits_{i=1}^n \sum\limits_{j=1}^{m} (\eta \cC_i^k(\newalpha_i(x^k) - h^k_i))_j a_{ij}a_{ij}^\top $
		\ENDFOR
	\end{algorithmic}
\end{algorithm}

\subsection{Global convergence}

The proof of global convergence of {\sf CUBIC-NEWTON-LEARN} was inspired by that in the stochastic subspace cubic Newton method of \citet{hanzely2020stochastic}. First, we need the following technical lemma, which show the progress for each step of Algorithm \ref{alg:CNL}. 

\begin{lemma}\label{lm:global}
	Let Assumption \ref{as:general} and Assumption \ref{as:cubic} hold. Then for all $k\geq 0$ and any $y \in \R^d$, we have 
	$$
	P(x^{k+1}) \leq P(y) + 9\gamma R^2 \|y-x^k\|^2 + \frac{M}{3}\|y-x^k\|^3. 
	$$
\end{lemma}

We first consider the general convex case. 

\begin{theorem}[Global convergence of {\sf CNL}: convex case]\label{th:concubic}
	Let Assumption \ref{as:general} and Assumption \ref{as:cubic} hold and further assume that $P$ is convex. Then for all $k\geq 0$ we have 
	\[
	P(x^k) - P(x^*) 
	\leq \frac{81\gamma R^2 R_x^2}{k} + \frac{9MR_x^3}{k^2} + \frac{4[P(x^0) - P(x^*)]}{k^3}.
	\]  In particular,  
	$$
	k = O\left( \frac{81\gamma R^2 R_x^2}{\epsilon}  + \sqrt{\frac{MR_x^3}{\epsilon}} + \left(  \frac{P(x^0) - P(x^*)}{\epsilon}  \right)^{1/3}   \right) \qquad \Rightarrow \qquad P(x^k) - P(x^*) \leq \epsilon.
	$$
\end{theorem}

In the strongly convex case, we have the following linear convergence result. 

\begin{theorem}[Global convergence of {\sf CNL}: strongly convex case]\label{th:stronglyconcubic}
	Let Assumption \ref{as:general} and Assumption \ref{as:cubic} hold and assume that $P$ is $\mu$-strongly convex. Then 	$$
	k = O\left( \left(  \frac{\gamma R^2}{\mu} + \sqrt{\frac{MR_x}{\mu}} + 1 \right) \log \left(  \frac{P(x^0) - P(x^*)}{\epsilon}  \right)   \right) \qquad \Rightarrow \qquad P(x^k) - P(x^*) \leq \epsilon.
	$$
\end{theorem}

\subsection{Superlinear convergence}

Let us define the Lyapunov function 
$$
\Phi_3^k \eqdef \|x^{k} - x^*\|^2 + \frac{4}{9mn\eta  \nu^2 R^2} {\cal H}^{k}, \qquad {\cal H}^k \eqdef  \sum_{i=1}^n \|h_i^k - \newalpha_i(x^*)\|^2.
$$
We now show that this Lyapunov function converges locally linearly, and that the squared distance to the optimum converges locally superlinearly. Both rates are independent of any condition number. Note that this also means that our learning procedure works, i.e., the estimates $h_i^k$ converge to the values $\newalpha_i(x^*)$, which also means that $\mH^k \to \mH(x^*)$. 

\begin{theorem}[Local superlinear convergecne of {\sf CNL}: strongly convex case]\label{th:cubicsup}
	Let Assumption \ref{as:general} and Assumption \ref{as:cubic} hold. Assume $P$ is $\mu$-strongly convex, and that $\|x^k - x^*\|^2 \leq \frac{\mu^2}{432m n \nu^2R^6}$ for all $k\geq 0$. Then for Algorithm \ref{alg:CNL}, we have the relations
\begin{eqnarray*}
	\mathbb{E} \left[\Phi_3^k \right] & \leq &  \theta_3^k \Phi_3^0 \\
		\mathbb{E} \left[  \frac{\|x^{k+1} - x^*\|^2}{\|x^k - x^*\|^2 }  \right]	&	\leq & \theta_3^k  \left(  {3m n\eta} + 1  \right) \frac{100\nu^2 R^6}{\mu^2} \Phi_3^0, 
	\end{eqnarray*}
	where $\theta_3 \eqdef  1 - \min \left\{  \frac{\eta}{2}, \frac{1}{4}  \right\}$.
\end{theorem}

Theorem \ref{th:stronglyconcubic}, Theorem \ref{th:cubicsup}, and the fact that $\frac{\mu}{2} \|x-x^*\|^2 \leq P(x) - P(x^*)$ when $P$ is $\mu$-strongly convex, imply the following corollary. 

\begin{corollary}\label{co:supcubic}
	Let Assumption \ref{as:general} and Assumption \ref{as:cubic} hold. Assume $P$ is $\mu$-strongly convex. Then 
	$$
	\lim\limits_{k\to +\infty} \mathbb{E} \left[  \frac{\|x^{k+1} - x^*\|^2}{\|x^k - x^*\|^2 }  \right] =0. 
	$$
\end{corollary}

\section{Experiments}
\label{sec:experiments}

We now study the empirical performance of our second order methods {\sf NL1}, {\sf NL2} and {\sf CNL}, and compare them with relevant benchmarks and with state-of-the-art methods. We  test on  the regularized logistic regression problem
\begin{eqnarray*}\squeeze
	\min\limits_{x\in\R^d}\left\{\frac{1}{n}\sum\limits_{i=1}^n\frac{1}{m}\sum\limits_{j=1}^m\log\left(1+\exp(-b_{ij}a_{ij}^\top x)\right) + \frac{\lambda}{2}\|x\|^2\right\},
\end{eqnarray*}
where $\{a_{ij}, b_{ij}\}_{j\in[m]}$ are data samples at the $i$-th node.

\subsection{Data sets and parameter settings} 

In our experiments we use five standard datasets from the LIBSVM library:  {\tt a2a}, {\tt a7a}, {\tt a9a},  {\tt w8a} and {\tt phishing}. Besides, we generated an artificial dataset {\tt artificial} as follows: each of the $d$ elements of the data vector $a_{ij}\in \R^d$ was sampled from the normal distribution $\mathcal{N} (10, 10).$ The corresponding label $b_{ij}$ was sampled uniformly at random from $\{-1, 1\}$. We partitioned each dataset across several nodes (selection of $n$)  in order to capture a variety of scenarios. See Table~\ref{table1} for more details on all the datasets and the choice of $n$. 

In all experiments we use the theoretical parameters (e.g., stepsizes) for all the three algorithms: vanilla Distributed Compressed Gradient Descent (DCGD) \citep{KFJ}, DIANA \citep{DIANA}, and ADIANA \citep{ADIANA}. 

\begin{table}[h]
	\caption{Data sets used in the experiments, and the number of worker nodes $n$ used in each case.}
	\label{table1}
	\begin{center}
		\begin{tabular}{| l | r | r | r |}
			\toprule
			{\bf Data set}  & \bf \# workers $n$ &{\bf \# data points} $(=nm)$ & \bf\# features $d$ \\
			\midrule 
			{\tt a2a} & $15$ & $2~265$ & $123$\\ \hline
			{\tt a7a} & $100$ & $16~100$ & $123$\\ \hline
			{\tt a9a} & $80$ & $32~560$ & $123$\\ \hline
			{\tt w8a} & $142$ & $49~700$ & $300$\\ \hline
			{\tt phishing} & $100$ & $11~000$ & $68$\\ \hline
			{\tt artificial} & $100$ & $1~000$ & $200$\\
			\bottomrule
		\end{tabular}
	\end{center}
\end{table}

As the initial approximation of the Hessian in BFGS \citep{Broyden1967, Fletcher1970, Goldfarb1970, shanno1970conditioning}, we use $\mH^0 = \nabla^2P(x^0)$, and the stepsize is $1$. We set the same constants in DINGO \citep{DINGO2019} as they did: $\theta=10^{-4}, \phi=10^{-6}, \rho=10^{-4},$ and use backtracking line search for DINGO to select the largest stepsize in $\{1, 2^{-1}, 2^{-2}, 2^{-4},\dots, 2^{-10}\}$. We conduct experiments for three values of the regularization parameter $\lambda$: $10^{-3}, 10^{-4}, 10^{-5}$.  
In the figures we plot the relation of the optimality gap $P(x^k) - P(x^*)$ and the number of accumulated transmitted bits or th number of iterations. The optimal value $P(x^*)$ in each case is the function value at the $20$-th iterate of standard Newton's method. In all plots, ``communicated bits'' refers to the total number of bits that all nodes send to the server. We adopt the realistic setting where the server does not have access to the local data ({\bf Option 1}).

\subsection{Compression operators} For the first order methods we use three compression operators: random sparsification \citep{stich2018sparsified}, random dithering \citep{Alistarh17}, and natural compression \citep{Cnat} (all defined below).
For random-$r$ sparsification, the number of communicated  bits per iteration is $32r+\log_2{\binom{d}{r}}$, and we choose $r = d/4$. For random dithering, we choose $s = \sqrt{d}$, which means the
number of communicated bits per iteration is $2.8d + 32$. For natural compression, the number of communicated bits per iteration is $9d$ bits. 

For {\sf NL1} and {\sf NL2} we use the random-$r$ sparsification operator with a selection of values of $r$.  For {\sf CNL} we use the random sparsification operator $\cC_p$ (with $p=1/20$)  induced by the random-$r$ compressor with $r=1$. This compressor is also defined below.

\subsubsection{Random sparsification} The random sparsification compressor \citep{stich2018sparsified}, denoted random-$r$, is a randomized mapping $\cC:\R^m\to \R$ defined as $$\cC(x) \eqdef \frac{m}{r} \cdot \xi \circ x$$ where  $\xi \in \R^m$ is a random vector distributed uniformly at random on the discrete set $\{  y \in \{0, 1 \}^m : \|y\|_0 = r  \}$, where $\|y\|_0\eqdef \{i \;|\; y_i \neq 0\}$ and  $\circ$ is the Hadamard product. The variance  parameter associate with this compressor is $\omega = \frac{m}{r} - 1$. 

\subsubsection{Random dithering} The random dithering compressor \citep{Alistarh17, DIANA2} with $s$ levels is defined via 
$$\cC(x)\eqdef \text{sign}(x)\cdot \|x\|_q\cdot \frac{1}{s} \cdot \xi_s,$$ 
where $\|x\|_q\eqdef \left(\sum_i |x_i|^q\right)^{1/q}$ and $\xi_s\in\R^m$ is a random vector with 
$i$-th element  defined as
$$\xi_s(i)\eqdef \begin{cases}
	l+1 &\text {with probability } \frac{|x_i|}{\|x\|_q}s- l \\
	l &\text {otherwise}
\end{cases}.$$ Here, $l$ satisfies $\frac{|x_i|}{\|x\|_q}\in [\frac{l}{s}, \frac{l+1}{s}]$ and $s \in \mathbb{N}_+$ denotes the levels of the rounding. The variance parameter of this compressor is $\omega \leq 2 + \frac{m^{1/2} + m^{1/q}}{s}$ \citep{DIANA2}. For $q=2$, one can get the improved bound $\omega \leq \min\{  \frac{m}{s^2}, \frac{\sqrt{m}}{s}  \}$ \citep{Alistarh17}.  

\subsubsection{Natural compression} The natural compression  \citep{Cnat} operator $\cC_{\rm nat}:\R^m \to \R$ is obtained by applying the random mapping $\cC_{}:\R\to \R$, defined next, to each coordinate of $x$ independently. We define $\cC(0)=0$ and for $t\neq 0$, we let
$$
\cC(t) \eqdef \begin{cases}
	{\rm sign}(t) \cdot 2^{\lfloor \log_2|t| \rfloor } &\text {with probability } \quad p(t) \eqdef \frac{2^{\lceil \log_2|t| \rceil } - |t|}{2^{\lfloor \log_2|t| \rfloor }} \\
	{\rm sign}(t) \cdot 2^{\lceil \log_2|t| \rceil } &\text {with probability } \quad 1-p(t) 
\end{cases}
$$
The variance parameter of natural compression is $\omega = \frac{1}{8}$.

\subsubsection{Bernoulli compressor} \label{sec:Bernoulli}
A variant of any compression operator $\cC:\R^m \to \R$ can be constructed as follows:
\begin{equation}\label{eq:Qp}
	\cC_p(x) \eqdef \left\{ \begin{array}{rl}
		\frac{1}{p}\cC(x) & \mbox{ with probability $p$} \\
		0 &\mbox{ otherwise}
	\end{array}, \right.
\end{equation}
where $p\in (0, 1]$ is a probability parameter. It is easy to verify that $\cC_p$ is still a compression operator with variance parameter $\omega_p \eqdef \frac{\omega+1}{p} - 1$, where $\omega$ is the variance parameter of the underlying compressor $\cC$.

\subsection{Behavior of {\sf NL1} and {\sf NL2}}

Before we compare our methods {\sf NL1} and {\sf NL2} with competing baselines,  we investigate how is their performance affected by the choice of the  sparsification parameter $r$ defining the random-$r$ sparsification operator  $\cC$. Likewise, we vary the probability parameter $p$ defining the induced Bernoulli compressor  $\cC_p$.

According to the results summarized in  Figure~\ref{exp:NL1_NL2}, the best performance of {\sf NL1} is obtained for $r=1$ and $p=1$. However, for {\sf NL2} these parameters are $r=1$ and $p=1/20$. We will use these  parameter settings for {\sf NL1} and {\sf NL2} in our subsequent experiments where we compare our methods with several baselines and state-of-the-art methods.

\begin{figure}[ht]
	\begin{center}
		\centerline{\begin{tabular}{cccc}
				\includegraphics[width = 0.23 \textwidth]{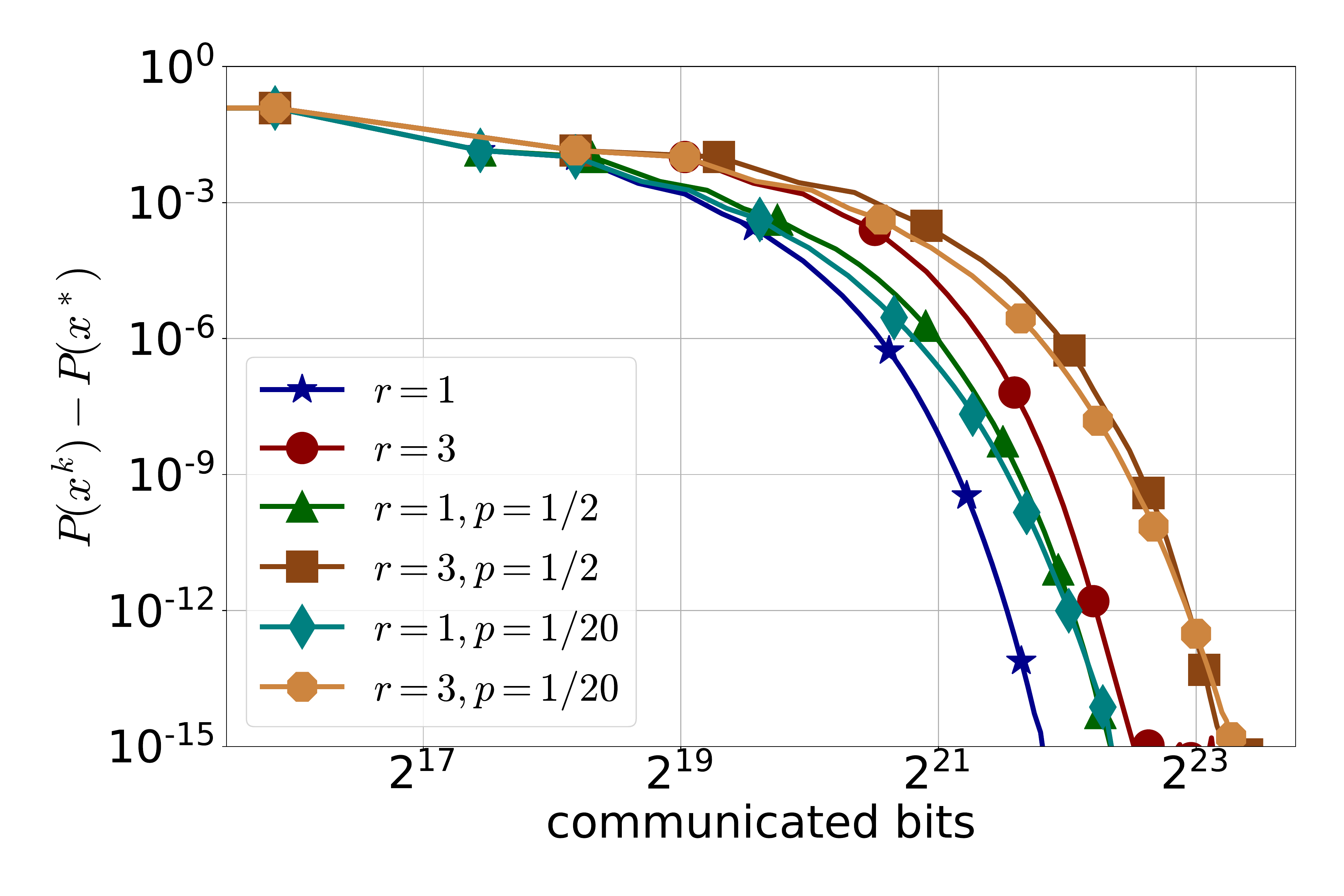}&
				\includegraphics[width = 0.23 \textwidth]{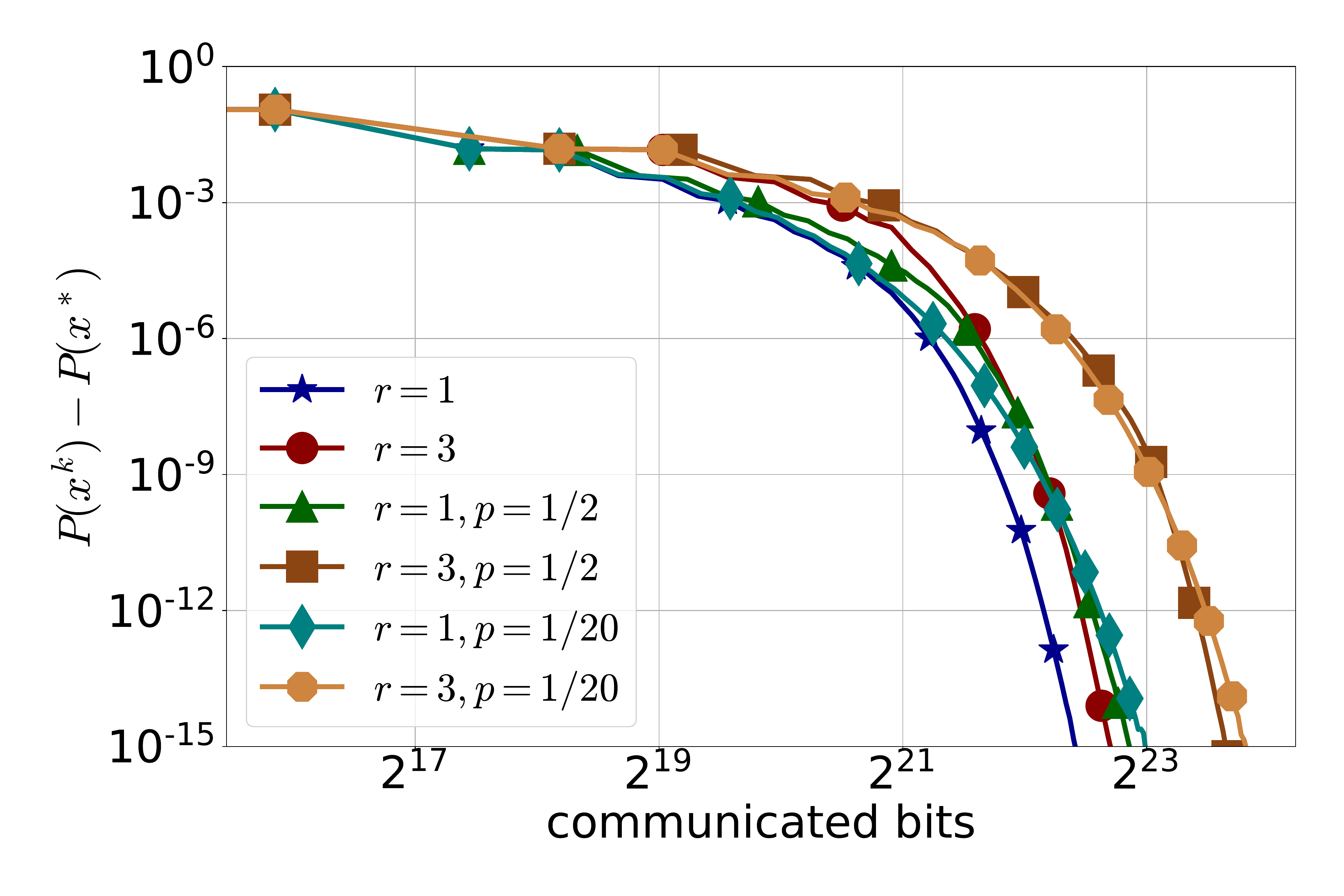}&
				\includegraphics[width = 0.23 \textwidth]{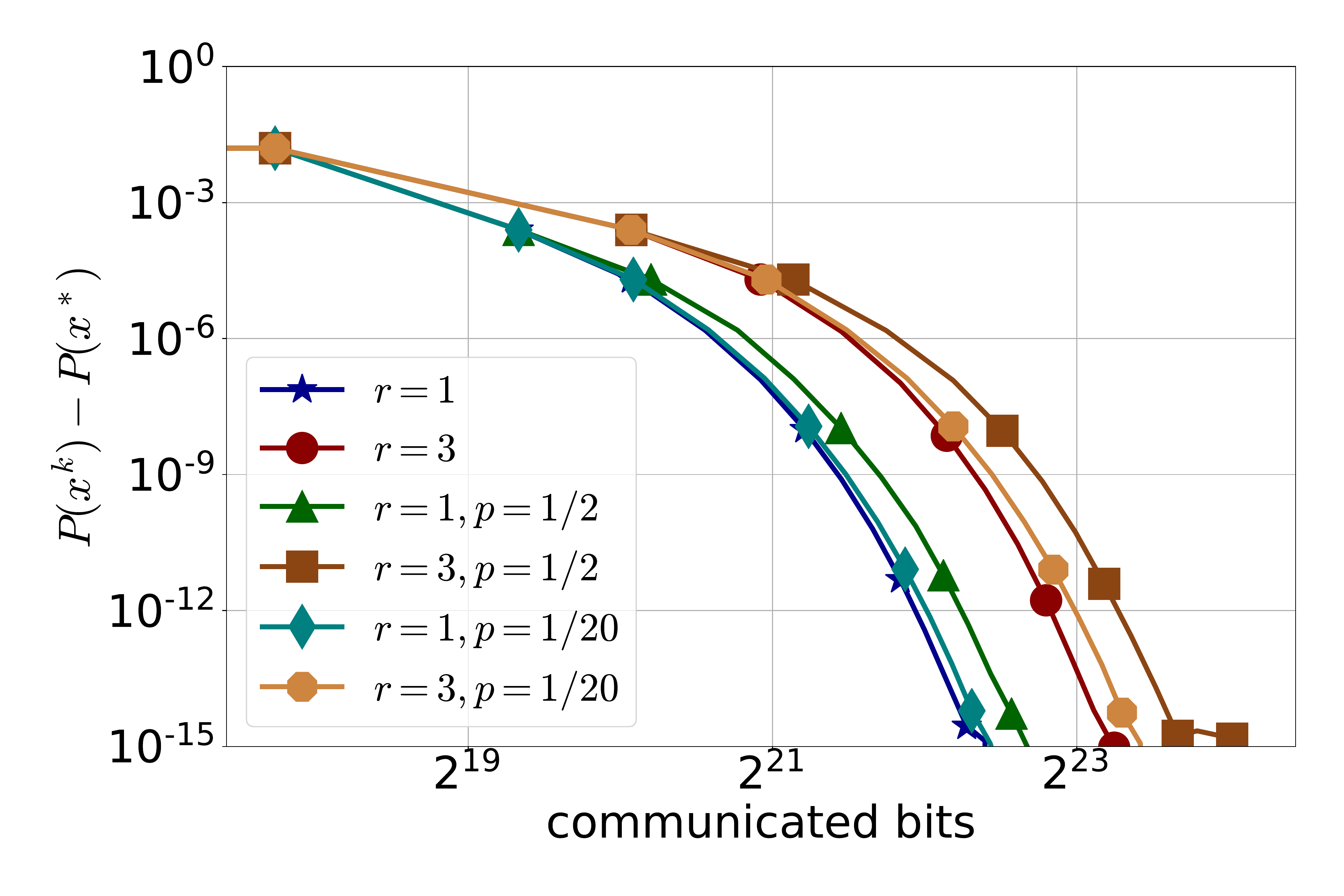}&
				\includegraphics[width = 0.23 \textwidth]{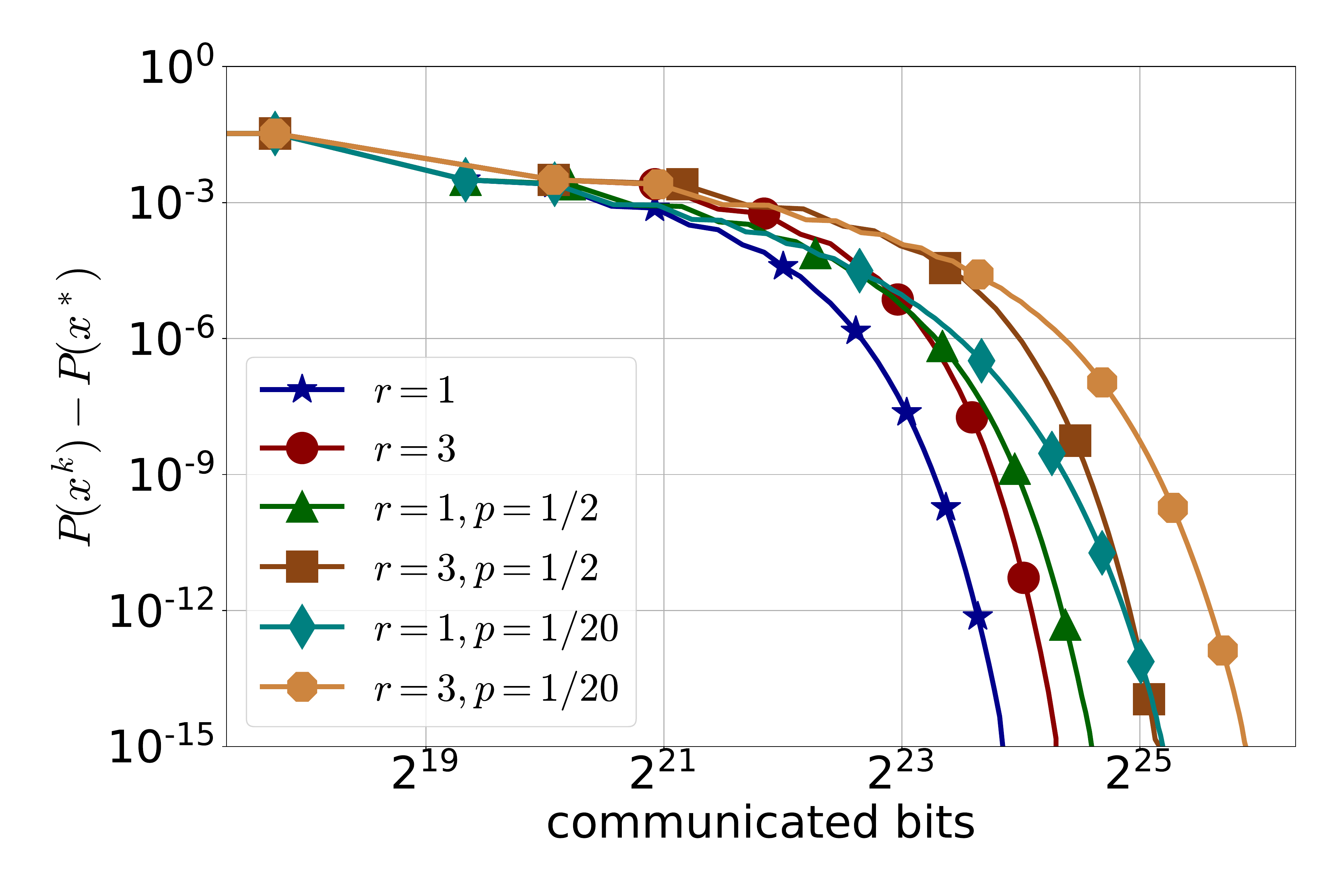}
				\\
				(a) {\tt a2a}, $\lambda=10^{-3}$ &(b) {\tt a2a}, $\lambda=10^{-4}$ & (c) {\tt phishing}, $\lambda=10^{-3}$ &(d) {\tt phishing}, $\lambda=10^{-4}$\\
				\includegraphics[width = 0.23 \textwidth]{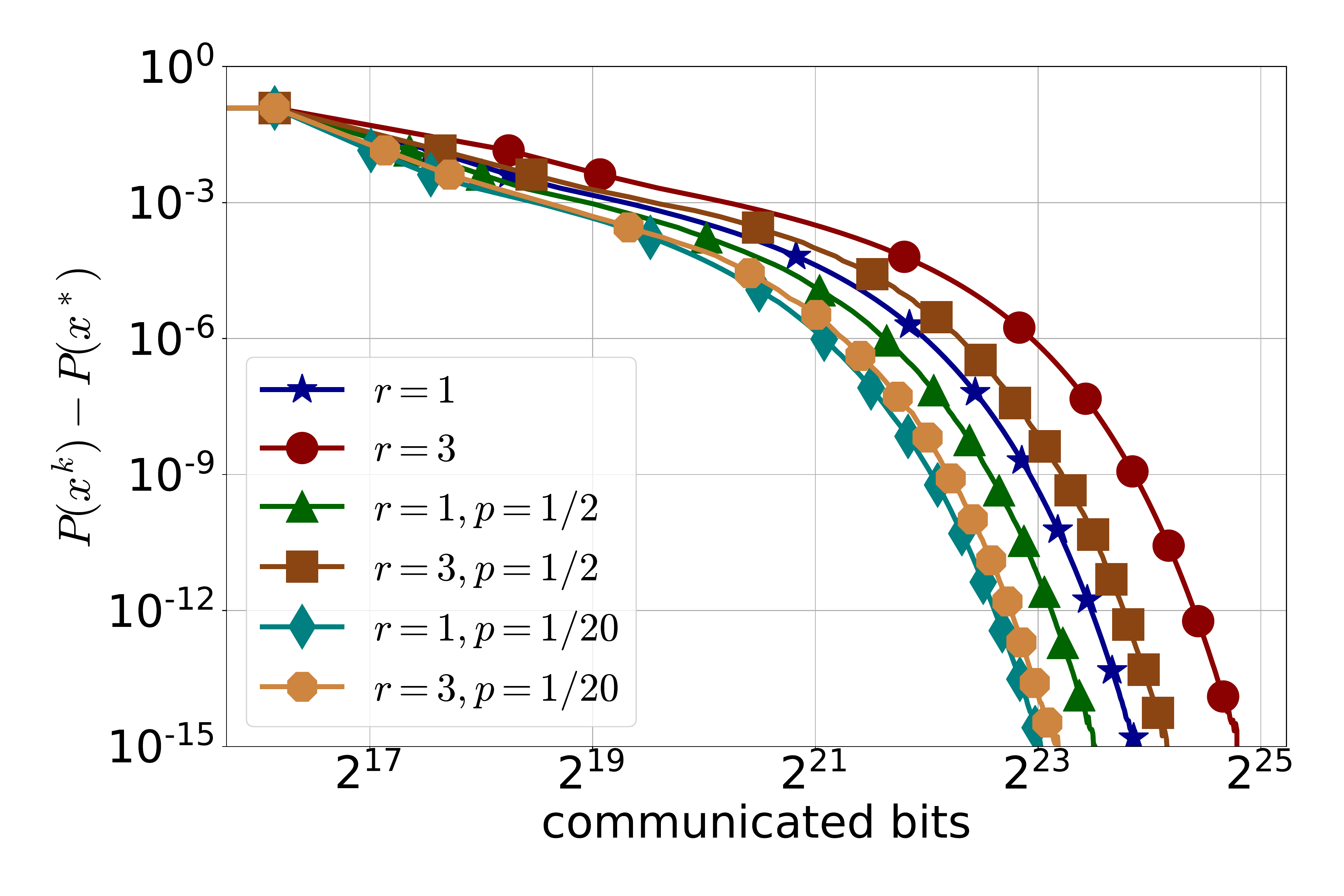}&
				\includegraphics[width = 0.23 \textwidth]{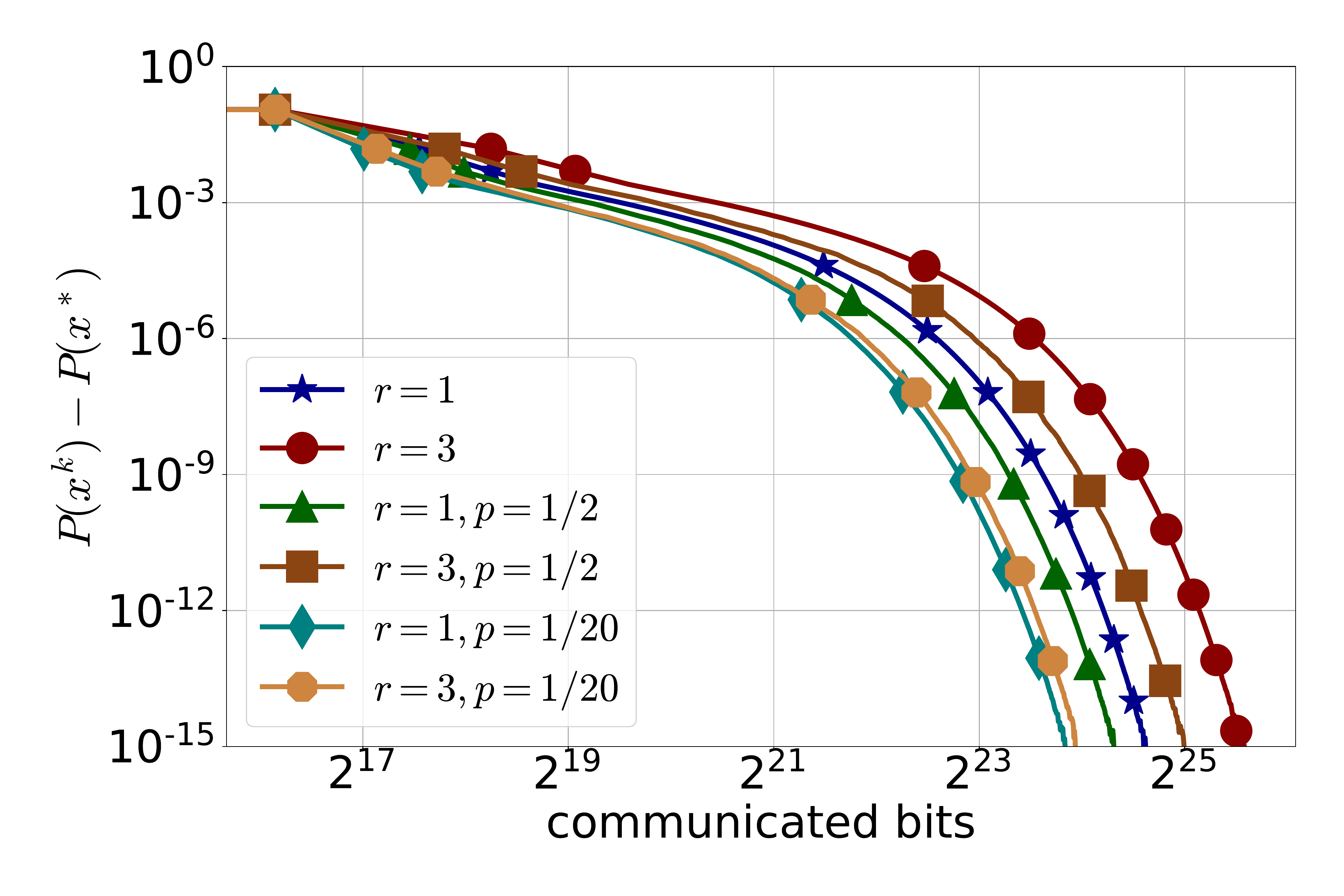}&
				\includegraphics[width = 0.23 \textwidth]{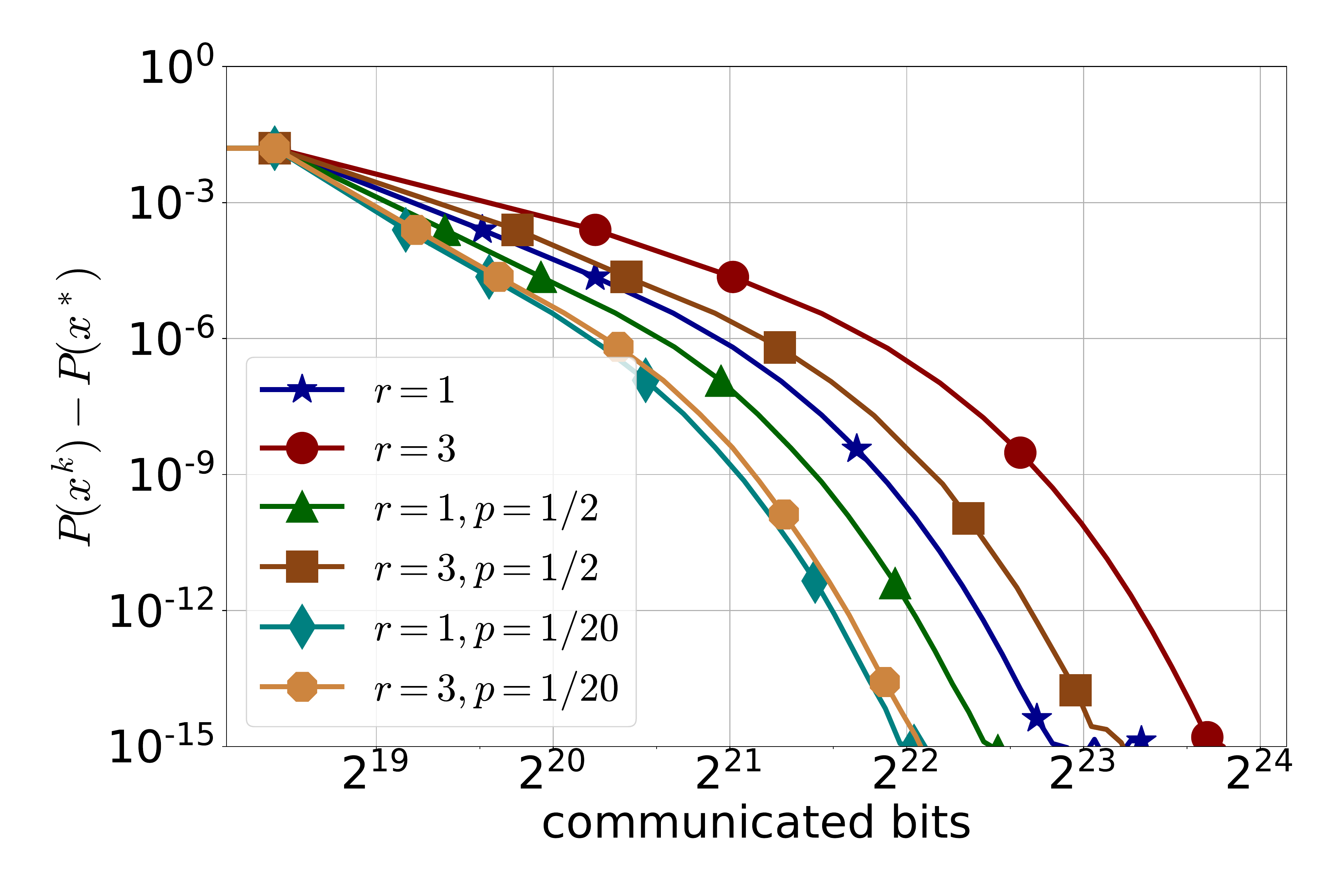}&
				\includegraphics[width = 0.23 \textwidth]{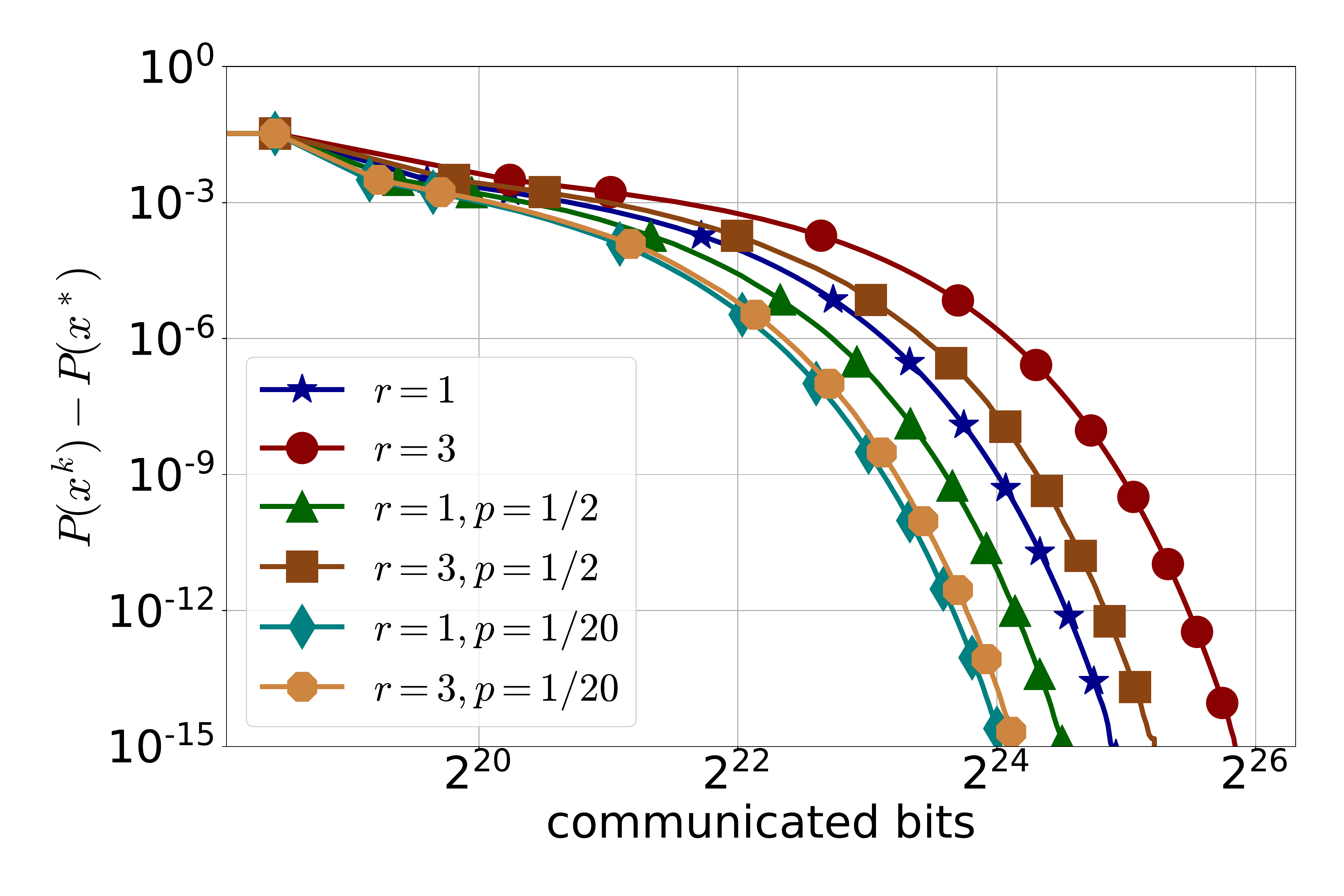}
				\\
				(e) {\tt a2a}, $\lambda=10^{-3}$ &(f) {\tt a2a}, $\lambda=10^{-4}$ & (g) {\tt phishing}, $\lambda=10^{-3}$ &(h) {\tt phishing}, $\lambda=10^{-4}$
		\end{tabular}}
		\caption{Performance of {\sf NL1} (first row) and {\sf NL2} (second row) across a few  values of $r$ defining the random-$r$ compressor, and a few values of $p$ defining the induced Bernoulli compressor $\cC_p$. }
		\label{exp:NL1_NL2}
	\end{center}
	\vskip -0.2in
\end{figure}

\subsection{Comparison of {\sf NL1} and {\sf NL2} with Newton's method}

In our next experiment we compare {\sf NL1} and {\sf NL2} using different values of $r$ for random-$r$ compression, with Newton's method; see Figure~\ref{a9a:alg1_alg2}. We clearly see that Newton's method performs better than {\sf NL1} and {\sf NL2} in terms of iteration complexity, as expected. However, our methods have better communication efficiency than Newton's method, by {\em several orders of magnitude}. Moreover, we see that the smaller $r$ is, the better {\sf NL1} and {\sf NL2} perform in terms of communication complexity. In  Figure~\ref{exp:newton} we perform a similar comparison for several more datasets, but focus on communication complexity only. The conclusions are unchanged: our methods  {\sf NL1} and {\sf NL2} have superior performance.

\begin{figure}[ht]
	\vskip 0.2in
	\begin{center}
		\begin{tabular}{cccc}
			\includegraphics[width = 0.23 \textwidth]{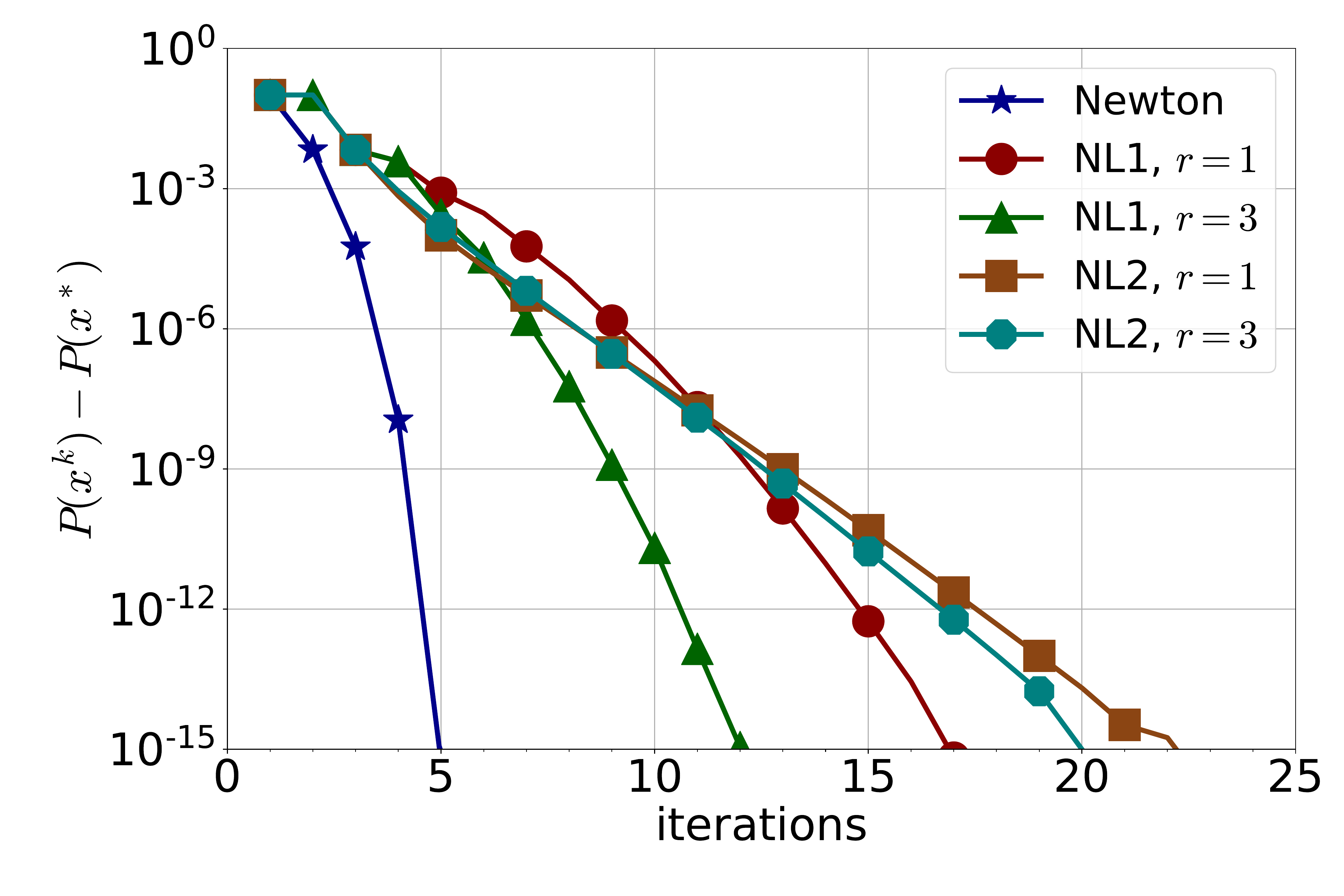}&
			\includegraphics[width = 0.23 \textwidth]{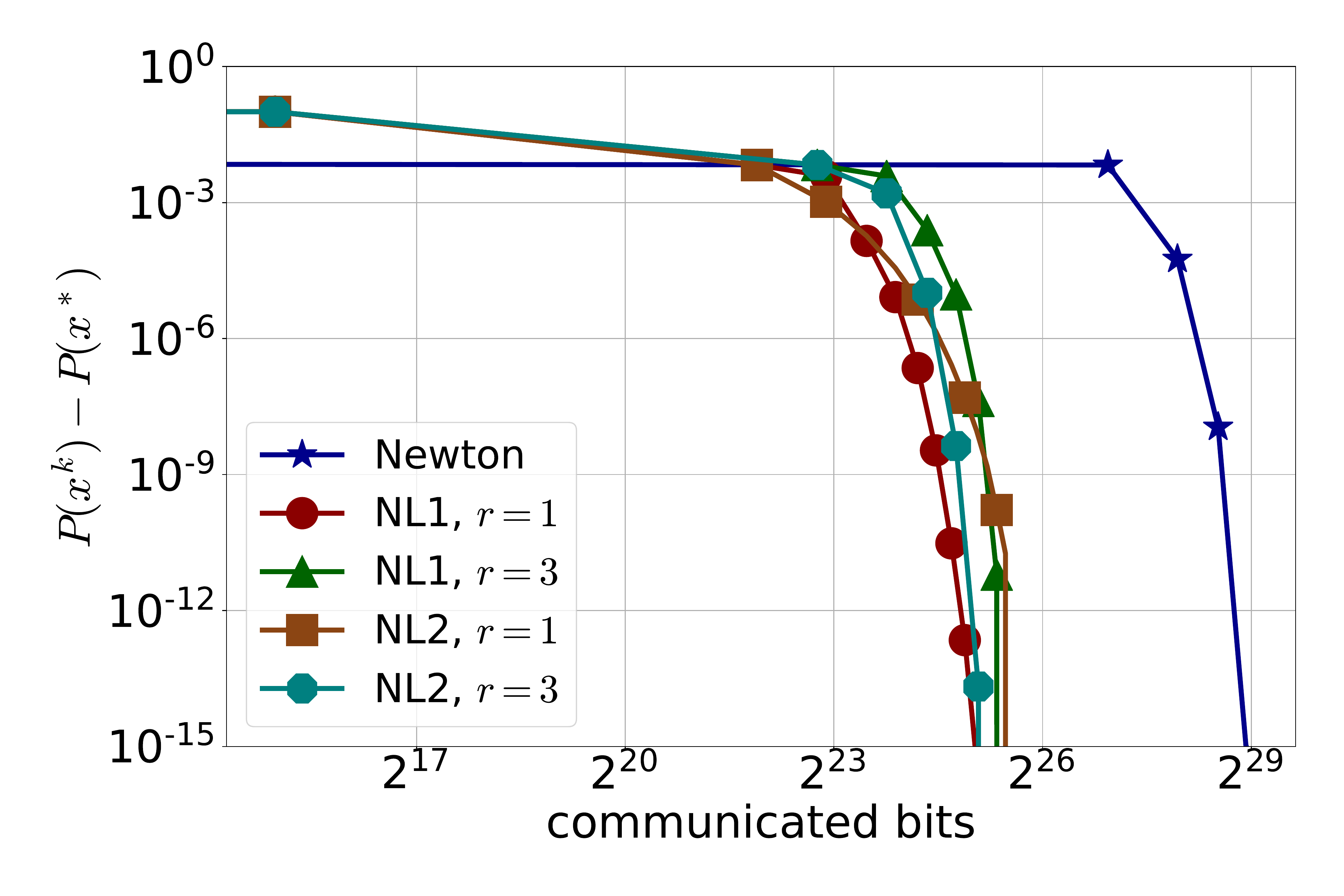}&	
			\includegraphics[width = 0.23 \textwidth]{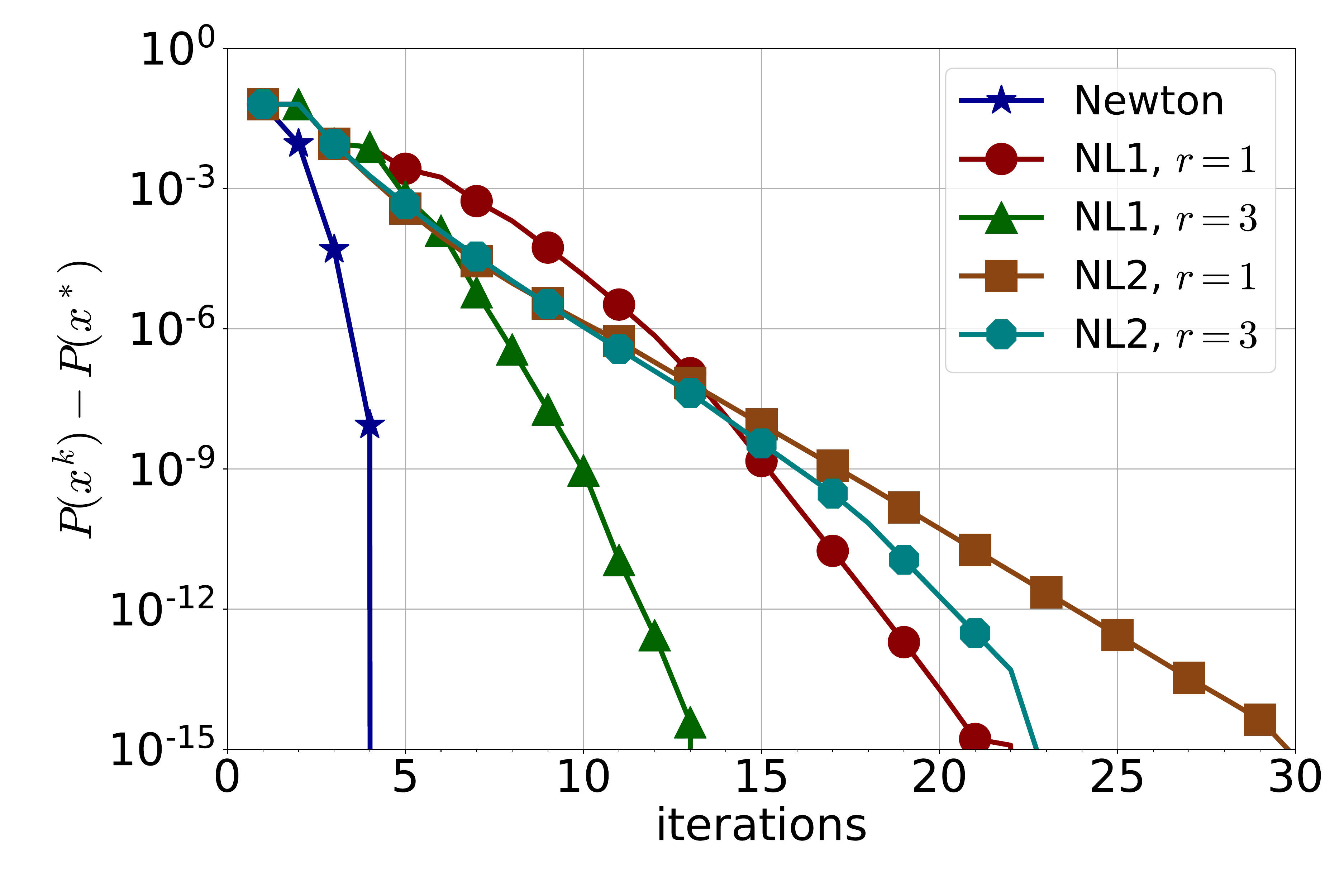}&
			\includegraphics[width = 0.23 \textwidth]{LogReg/artificial/Lambda=1e-4/artificial_nm_nl1_nl2_bits_lmb=0.0001.pdf}\\
			(a) {\tt artificial} & (b) {\tt artificial} & (c) {\tt artificial} & (d) {\tt artificial}\\
			$\lambda=10^{-4}$ & $\lambda=10^{-4}$ & $\lambda=10^{-5}$ & $\lambda=10^{-5}$\\
		\end{tabular}
		\caption{Comparison of {\sf NL1} and {\sf NL2} with Newton's method in terms of iteration complexity for (a), (c); in terms of communication complexity for (b), (d).}
		\label{a9a:alg1_alg2}
	\end{center}
\end{figure}

\begin{figure}[ht]
	\begin{center}
		\centerline{\begin{tabular}{cccc}
				\includegraphics[width = 0.23 \textwidth]{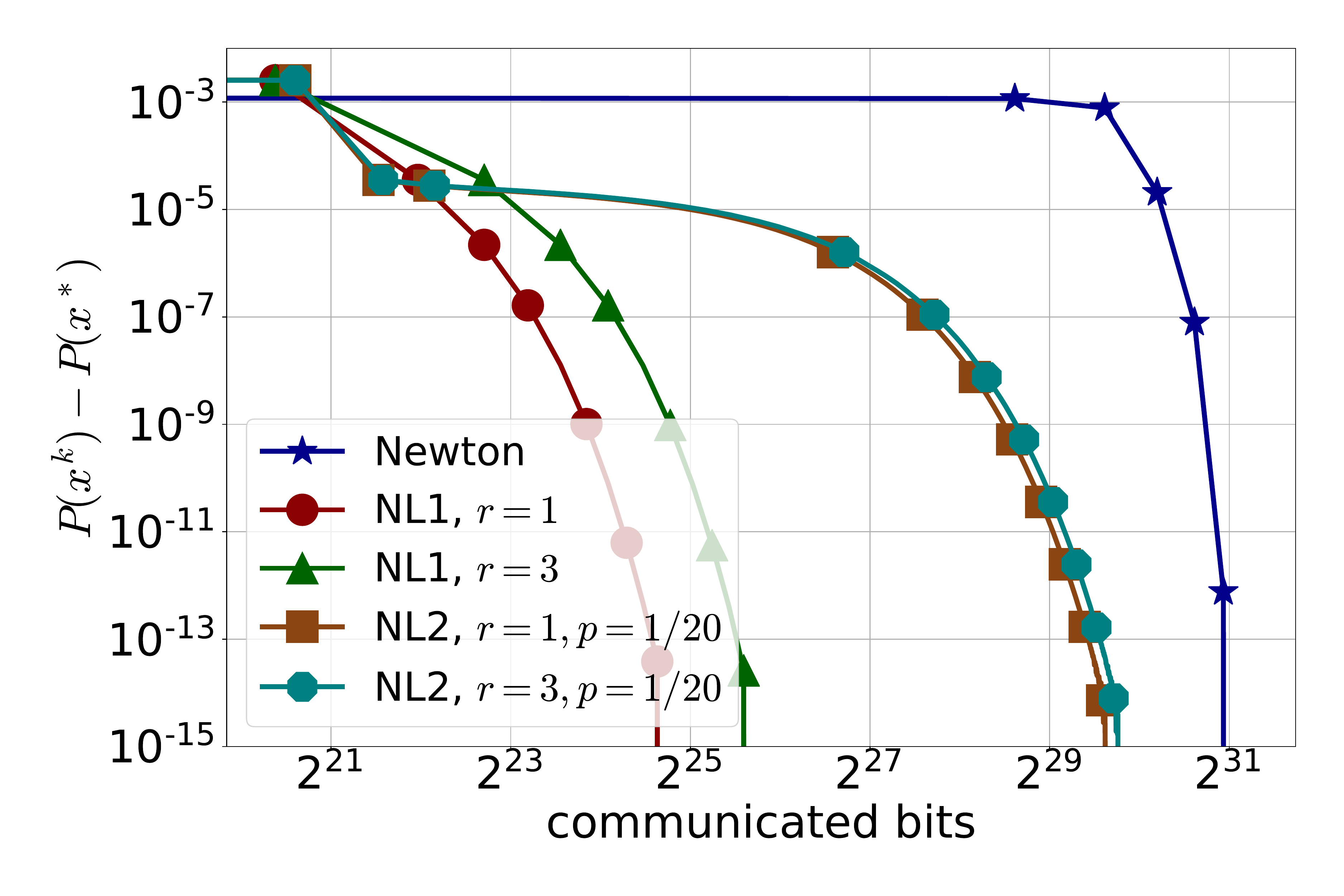}&
				\includegraphics[width = 0.23 \textwidth]{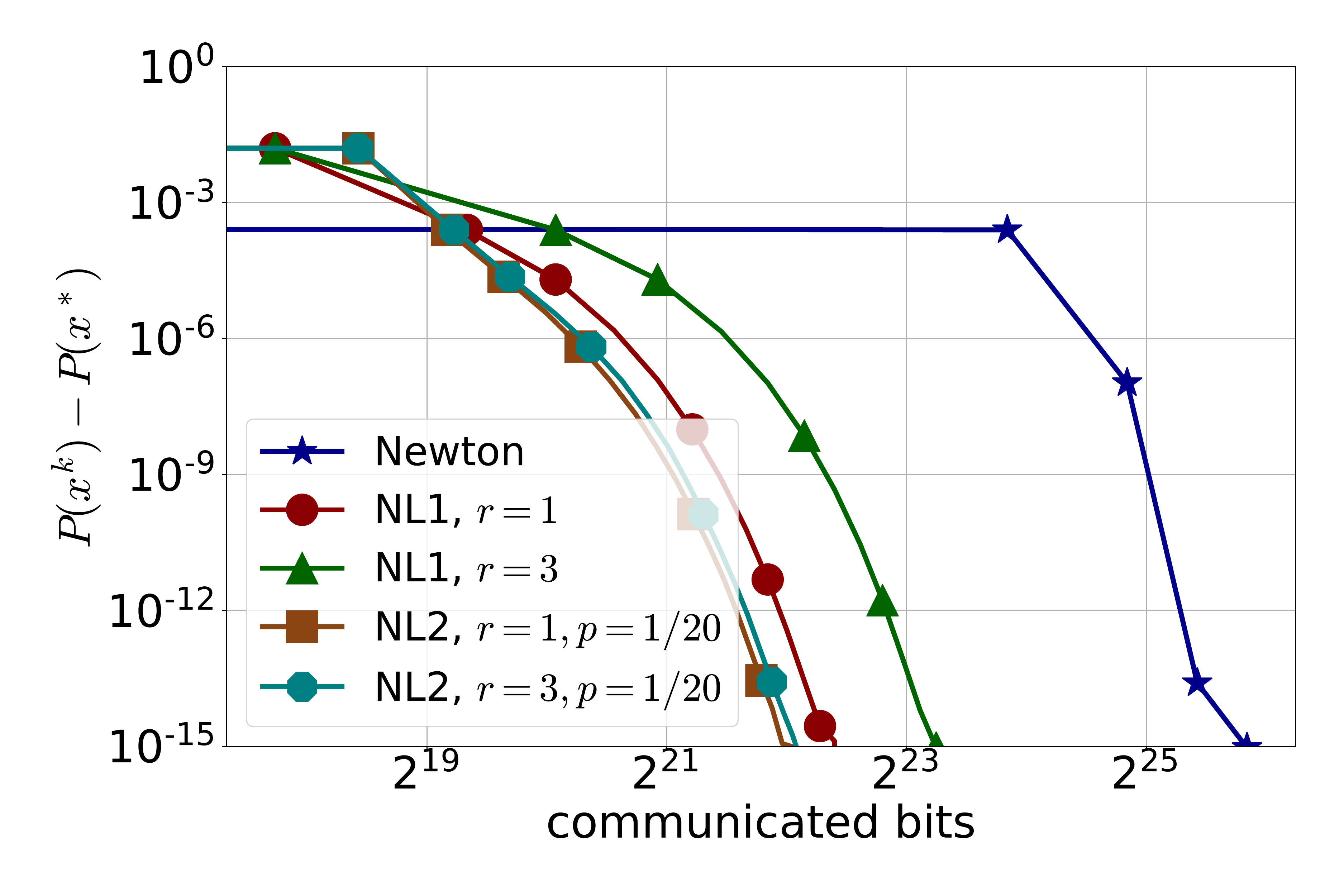}&
				\includegraphics[width = 0.23 \textwidth]{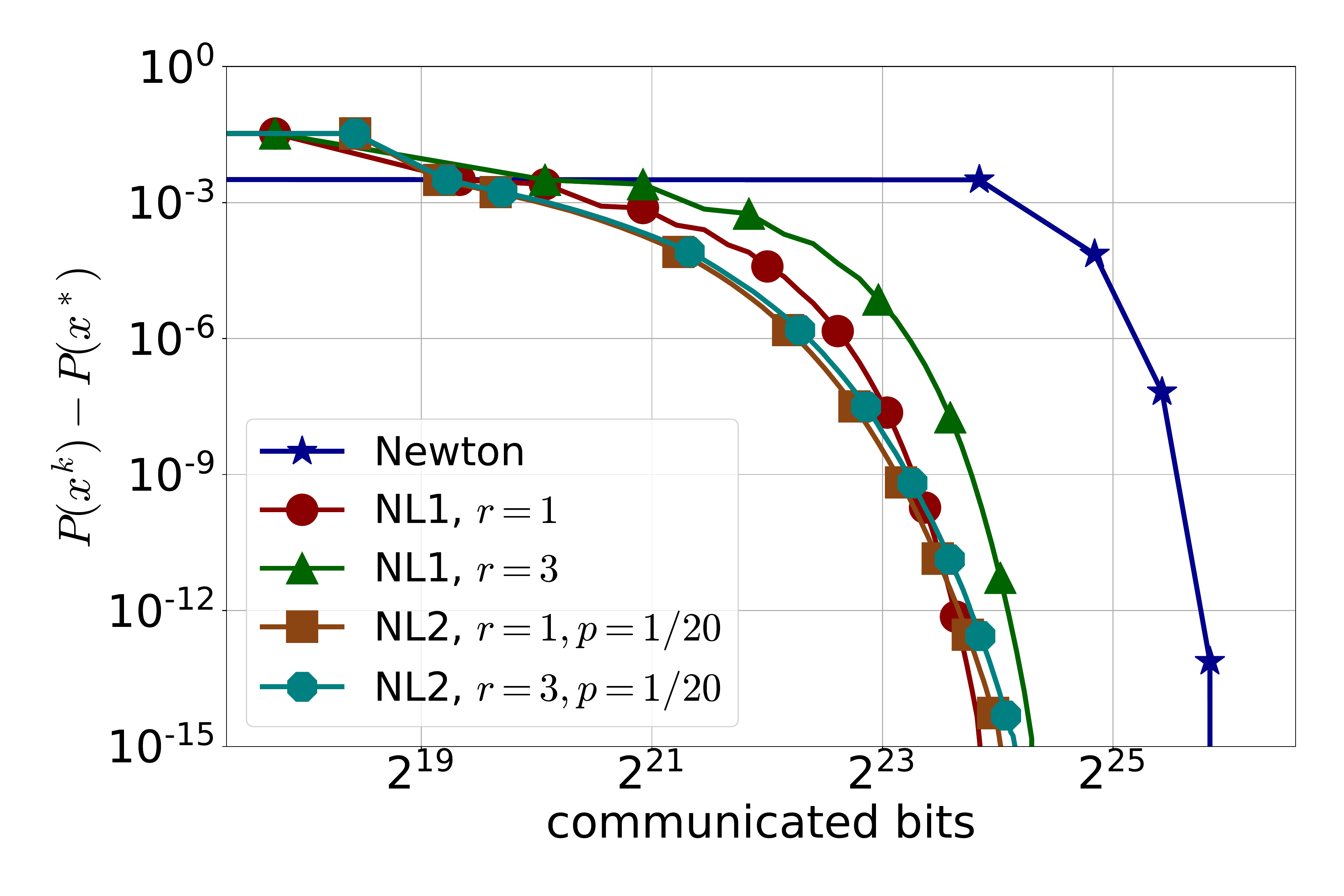}&
				\includegraphics[width = 0.23 \textwidth]{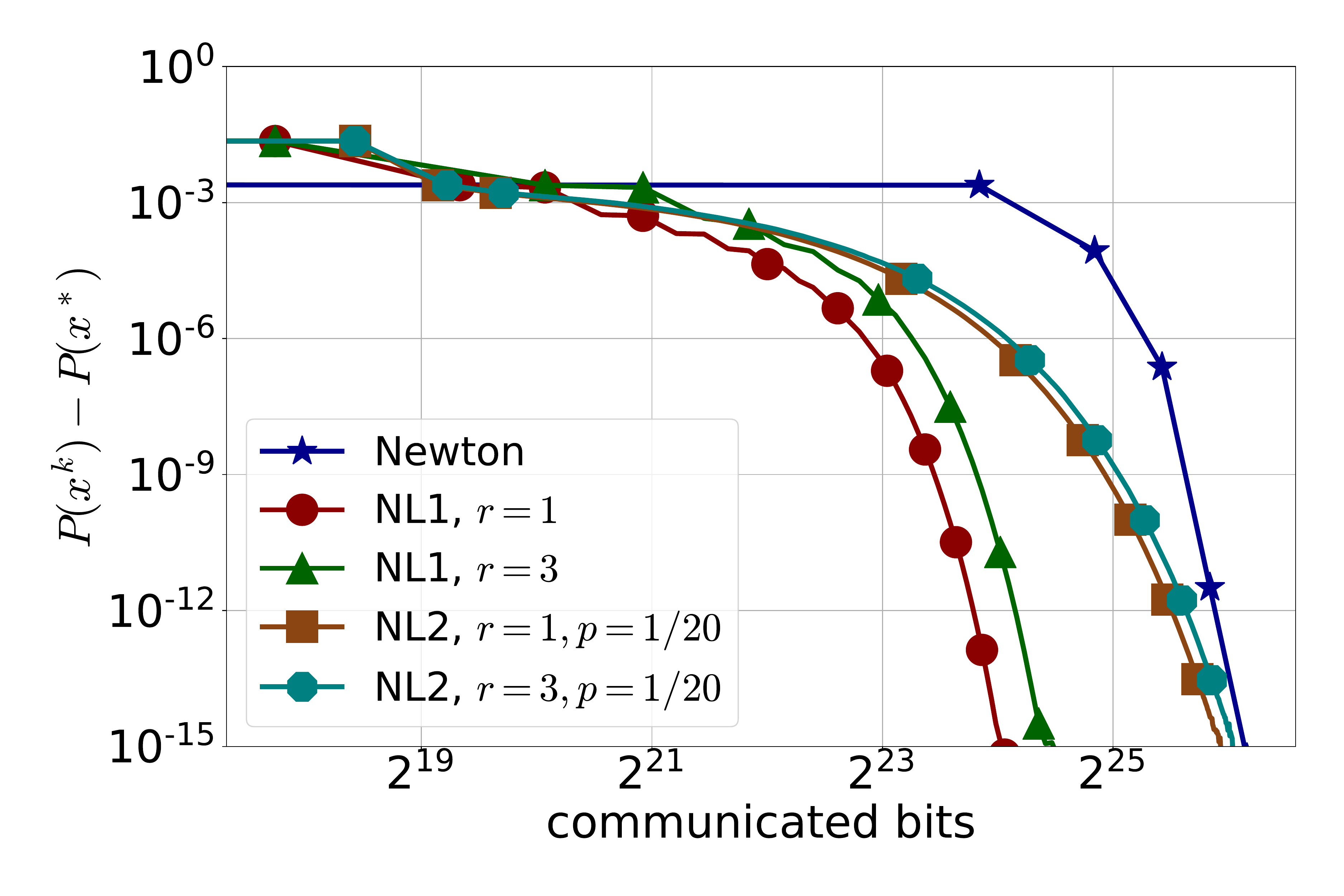}
				\\
				(a) {\tt w8a}, $\lambda=10^{-3}$ &(b) {\tt phishing}, $\lambda=10^{-3}$ & (c) {\tt phishing}, $\lambda=10^{-4}$ &(d) {\tt phishing}, $\lambda=10^{-5}$\\
				\includegraphics[width = 0.23 \textwidth]{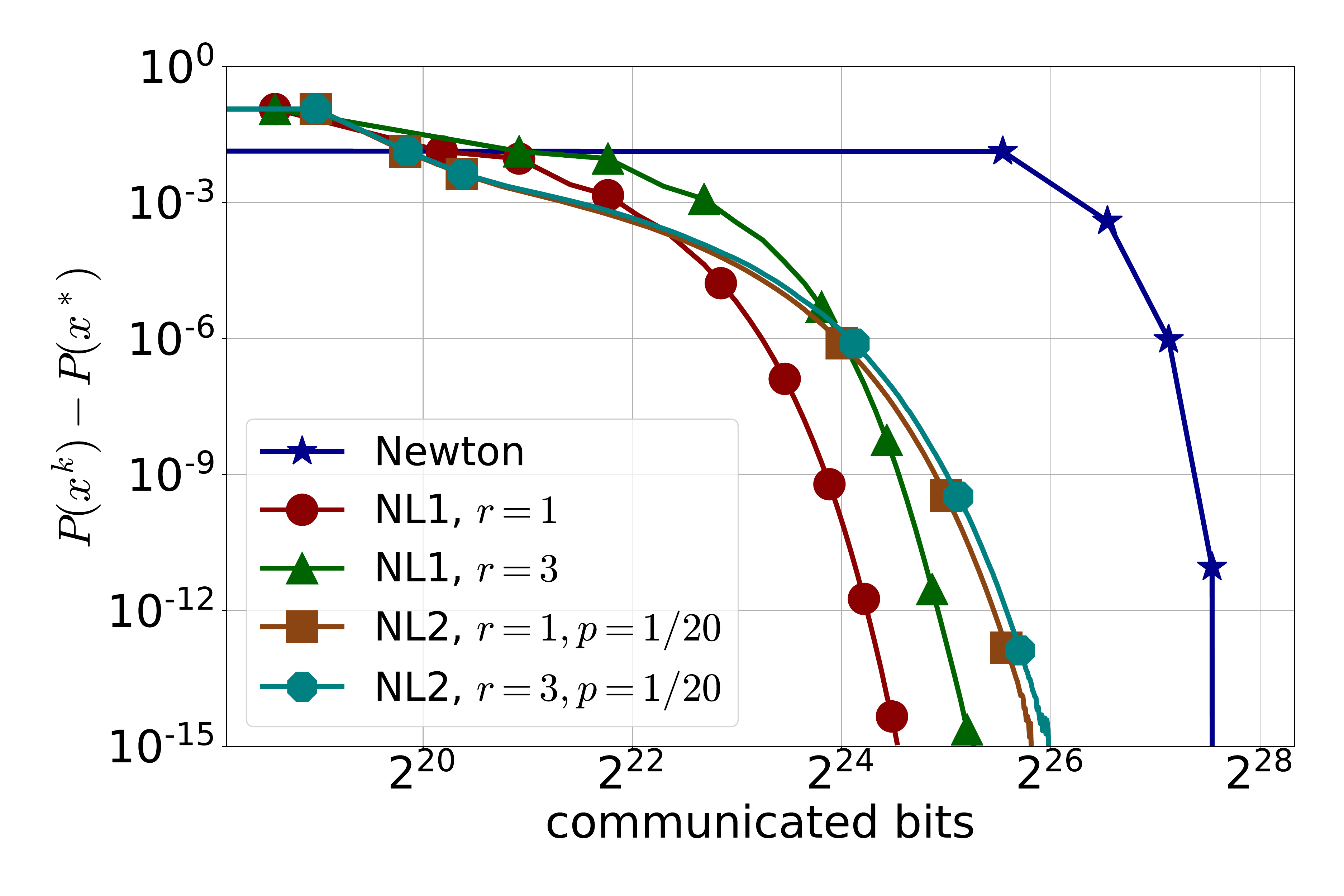}&
				\includegraphics[width = 0.23 \textwidth]{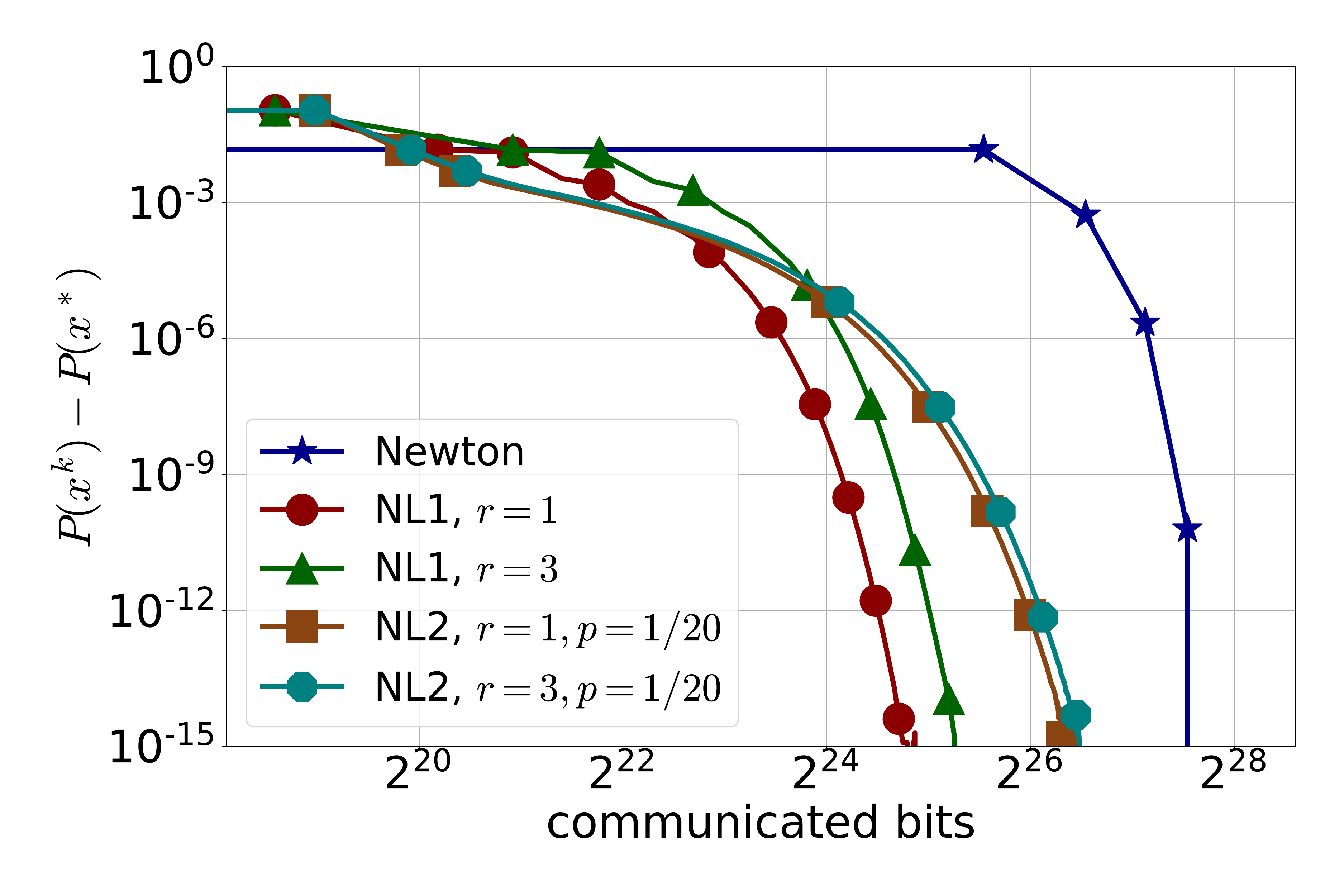}&
				\includegraphics[width = 0.23 \textwidth]{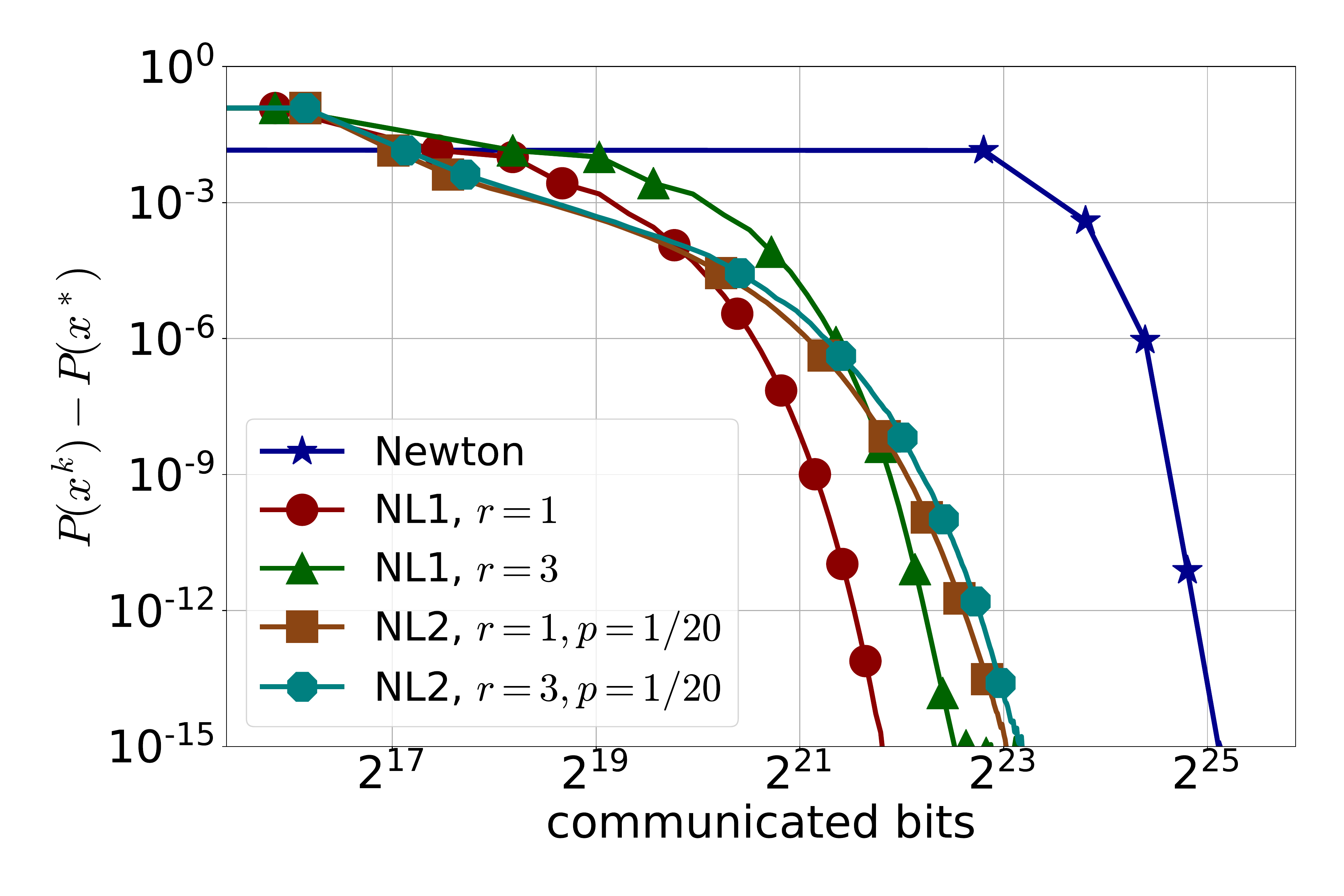}&
				\includegraphics[width = 0.23 \textwidth]{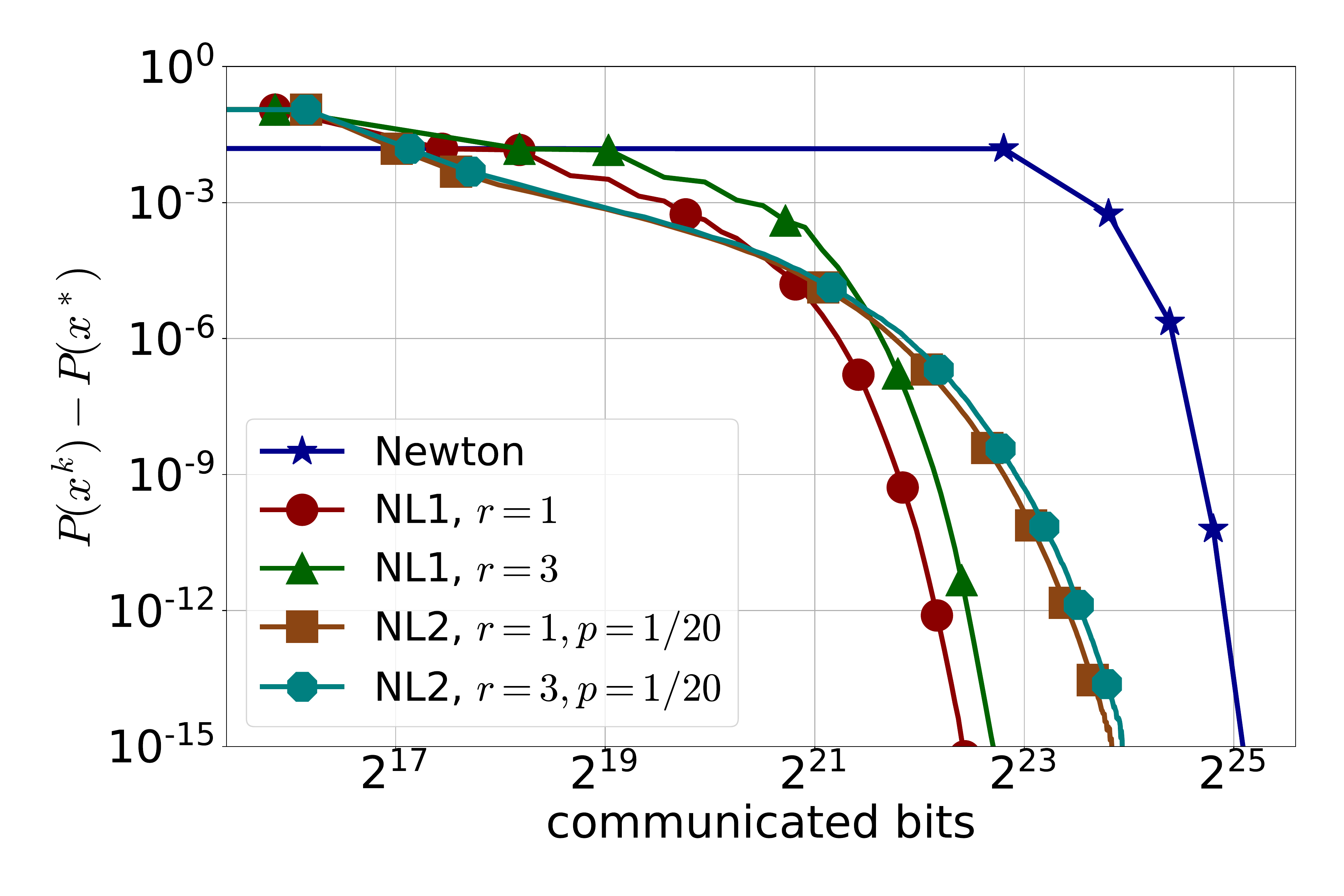}
				\\
				(e) {\tt a7a}, $\lambda=10^{-3}$ &(f) {\tt a7a}, $\lambda=10^{-3}$ & (g) {\tt a2a}, $\lambda=10^{-3}$ &(h) {\tt a2a}, $\lambda=10^{-4}$
		\end{tabular}}
		\caption{Comparison of {\sf NL1}, {\sf NL2} with Newton's method in terms of communication complexity.}
		\label{exp:newton}
	\end{center}
\end{figure}

\subsection{Comparison of {\sf NL1} and {\sf NL2} with BFGS}

In our next test, we compare {\sf NL1} and {\sf NL2}  with BFGS in Figure~\ref{exp:bfgs}. As we can see, our methods have better communication efficiency than BFGS, by {\em several orders of magnitude}.

\begin{figure}[ht]
	\begin{center}
		\centerline{\begin{tabular}{cccc}
				\includegraphics[width = 0.23 \textwidth]{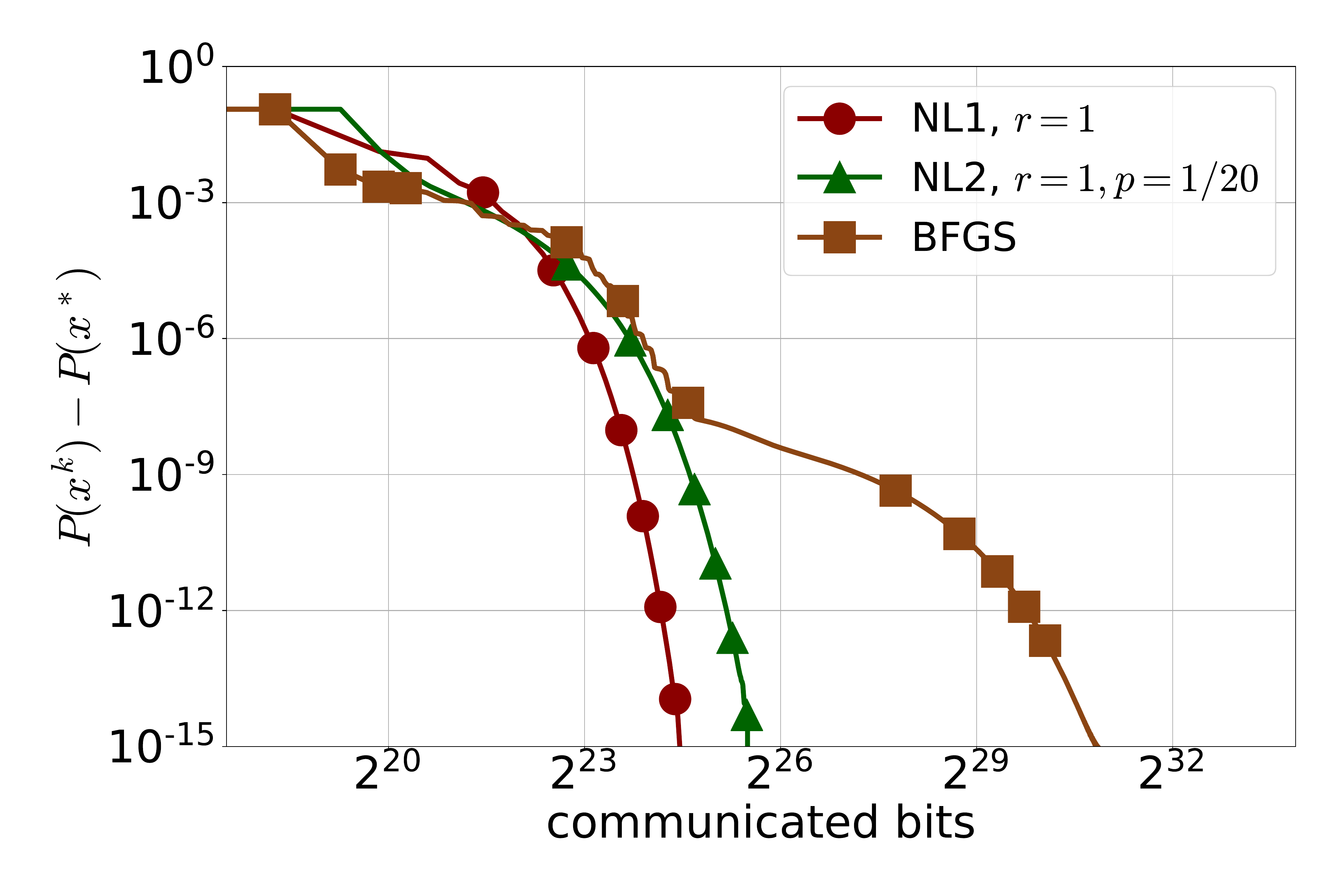}&
				\includegraphics[width = 0.23 \textwidth]{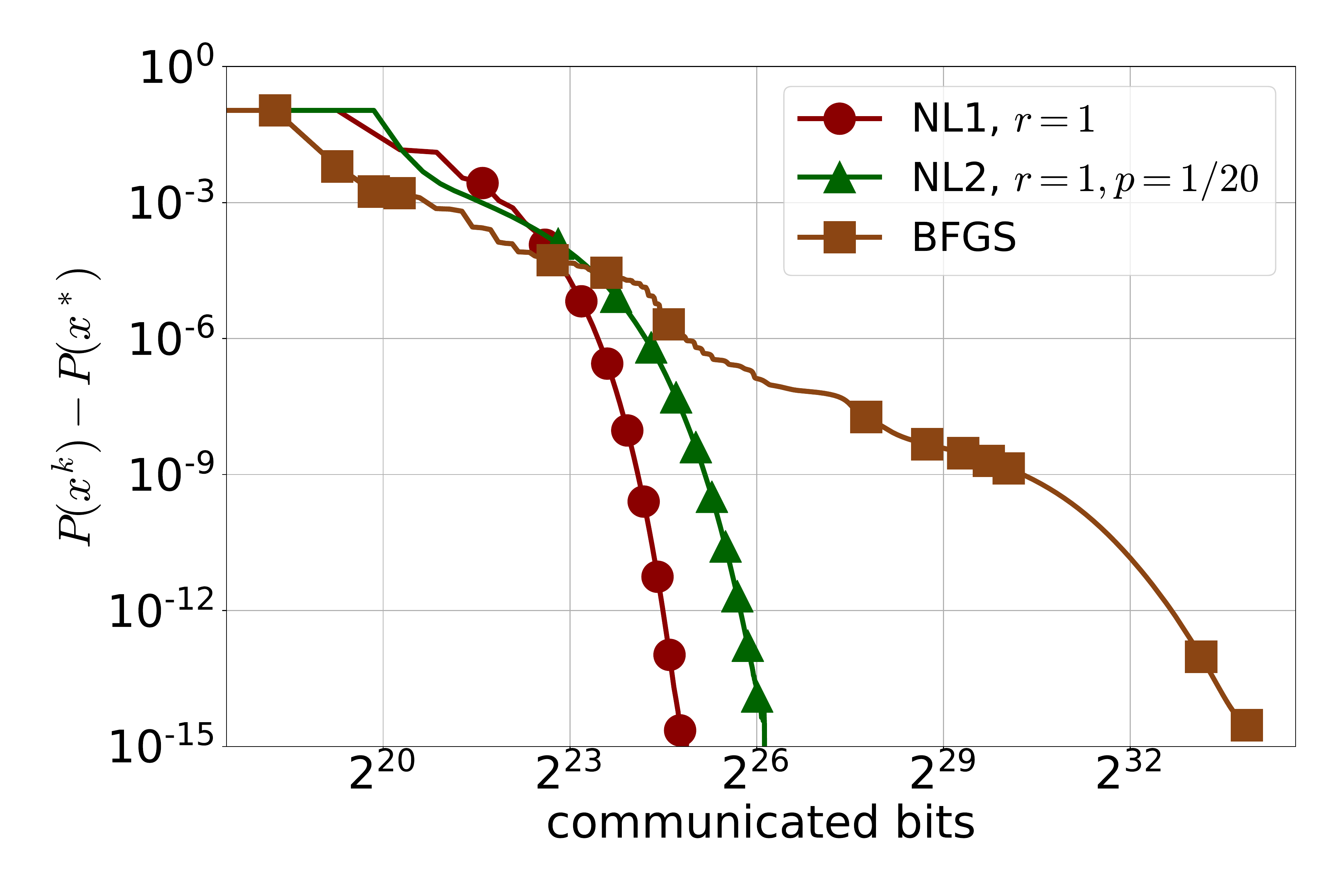}&
				\includegraphics[width = 0.23 \textwidth]{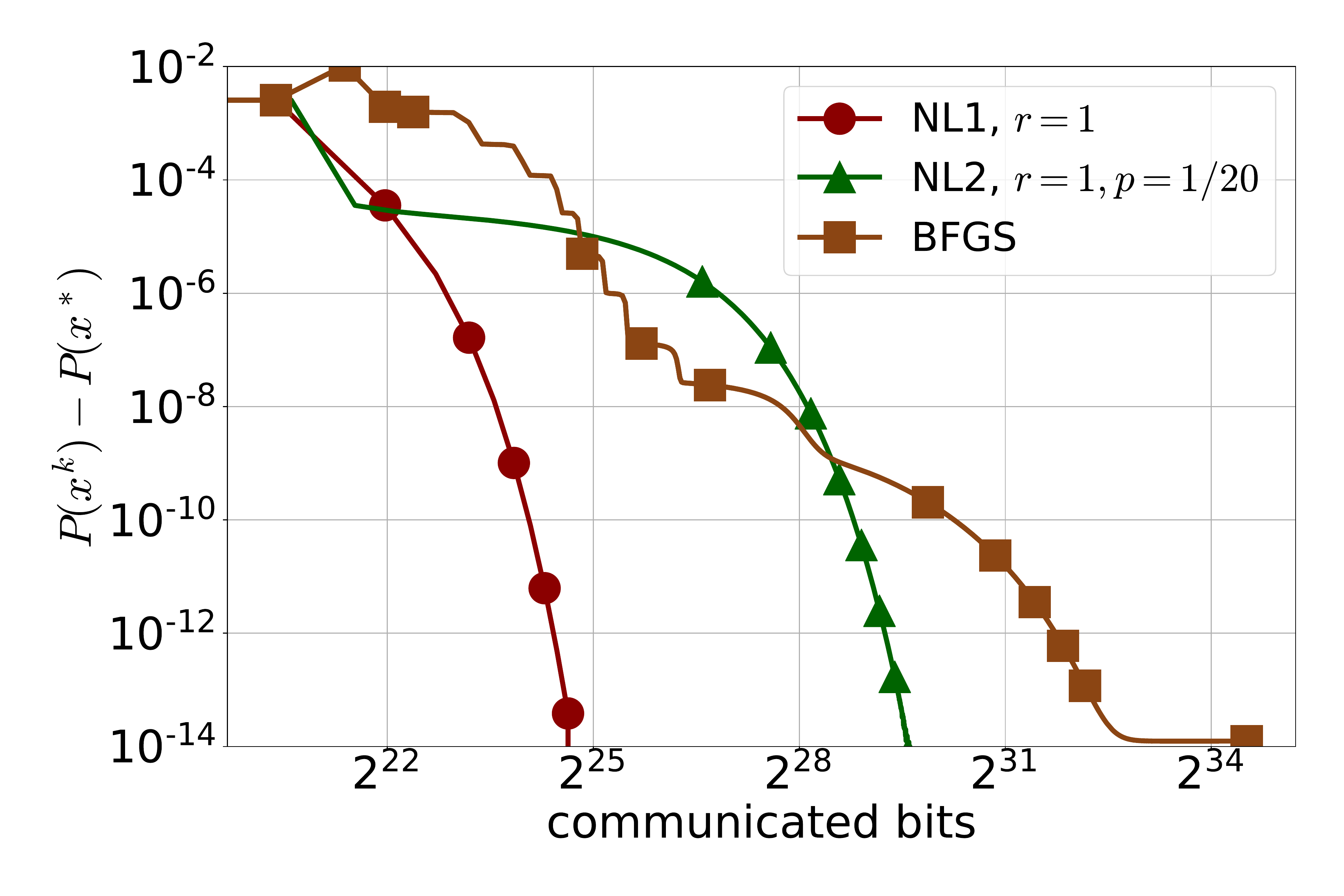}&
				\includegraphics[width = 0.23 \textwidth]{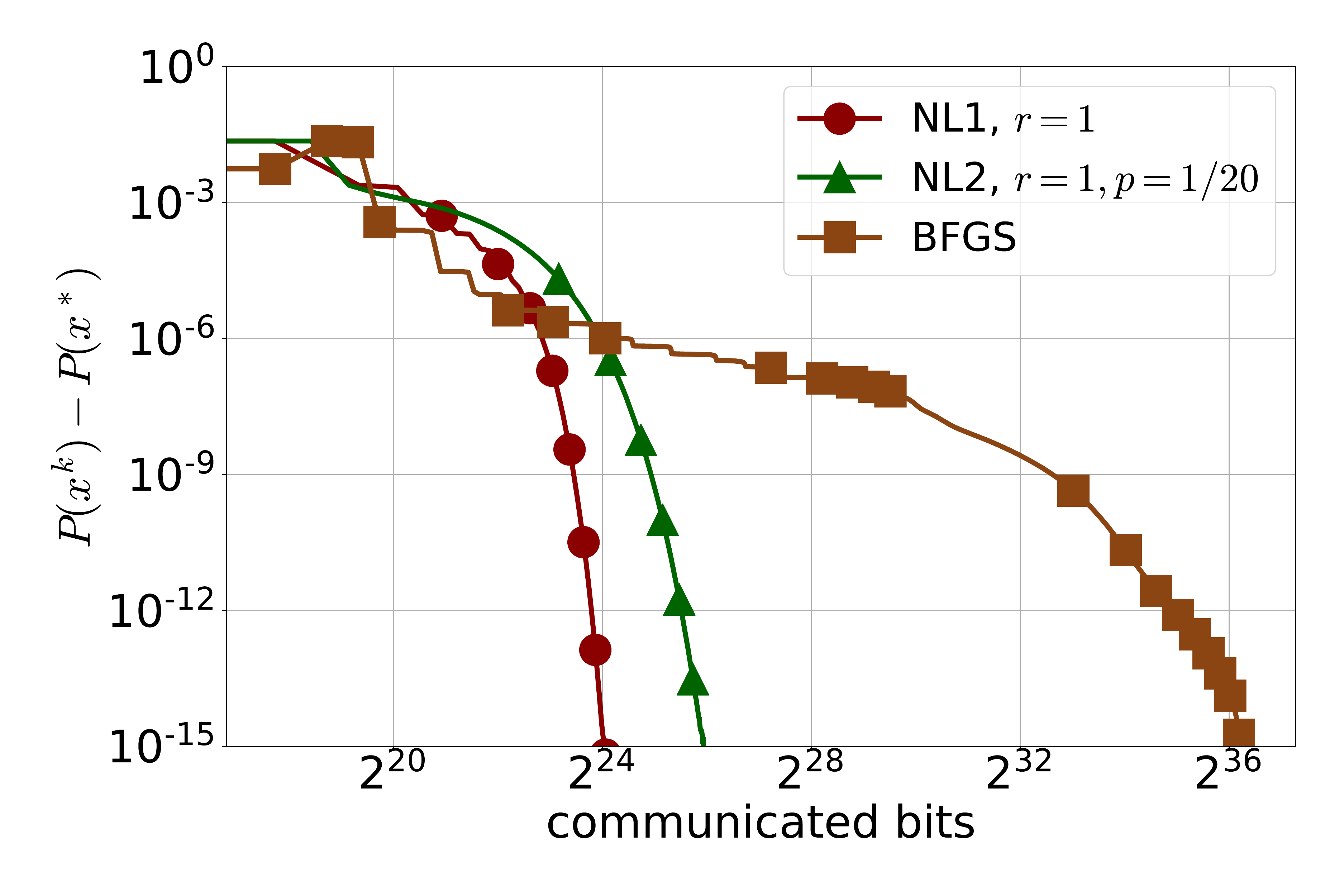}
				\\
				(a) {\tt a9a}, $\lambda=10^{-3}$ &(b) {\tt a9a}, $\lambda=10^{-4}$ & (c) {\tt w8a}, $\lambda=10^{-3}$ &(d) {\tt phishing}, $\lambda=10^{-5}$\\
				\includegraphics[width = 0.23 \textwidth]{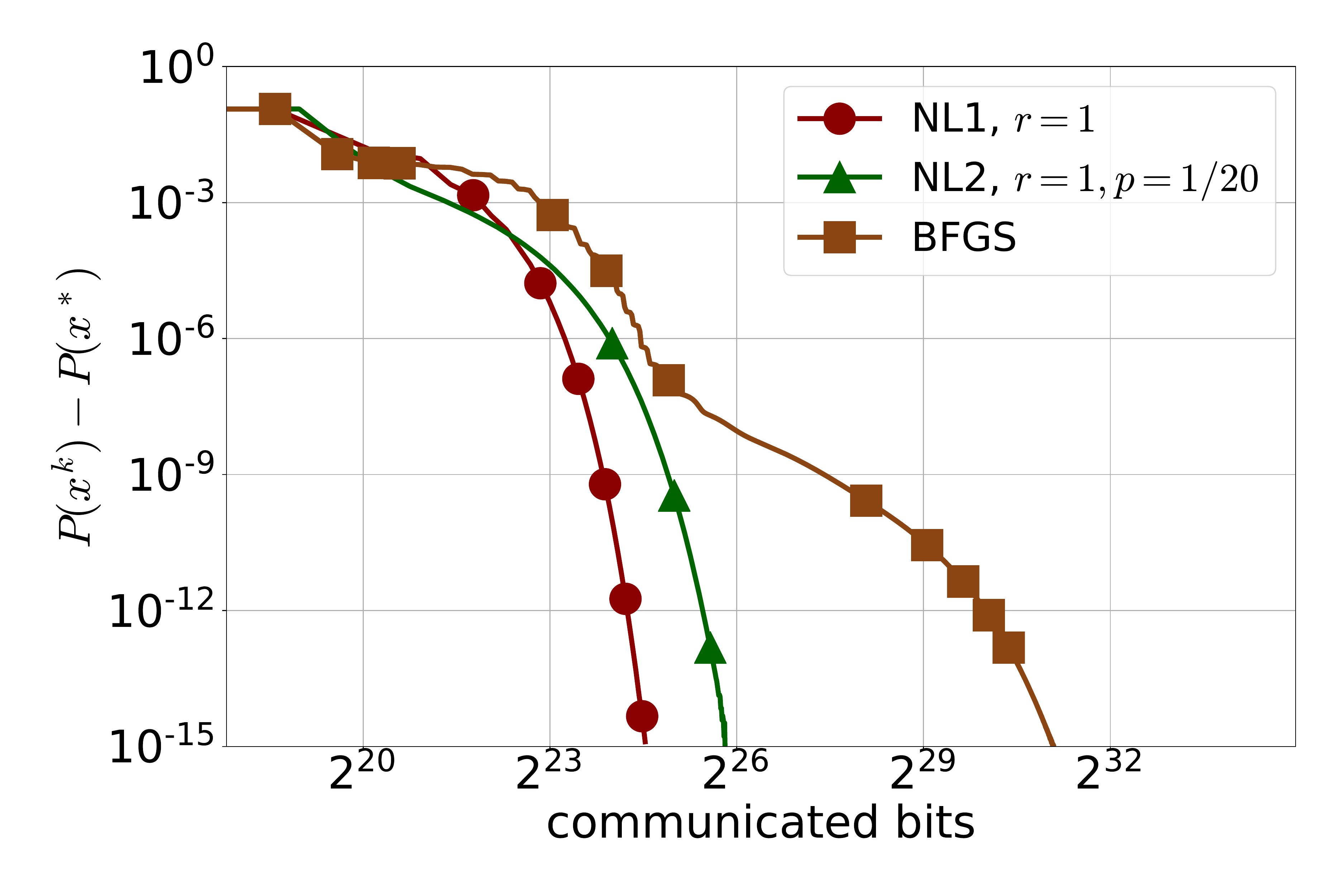}&
				\includegraphics[width = 0.23 \textwidth]{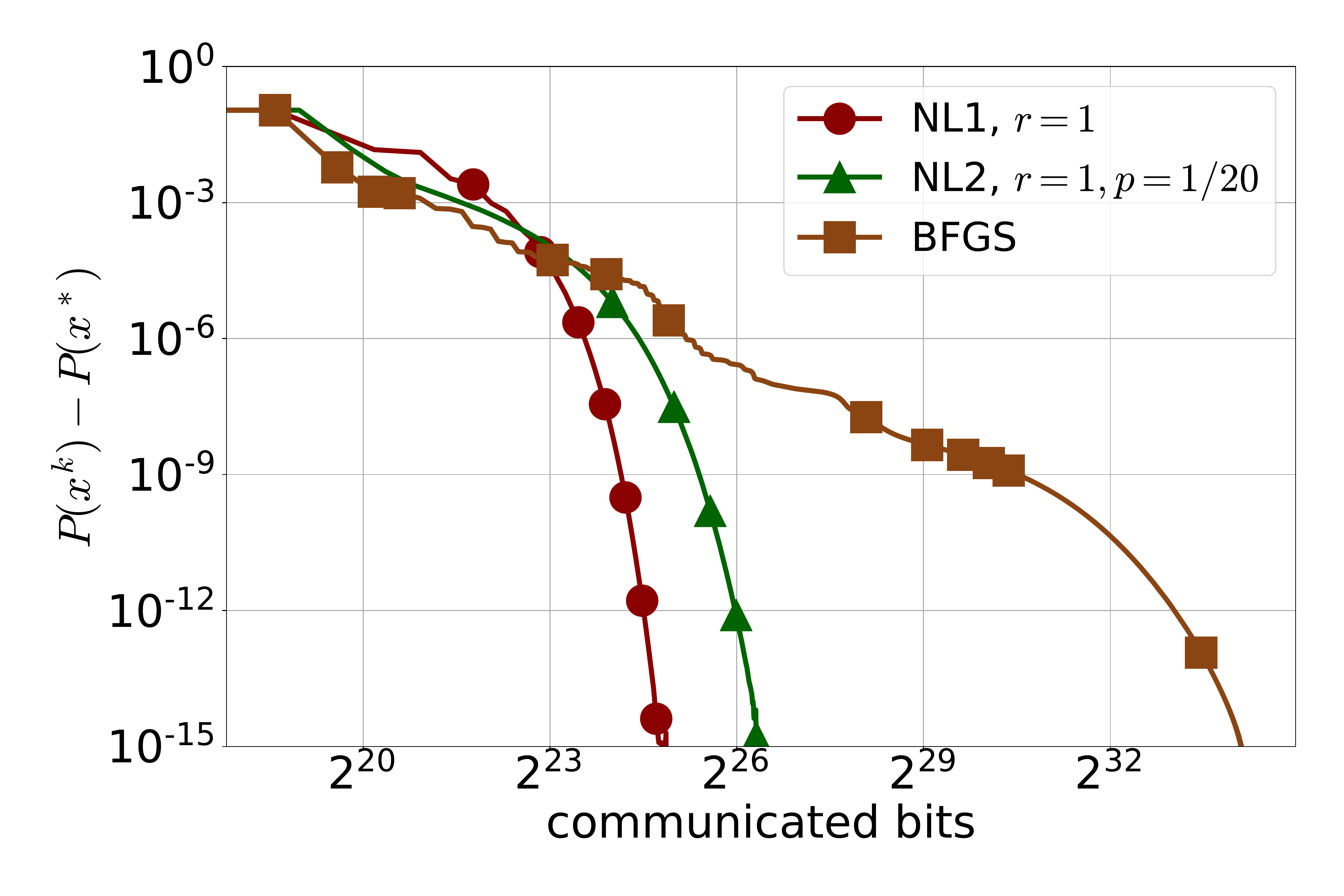}&
				\includegraphics[width = 0.23 \textwidth]{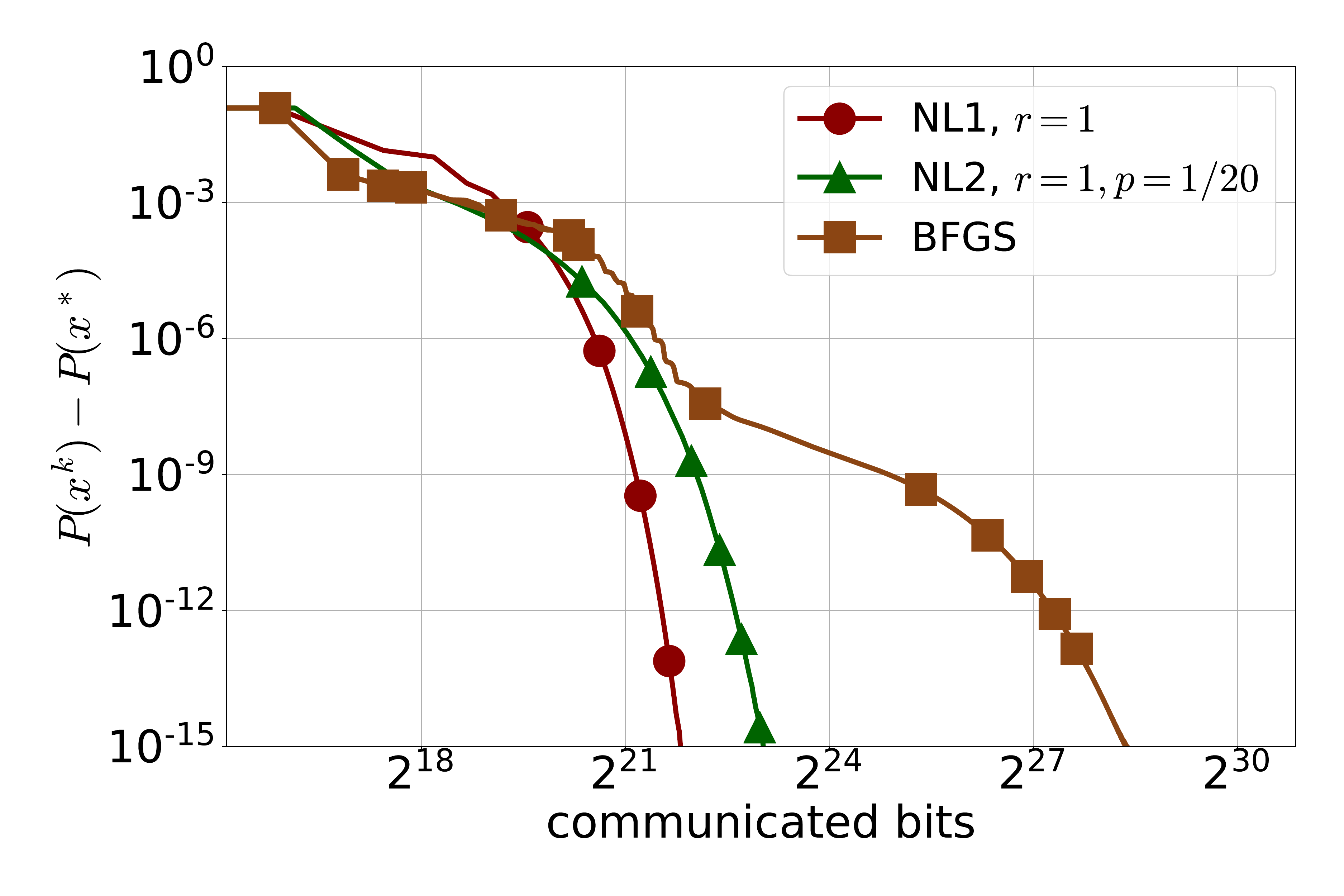}&
				\includegraphics[width = 0.23 \textwidth]{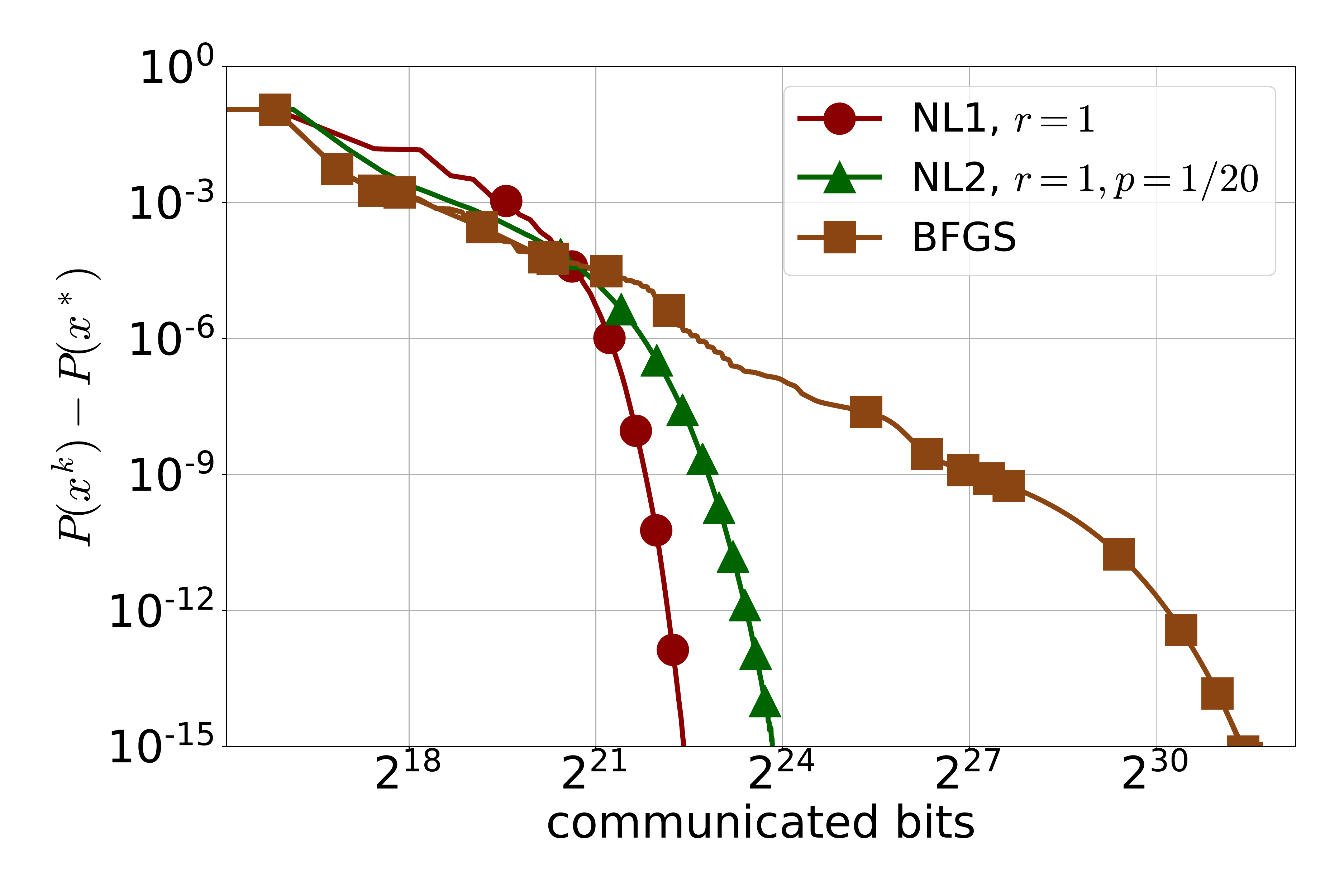}
				\\
				(e) {\tt a7a}, $\lambda=10^{-3}$ &(f) {\tt a7a}, $\lambda=10^{-4}$ & (g) {\tt a2a}, $\lambda=10^{-3}$ &(h) {\tt a2a}, $\lambda=10^{-4}$
		\end{tabular}}
		\caption{Comparison of {\sf NL1}, {\sf NL2} and BFGS in terms of communication complexity.}
		\label{exp:bfgs}
	\end{center}
\end{figure}

\subsection{Comparison of {\sf NL1} and {\sf NL2} with ADIANA}

Next, we compare {\sf NL1} and {\sf NL2} with ADIANA~\citep{ADIANA} using three different compression operators: natural compression (DIANA-NC),  random sparsification (DIANA-RS, $r = d/4$) and random dithering (DIANA-RD, $s = \sqrt{d}$); see Figure~\ref{exp:adiana}. Based on the experimental results, we can conclude that {\sf NL1} outperforms all three versions ADIANA for all types of compression, often by {\em several degrees of magnitude.} The same is true for  {\sf NL2}, with the exception of Figure~\ref{exp:adiana} (c) ({\tt w8a} dataset) where two variants of ADIANA are faster.

\begin{figure}[ht]
	\begin{center}
		\centerline{\begin{tabular}{cccc}
				\includegraphics[width = 0.23 \textwidth]{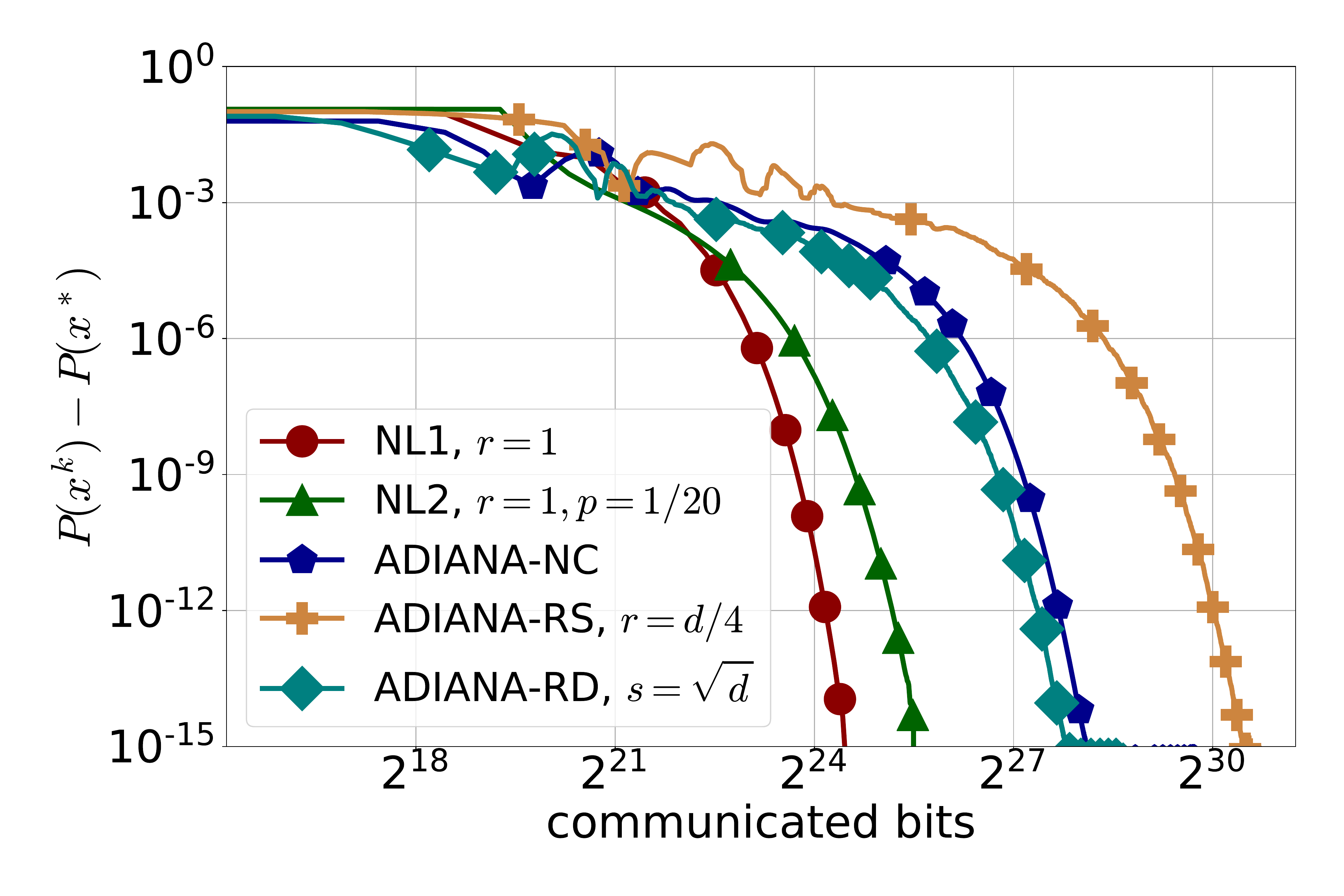}&
				\includegraphics[width = 0.23 \textwidth]{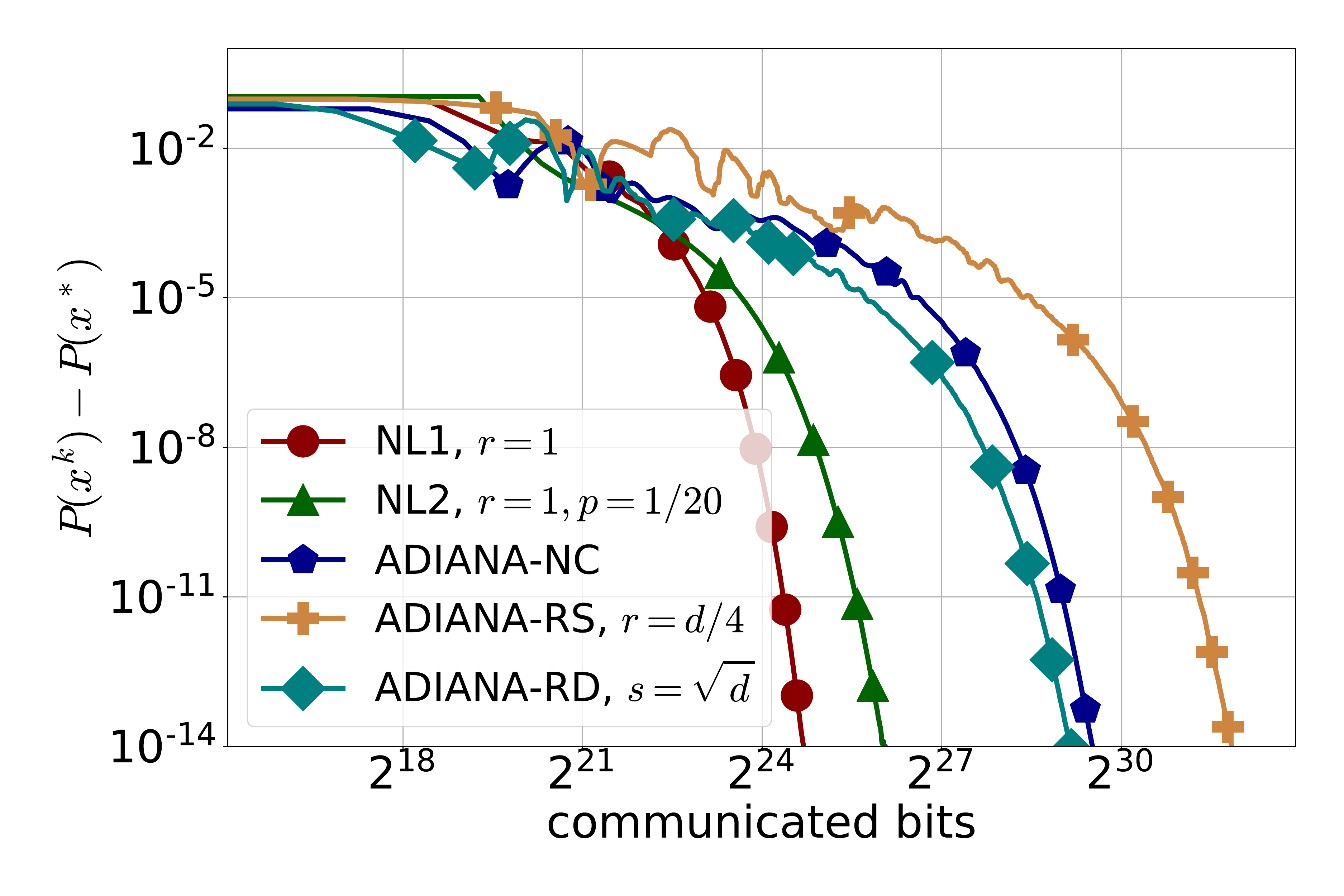}&
				\includegraphics[width = 0.23 \textwidth]{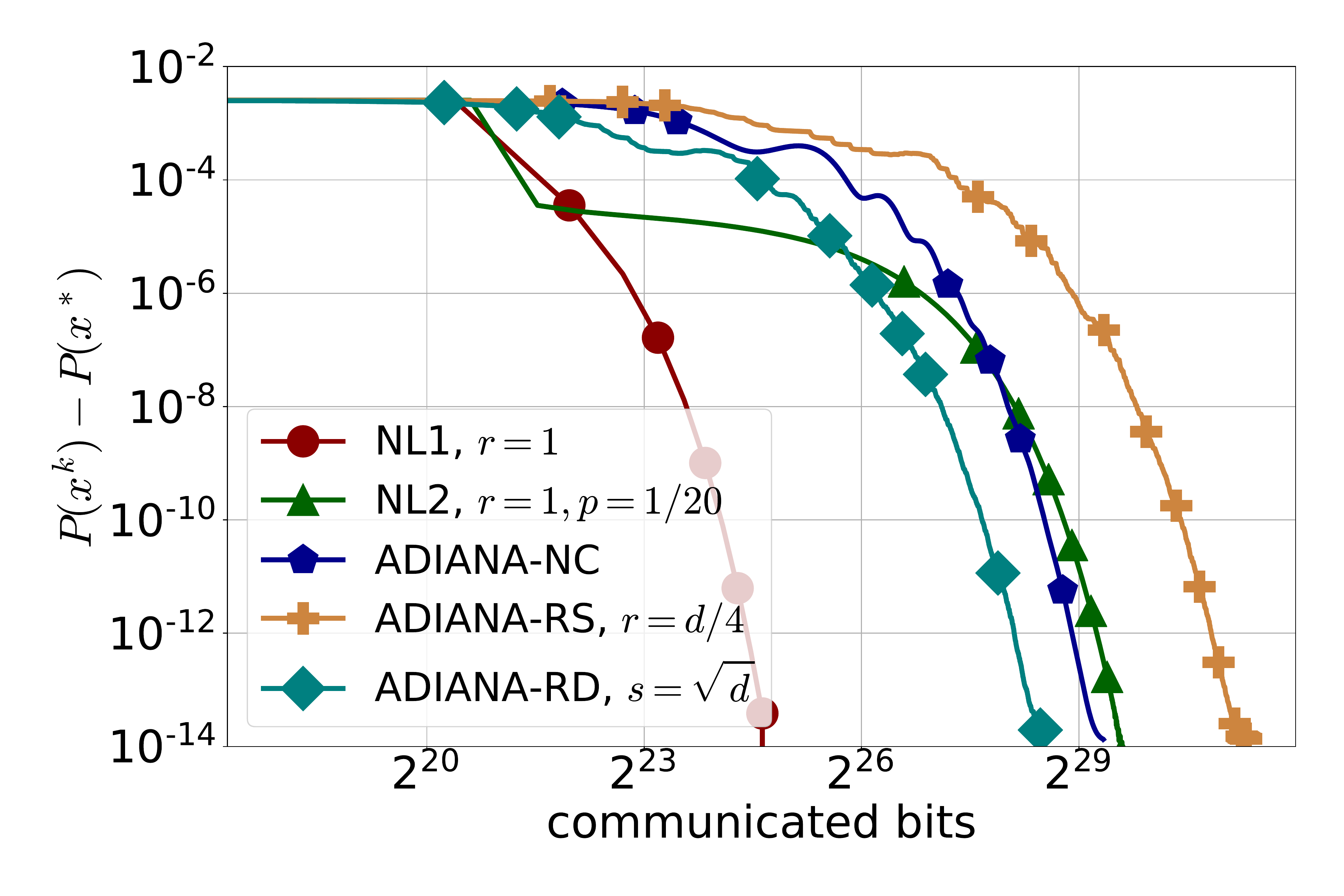}&
				\includegraphics[width = 0.23 \textwidth]{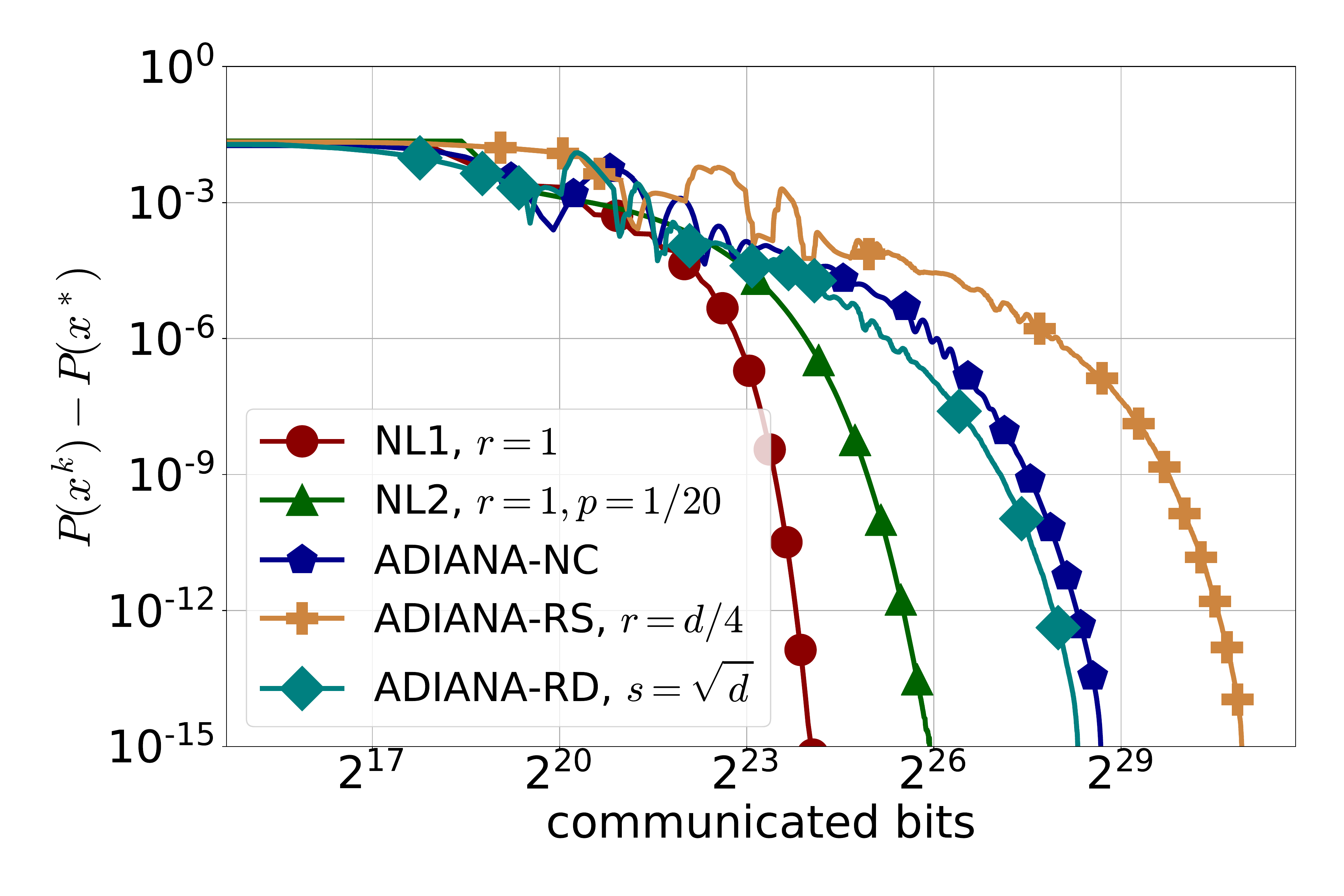}
				\\
				(a) {\tt a9a}, $\lambda=10^{-3}$ &(b) {\tt a9a}, $\lambda=10^{-4}$ & (c) {\tt w8a}, $\lambda=10^{-3}$ &(d) {\tt phishing}, $\lambda=10^{-5}$\\
				\includegraphics[width = 0.23 \textwidth]{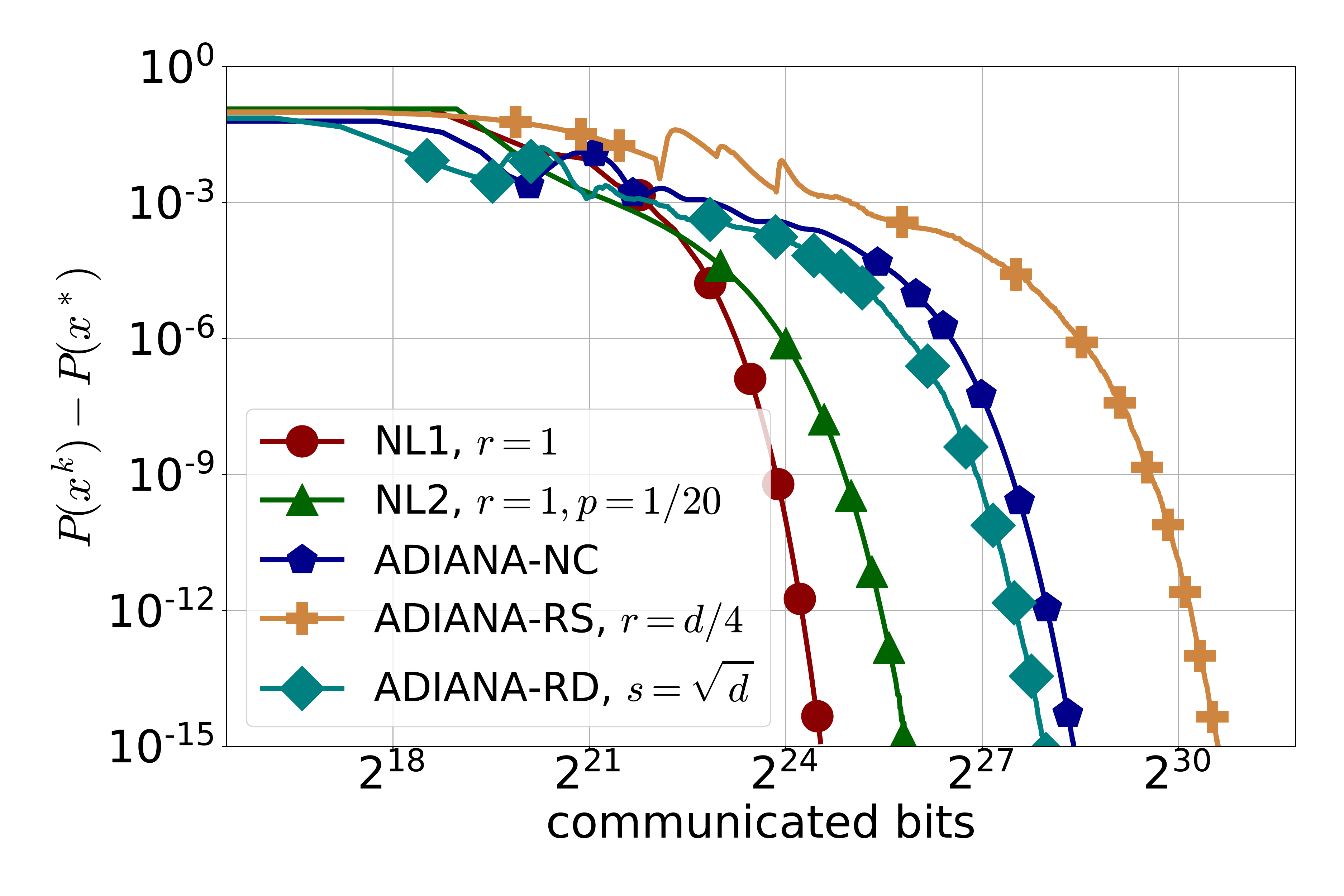}&
				\includegraphics[width = 0.23 \textwidth]{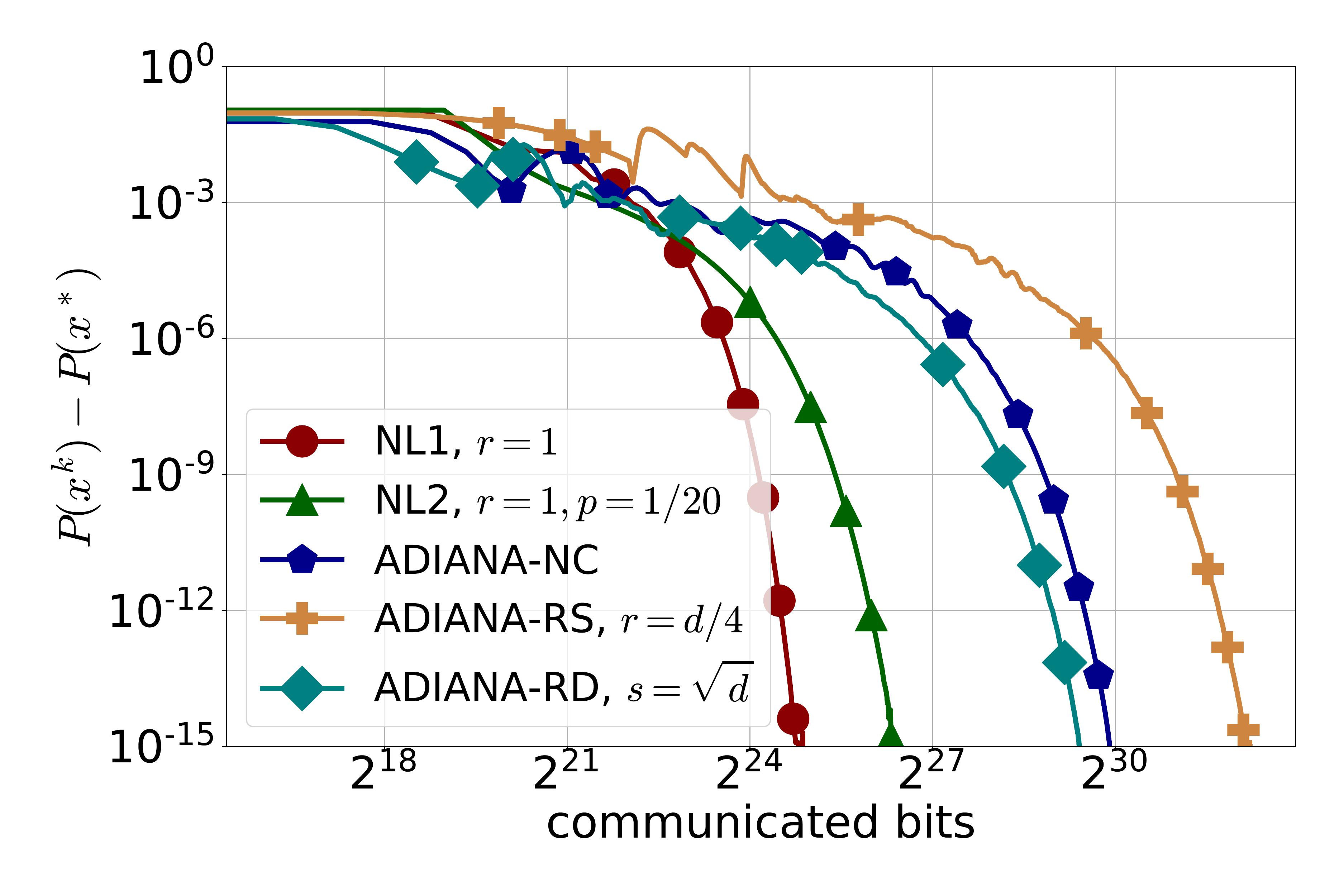}&
				\includegraphics[width = 0.23 \textwidth]{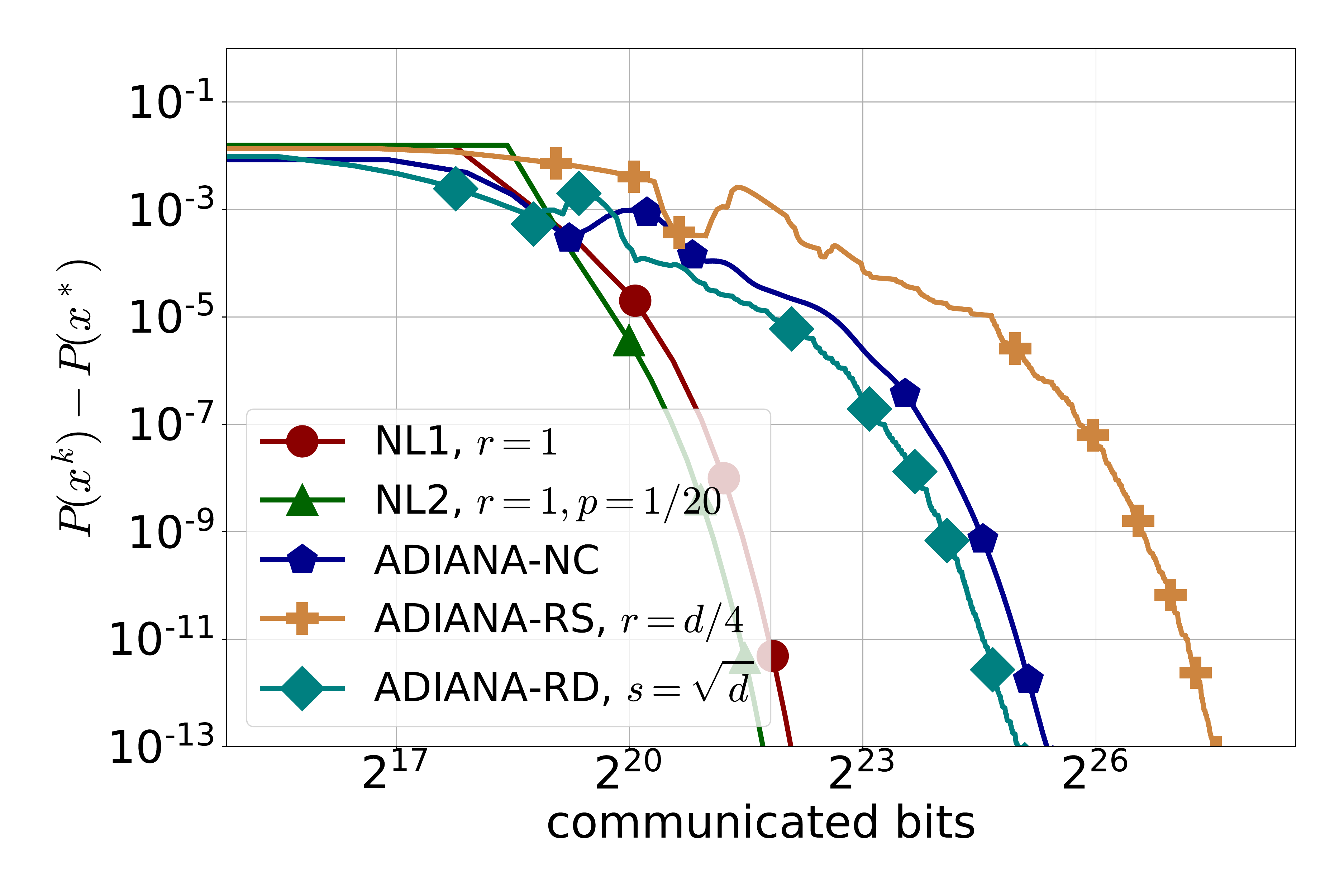}&
				\includegraphics[width = 0.23 \textwidth]{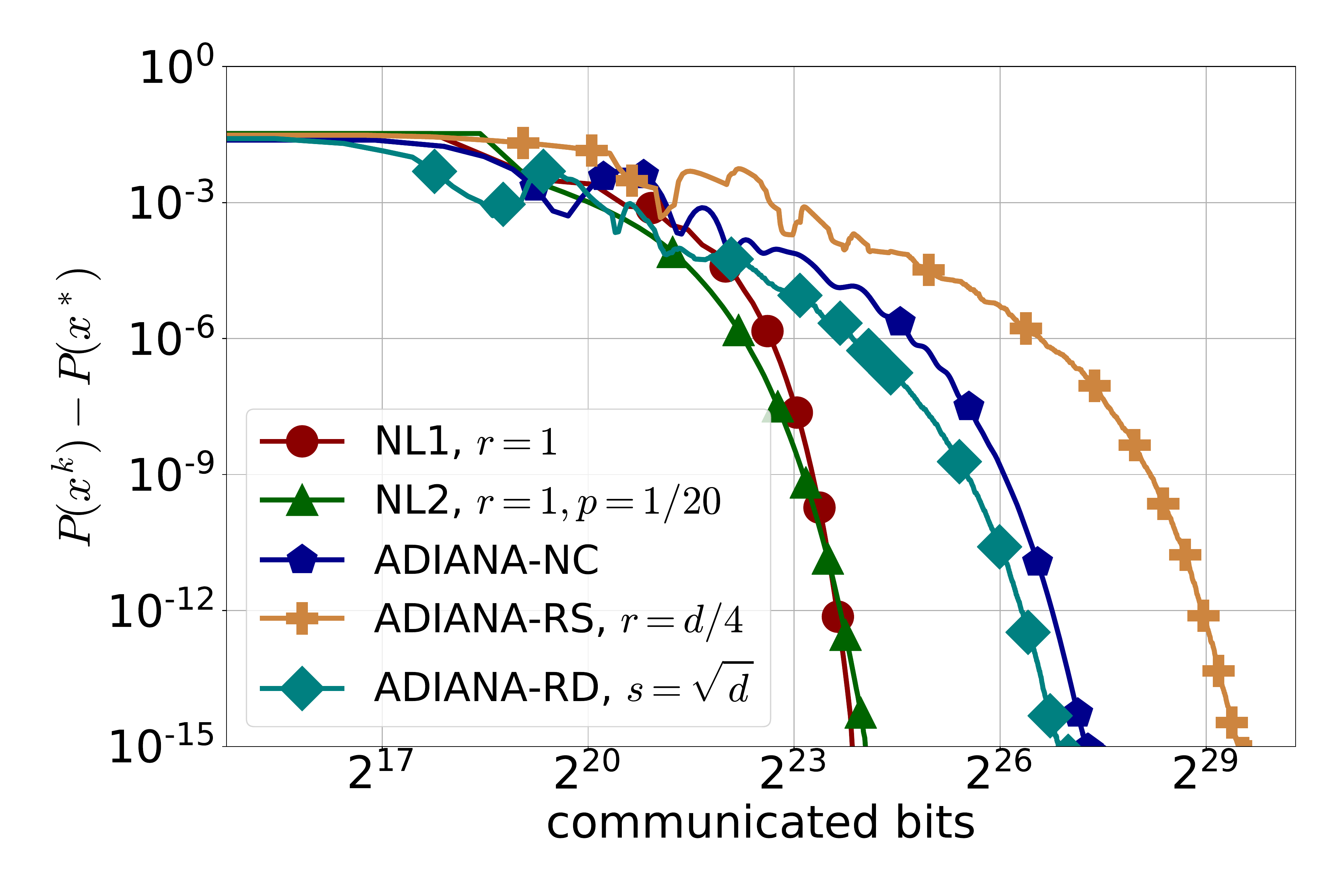}
				\\
				(e) {\tt a7a}, $\lambda=10^{-3}$ &(f) {\tt a7a}, $\lambda=10^{-4}$ & (g) {\tt phishing}, $\lambda=10^{-3}$ &(h) {\tt phishing}, $\lambda=10^{-4}$
		\end{tabular}}
		\caption{Comparison of {\sf NL1}, {\sf NL2} with ADIANA in terms of communication complexity.}
		\label{exp:adiana}
	\end{center}
\end{figure}

\subsection{Comparison of {\sf NL1} and {\sf NL2} with DINGO}

In our next experiment, we compare {\sf NL1} and {\sf NL2} with DINGO~\citep{DINGO2019}. The results, presented  in Figure~\ref{exp:dingo}, show that our methods are more communication efficient than DINGO  by {\em many orders of magnitude.}
This is true for all  experiments with the exception of Figure~\ref{exp:dingo} (c), where DINGO is slightly better than {\sf NL2}. 

\begin{figure}[ht]
	\begin{center}
		\centerline{\begin{tabular}{cccc}
				\includegraphics[width = 0.23 \textwidth]{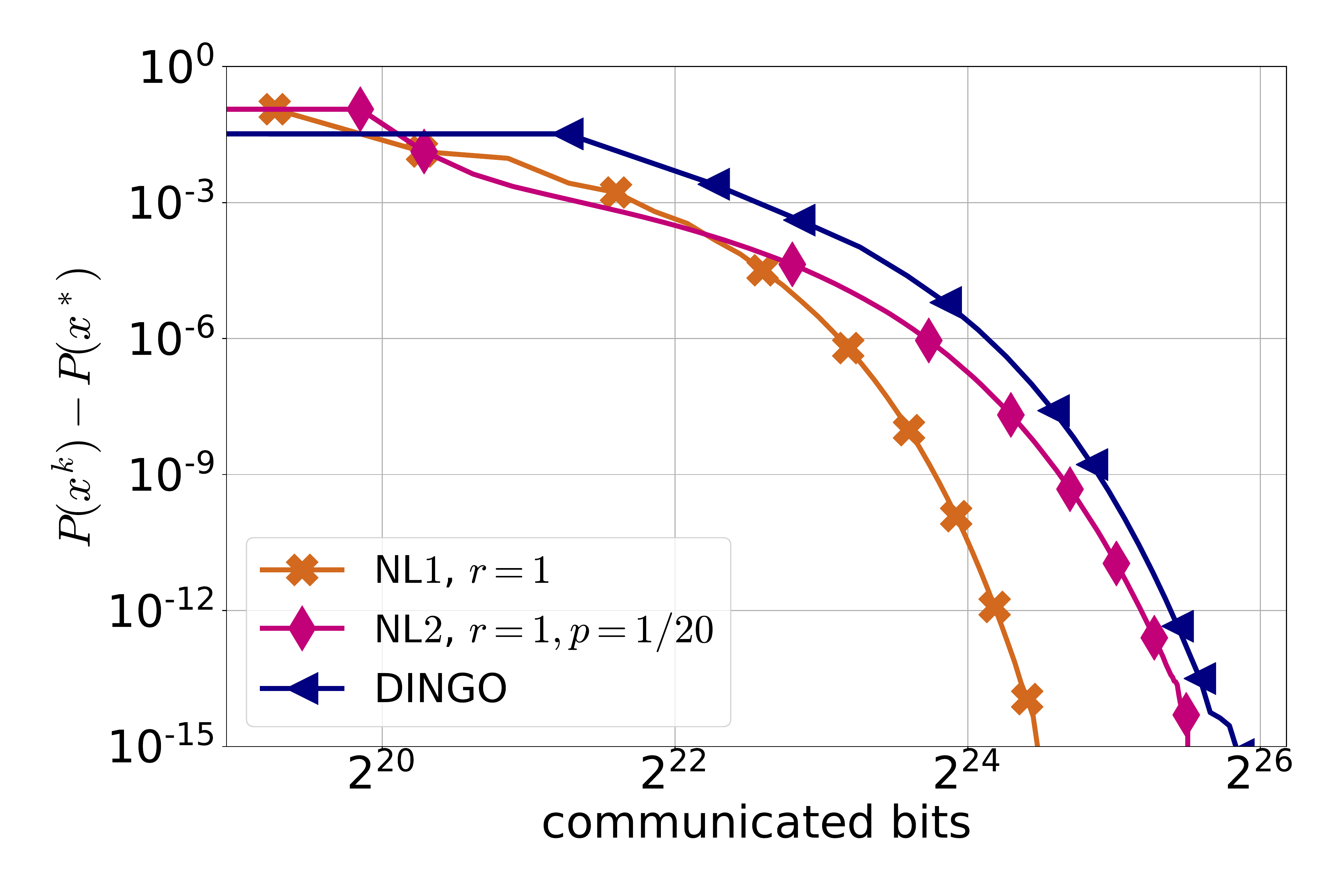}&
				\includegraphics[width = 0.23 \textwidth]{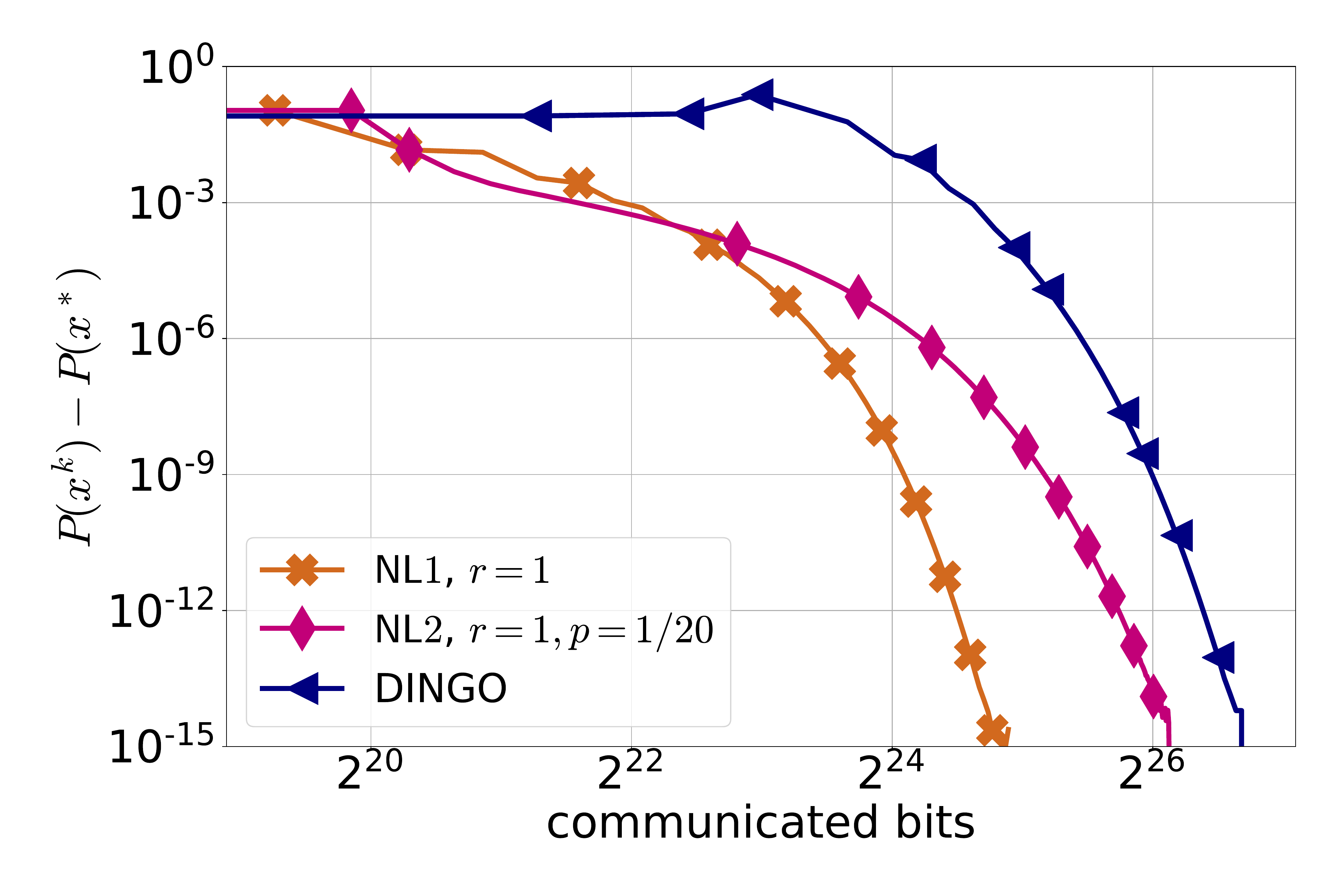}&
				\includegraphics[width = 0.23 \textwidth]{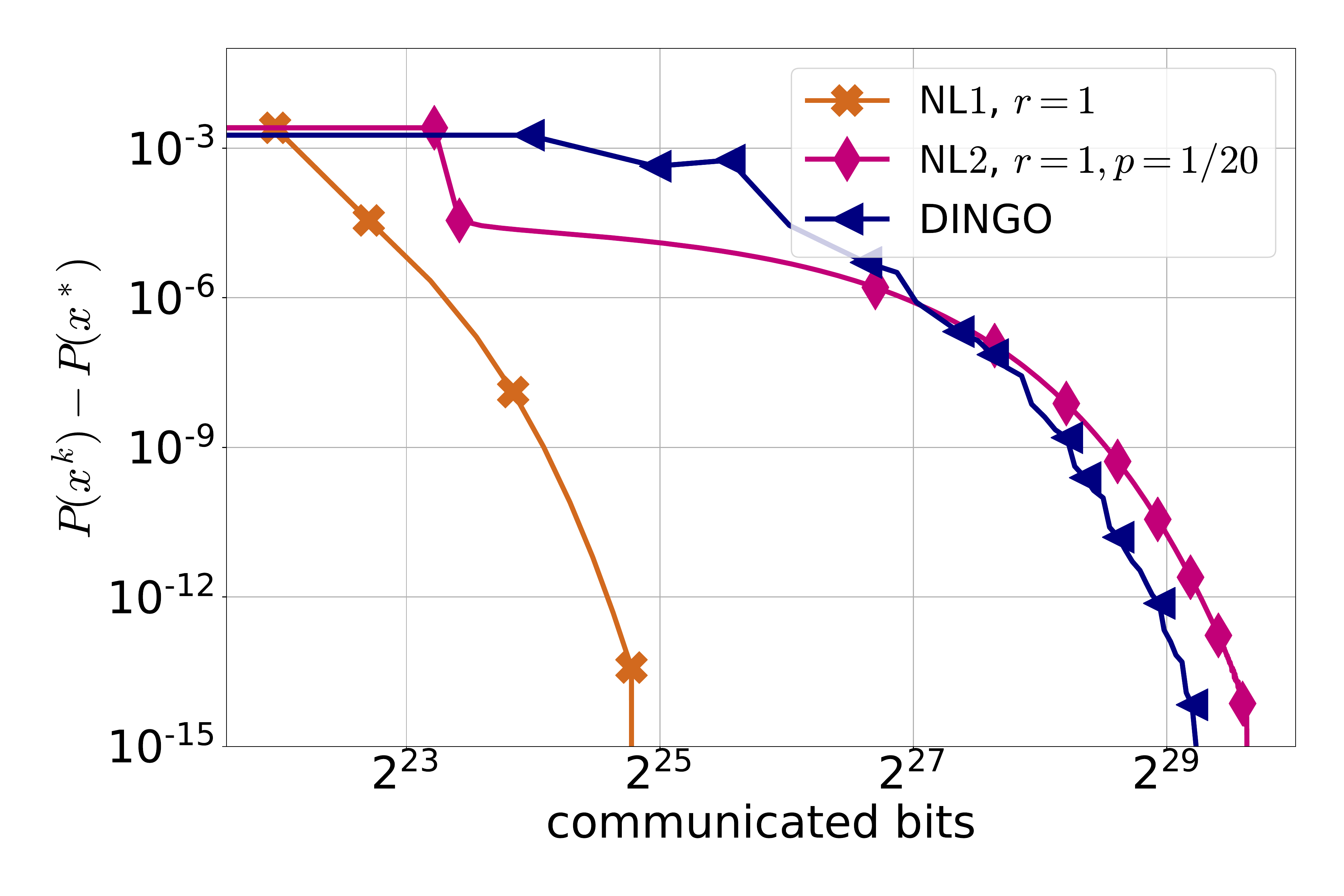}&
				\includegraphics[width = 0.23 \textwidth]{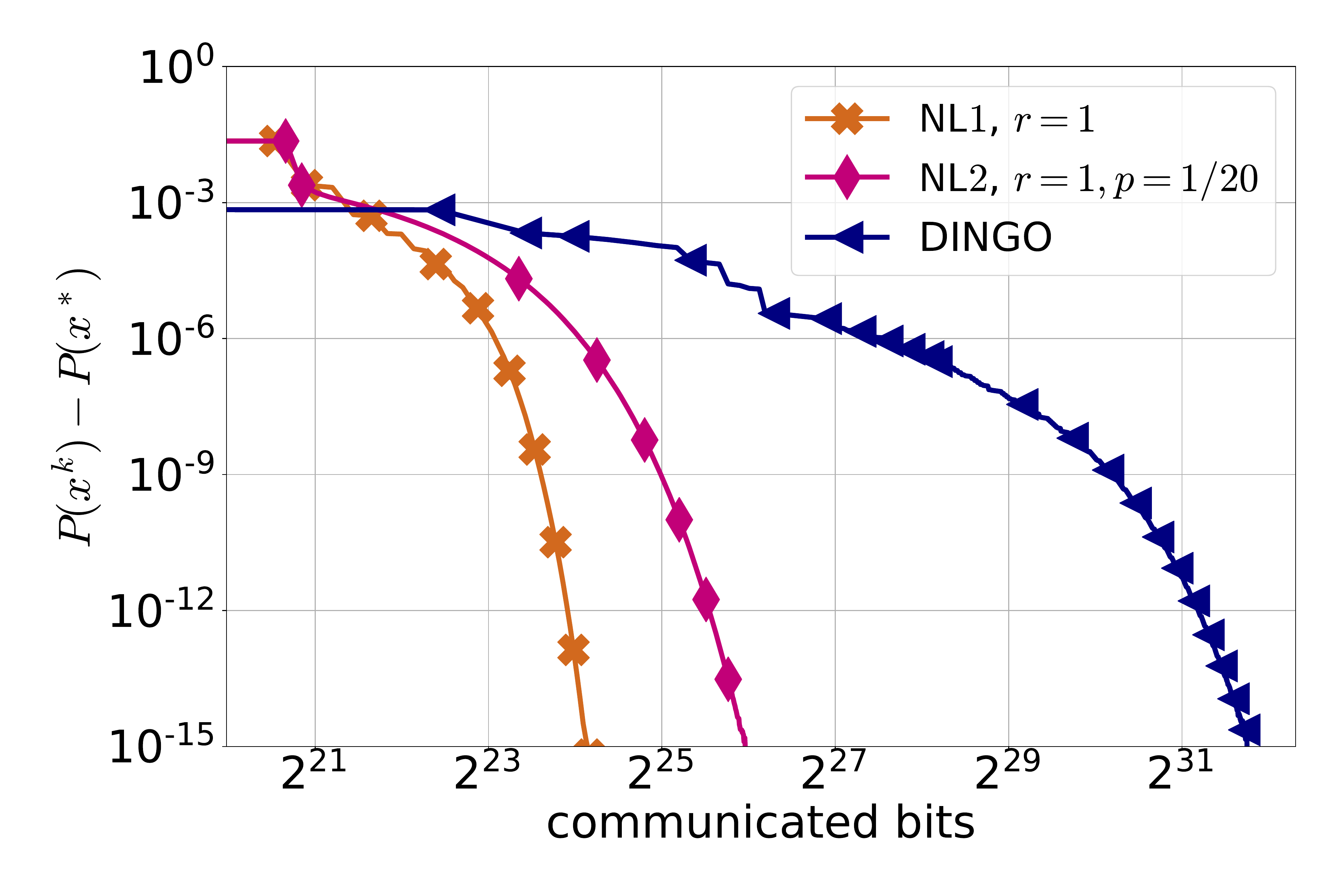}
				\\
				(a) {\tt a9a}, $\lambda=10^{-3}$ &(b) {\tt a9a}, $\lambda=10^{-4}$ & (c) {\tt w8a}, $\lambda=10^{-3}$ &(d) {\tt phishing}, $\lambda=10^{-5}$\\
				\includegraphics[width = 0.23 \textwidth]{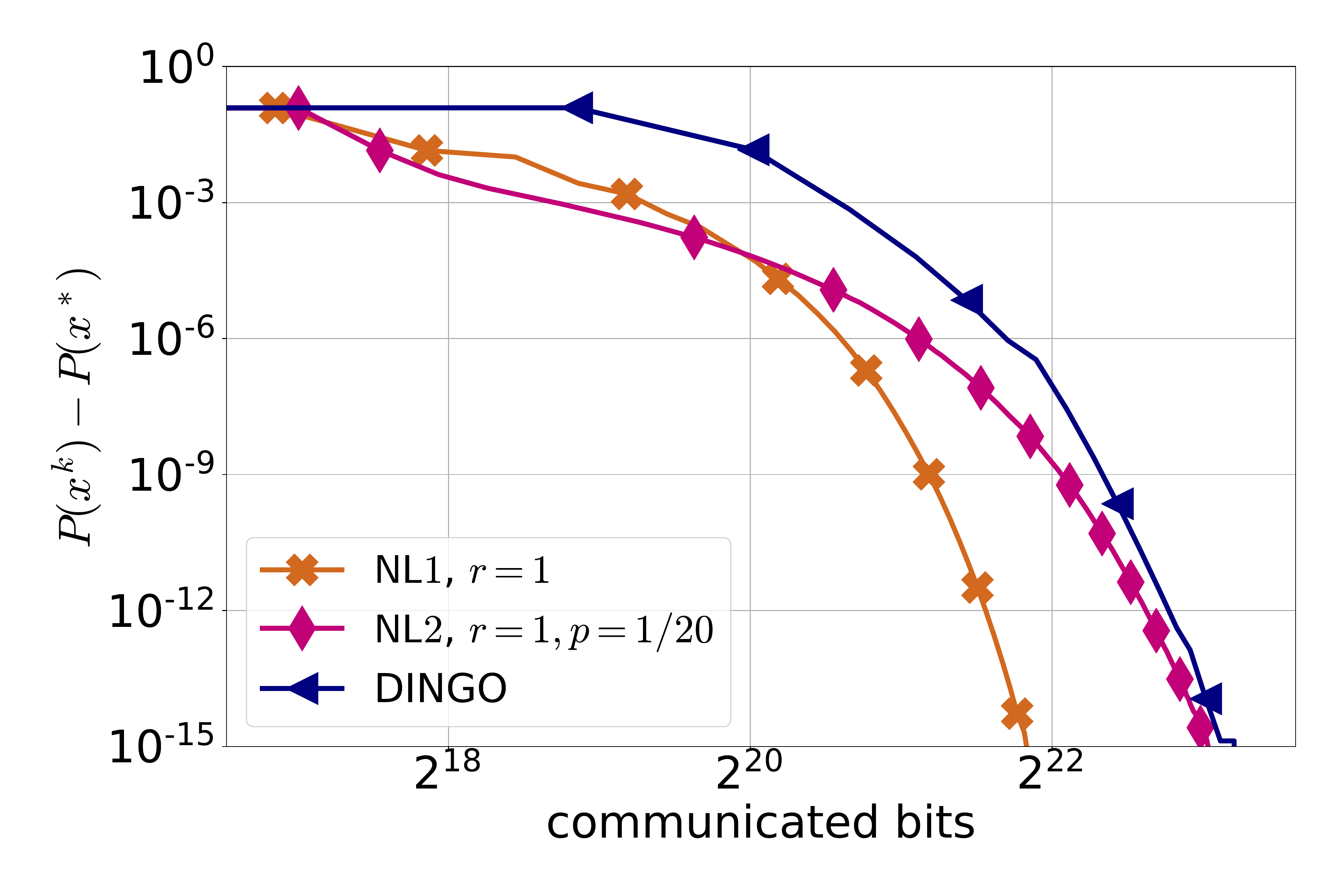}&
				\includegraphics[width = 0.23 \textwidth]{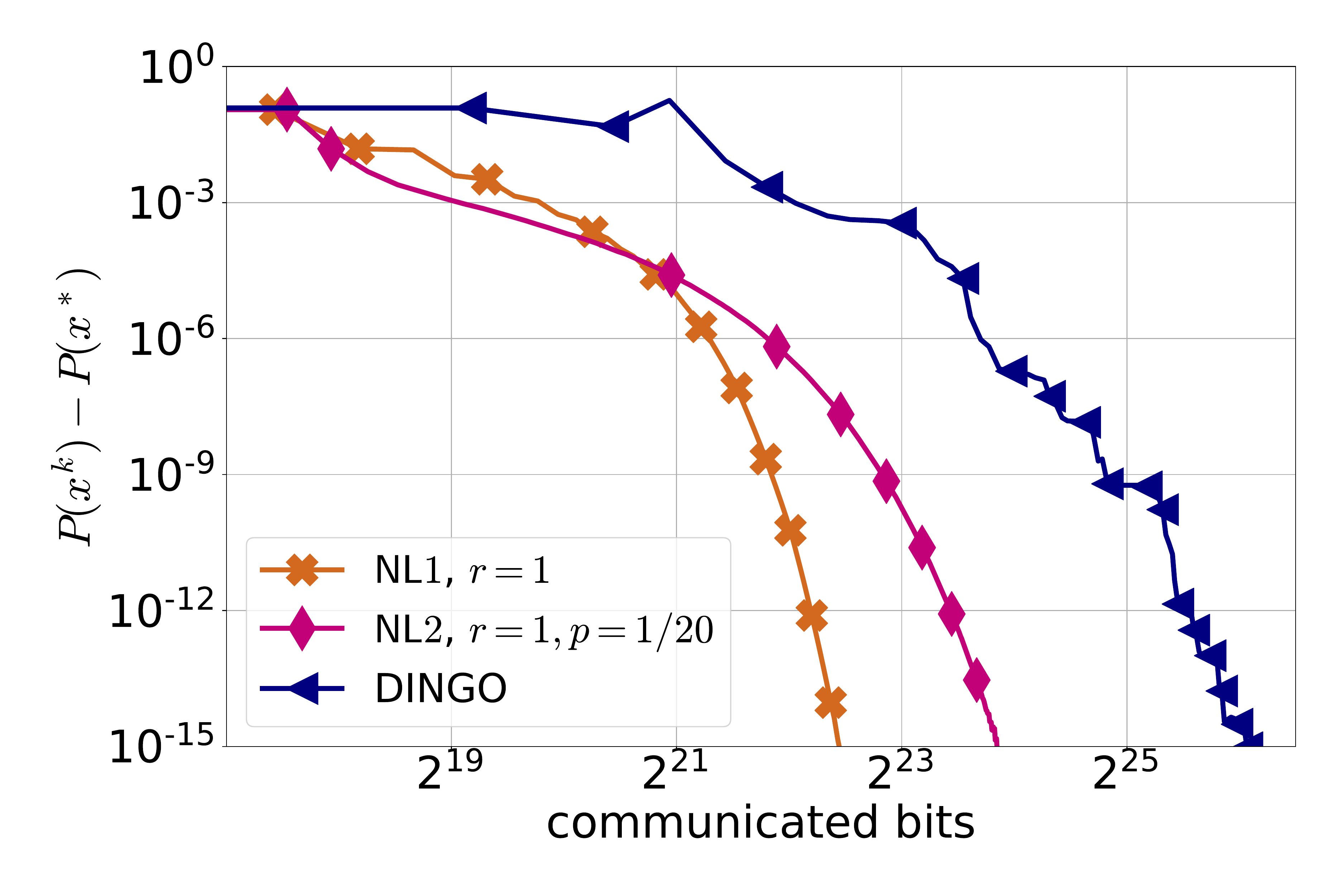}&
				\includegraphics[width = 0.23 \textwidth]{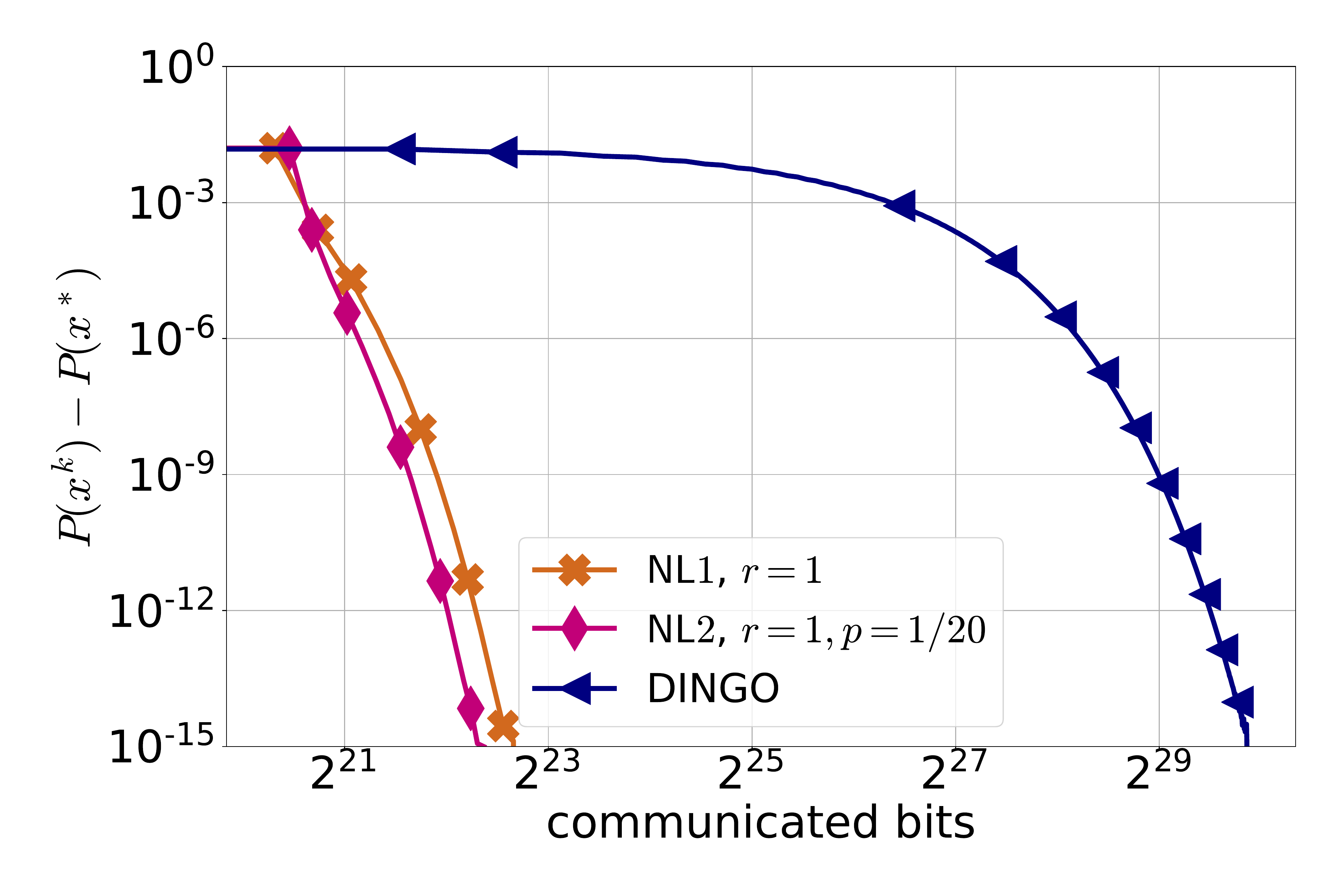}&
				\includegraphics[width = 0.23 \textwidth]{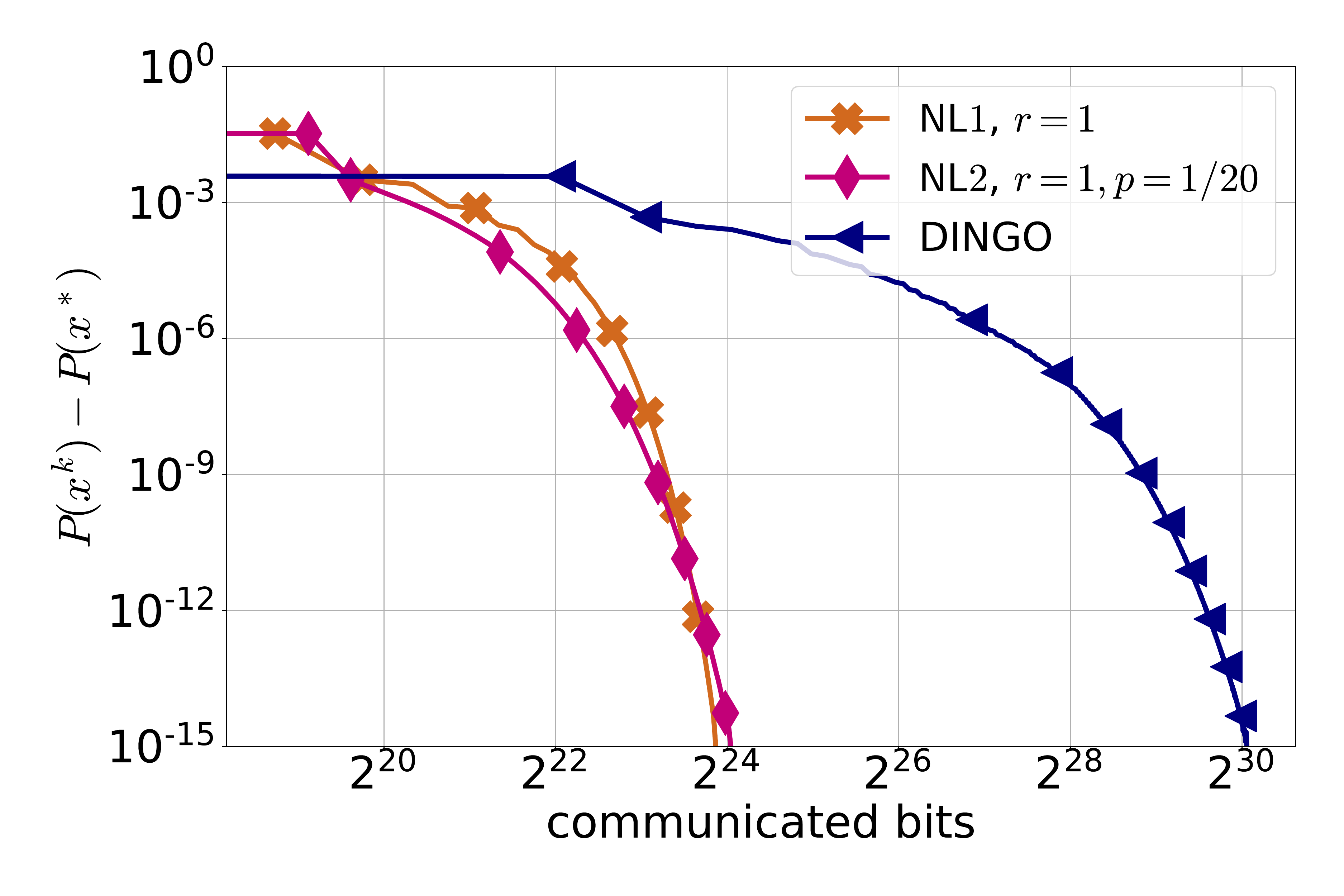}
				\\
				(e) {\tt a2a}, $\lambda=10^{-3}$ &(f) {\tt a2a}, $\lambda=10^{-4}$ & (g) {\tt phishing}, $\lambda=10^{-3}$ &(h) {\tt phishing}, $\lambda=10^{-4}$
		\end{tabular}}
		\caption{Comparison of {\sf NL1}, {\sf NL2} with ADIANA in terms of communication complexity.}
		\label{exp:dingo}
	\end{center}
\end{figure}

\subsection{Comparison of {\sf CNL} with  DCGD and DIANA}

We now deploy our {\sf CUBIC-NEWTON-LEARN (CNL)} method equipped with the random sparsification operator $\cC_p$ described in Section~\ref{sec:Bernoulli}. We compare {\sf CNL} against DIANA~\citep{DIANA} and DCGD~\citep{KFJ}, both in two variants: one using natural compression (NC) and one using random sparsification (RS, $r=d/4$). In Figure~\ref{exp:dcgd_diana} we see that for $\lambda=10^{-3}$, {\sf CNL} performs slightly better than the gradient-type methods DIANA and DCGD. However, once the regularization parameter becomes small enough ($\lambda=10^{-4}, 10^{-5}$), the increased condition number hurts the first-order methods, and {\sf CNL} outperforms both DCGD and DIANA significantly.

\begin{figure}[ht]
	\begin{center}
		\centerline{\begin{tabular}{cccc}
				\includegraphics[width = 0.23 \textwidth]{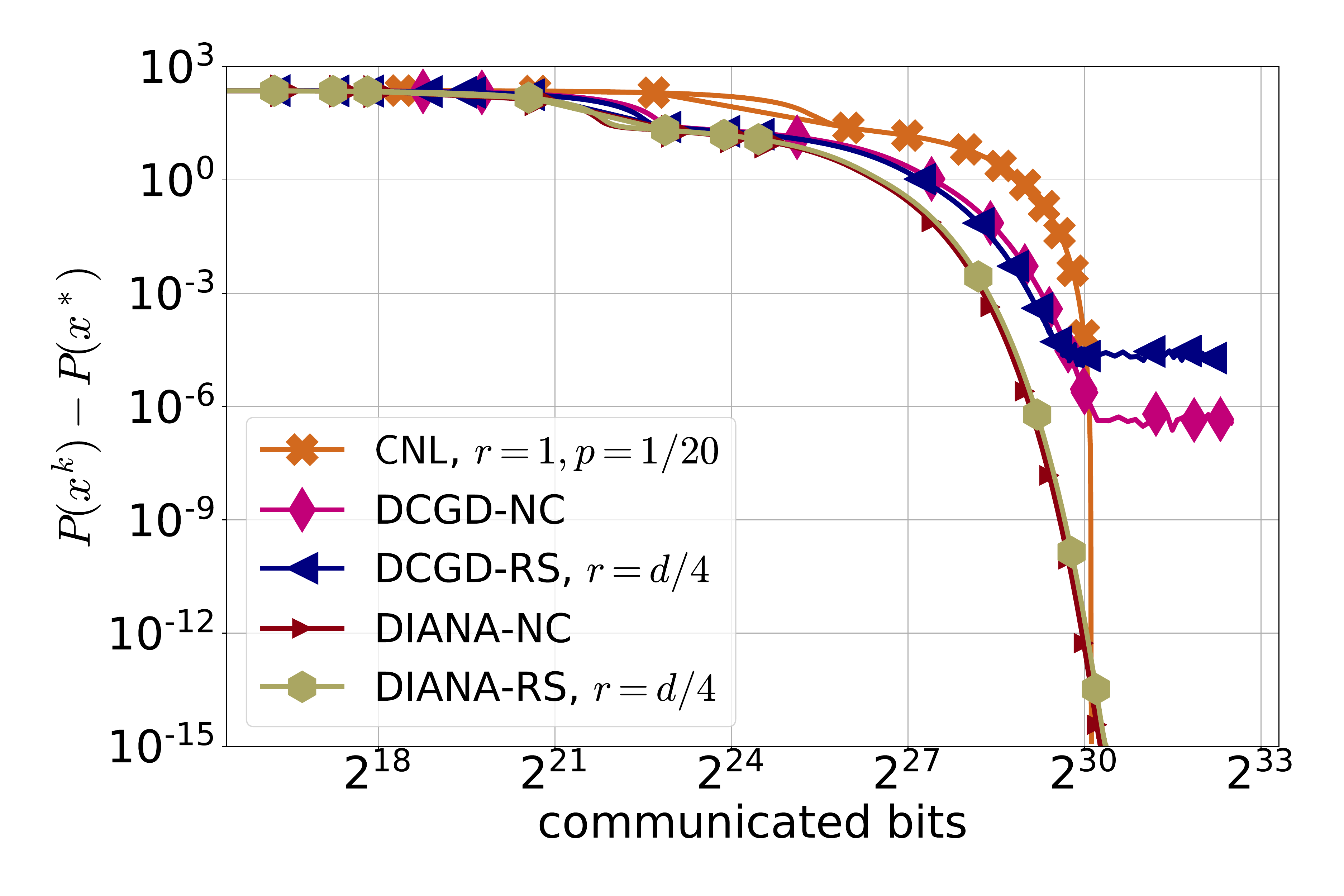}&
				\includegraphics[width = 0.23 \textwidth]{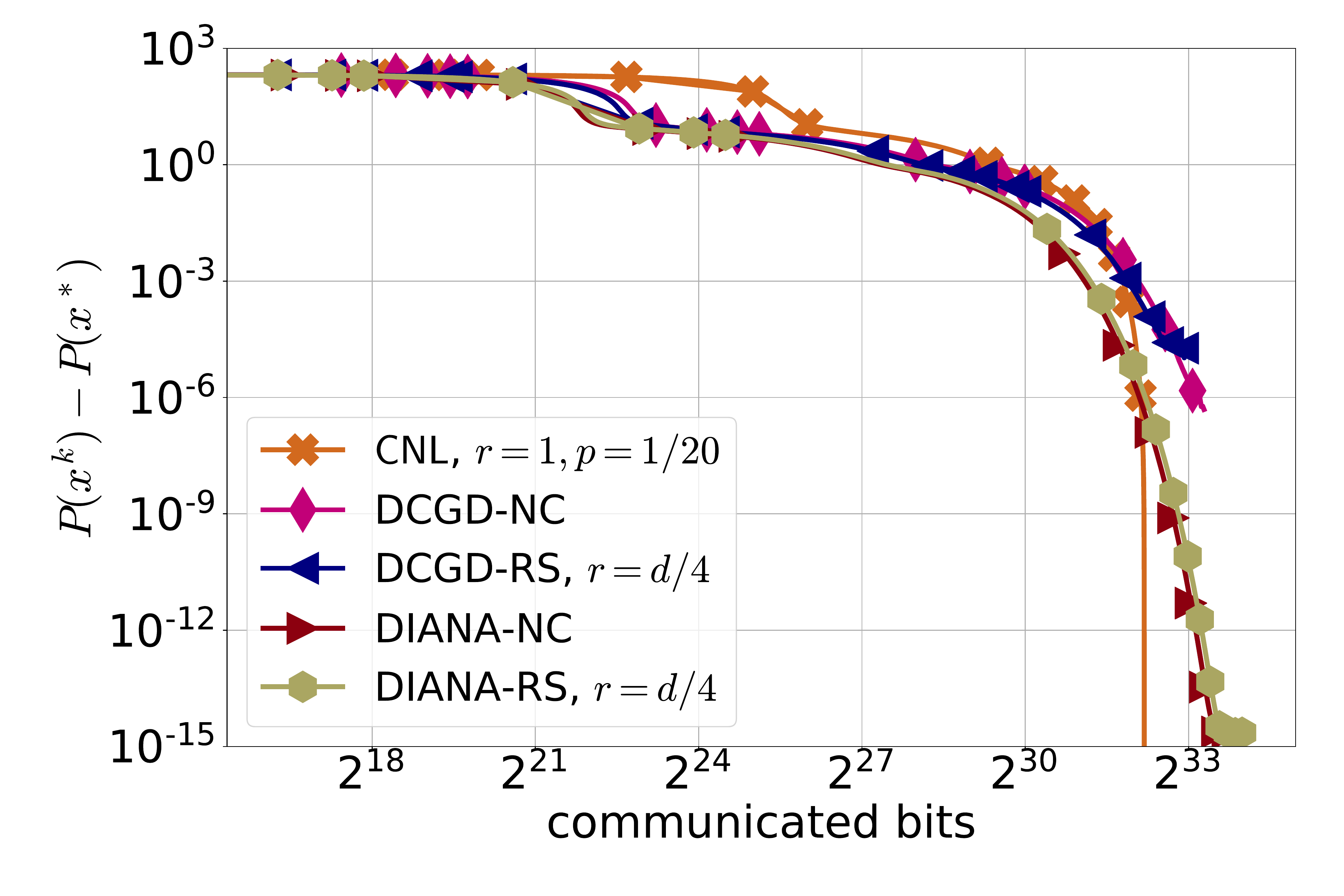}&
				\includegraphics[width = 0.23 \textwidth]{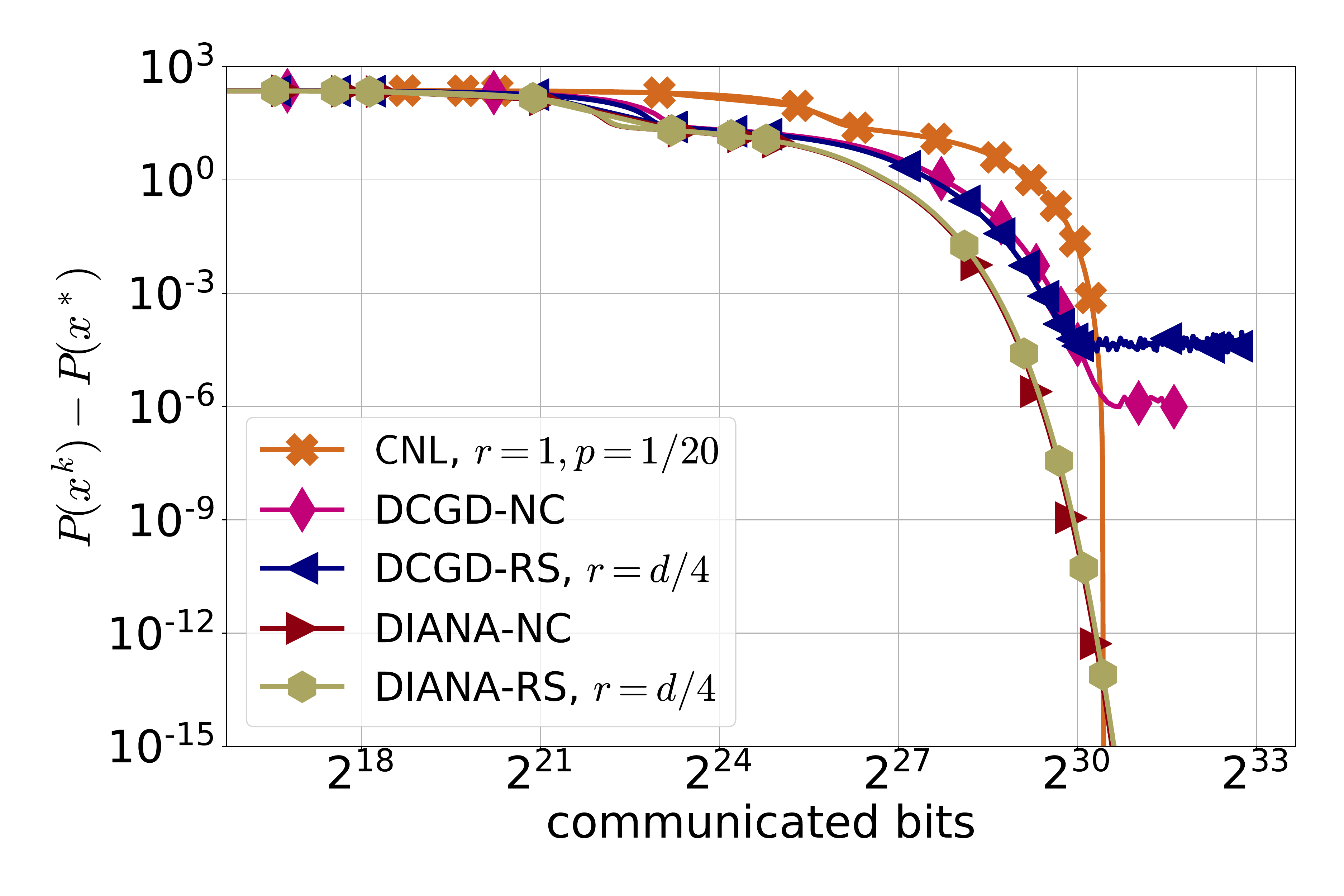}&
				\includegraphics[width = 0.23 \textwidth]{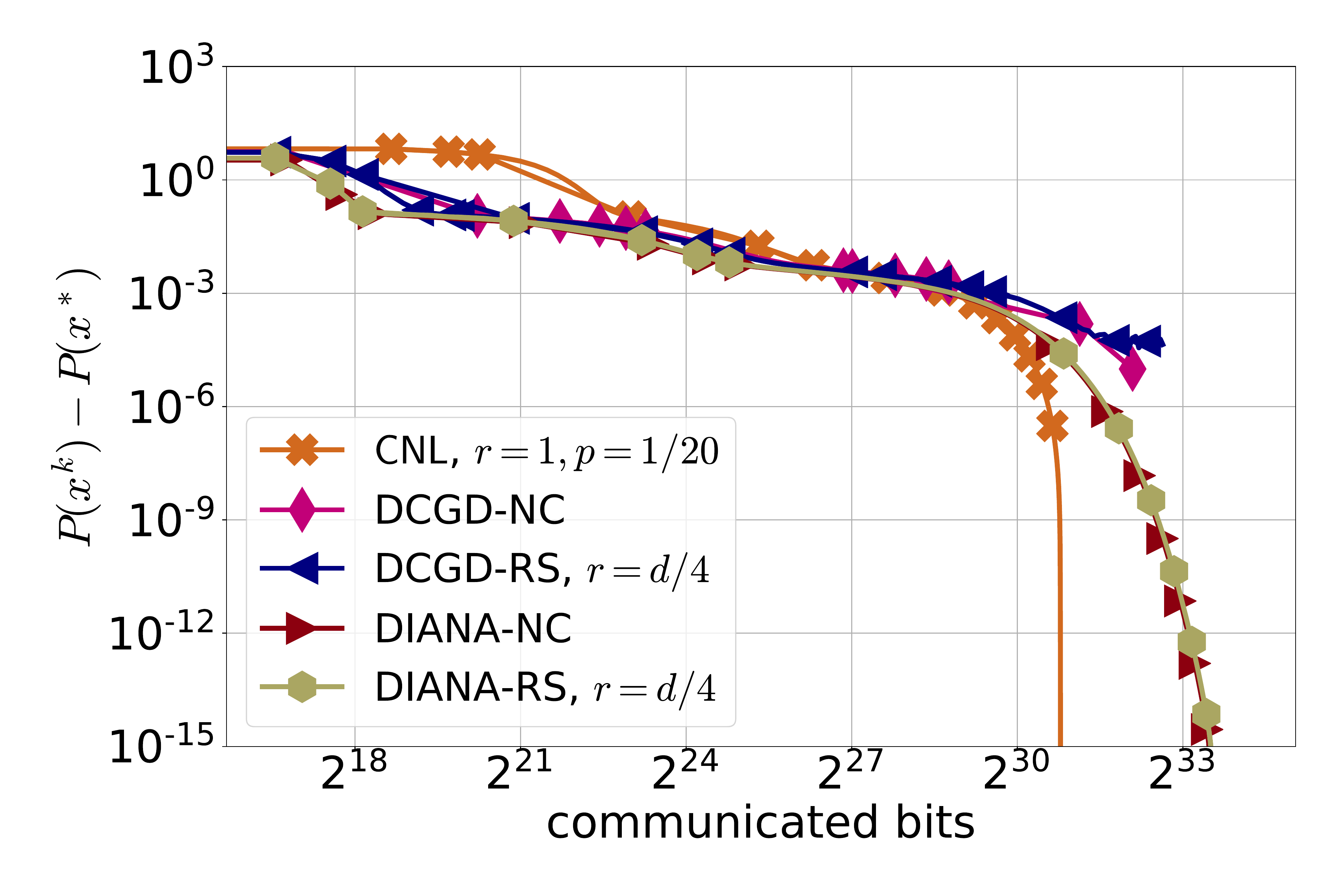}
				\\
				(a) {\tt a9a}, $\lambda=10^{-3}$ &(b) {\tt a9a}, $\lambda=10^{-4}$ & (c) {\tt a7a}, $\lambda=10^{-3}$ &(d) {\tt a7a}, $\lambda=10^{-4}$\\
				\includegraphics[width = 0.23 \textwidth]{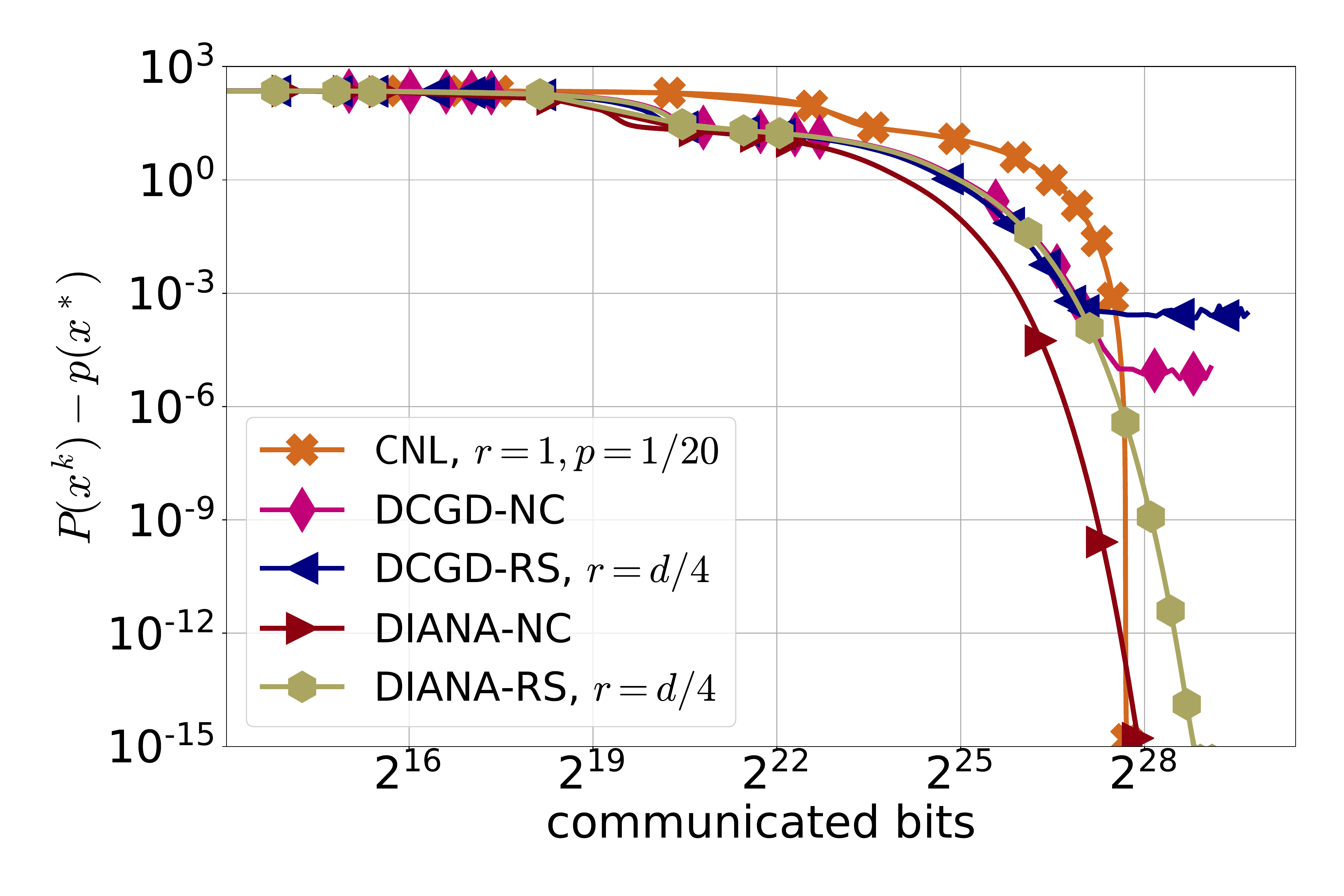}&
				\includegraphics[width = 0.23 \textwidth]{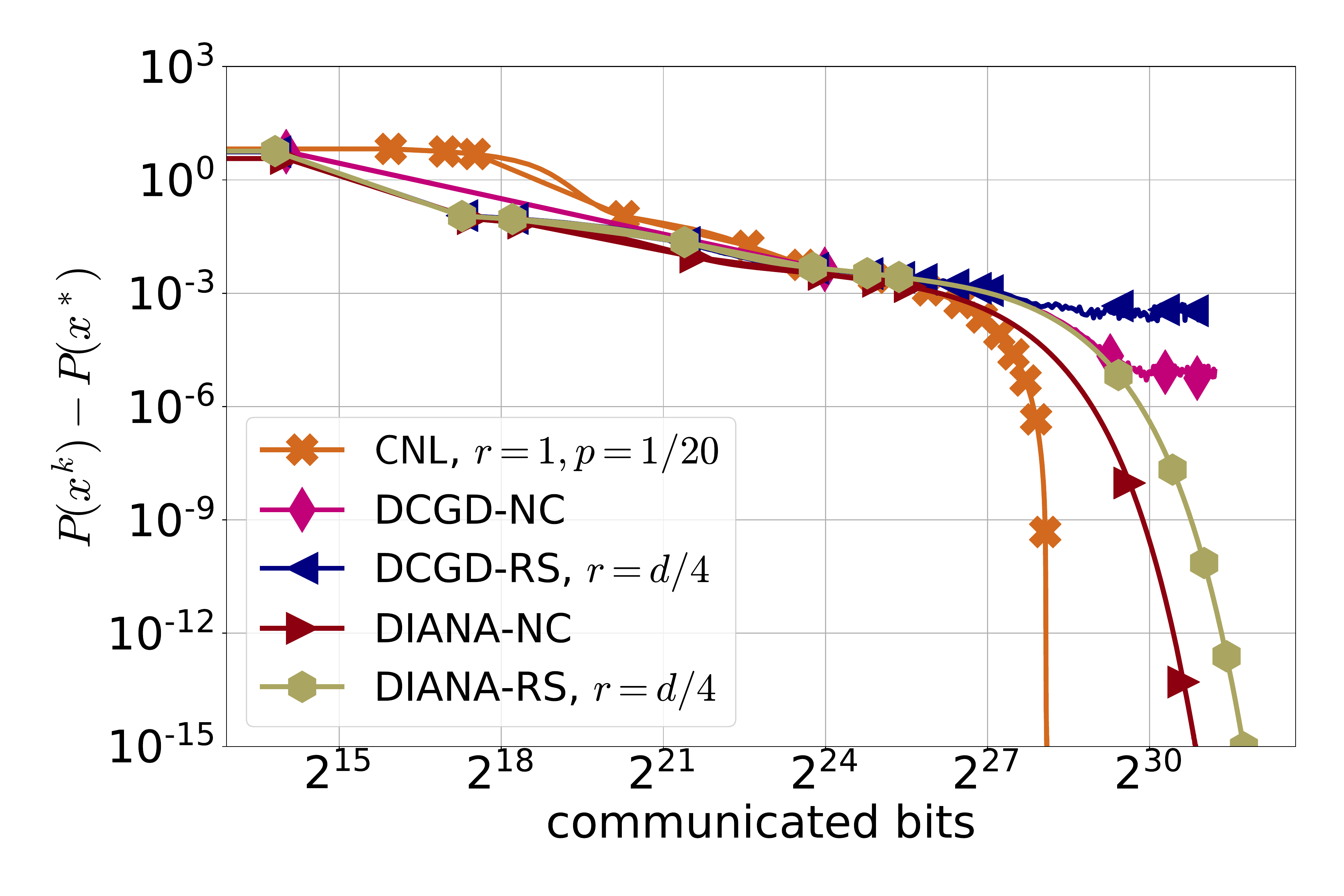}&
				\includegraphics[width = 0.23 \textwidth]{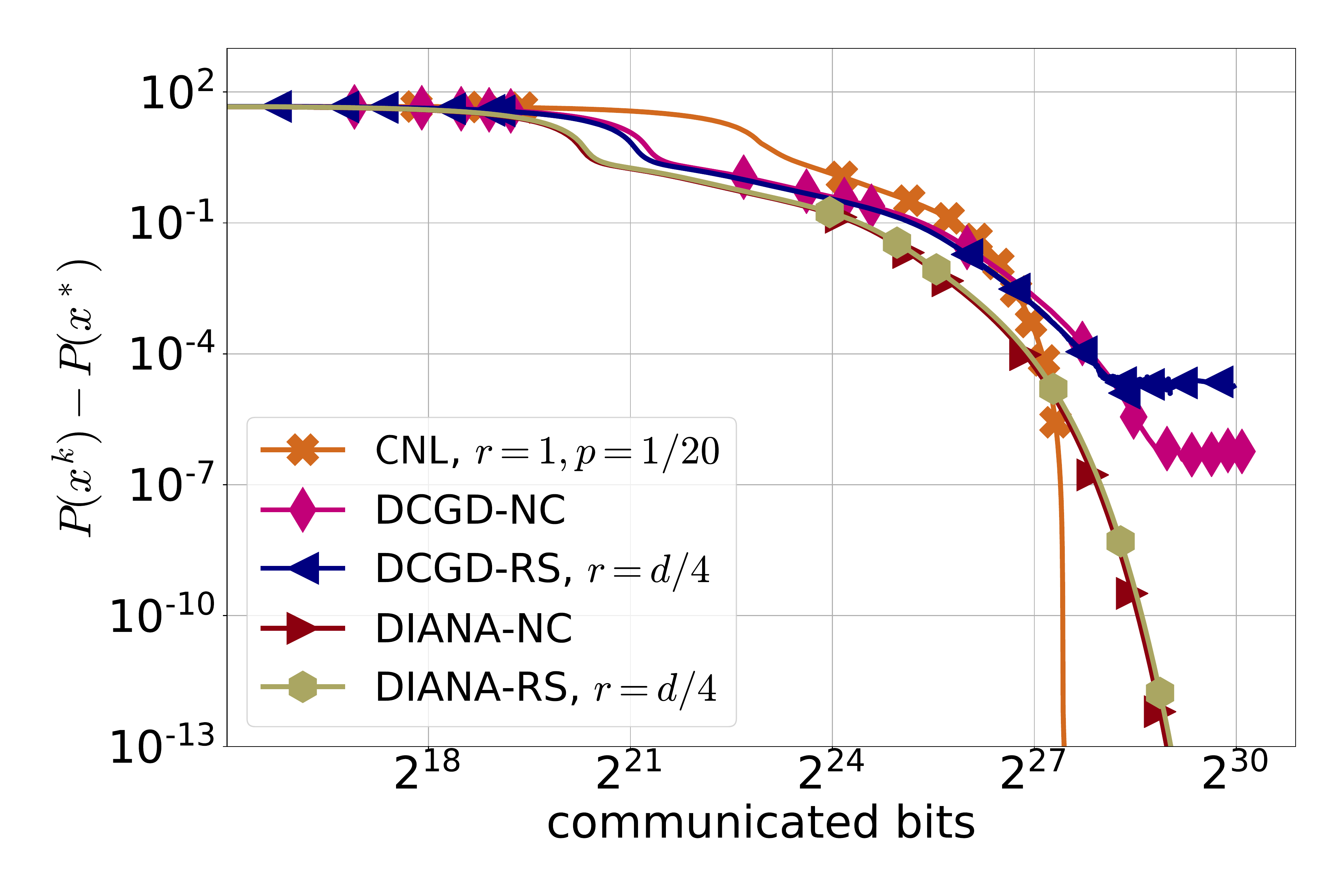}&
				\includegraphics[width = 0.23 \textwidth]{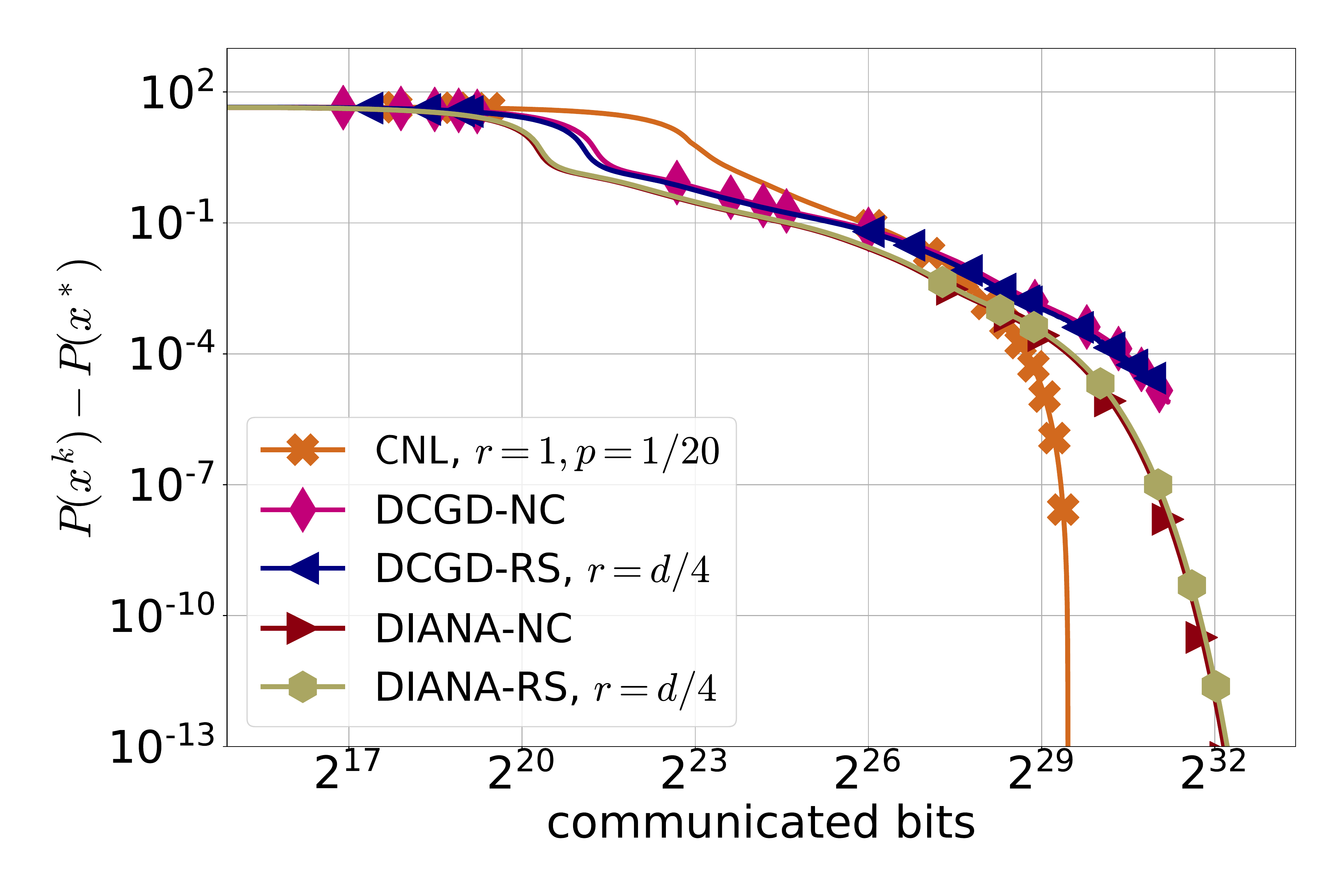}
				\\
				(e) {\tt a2a}, $\lambda=10^{-3}$ &(f) {\tt a2a}, $\lambda=10^{-4}$ & (g) {\tt phishing}, $\lambda=10^{-4}$ &(h) {\tt phishing}, $\lambda=10^{-5}$
		\end{tabular}}
		\caption{Comparison of {\sf CNL} with DCGD and DIANA in terms of communication complexity.}
		\label{exp:dcgd_diana}
	\end{center}
	\vskip -0.2in
\end{figure}

\clearpage

\bibliography{NL-arXiv}
\bibliographystyle{plainnat}

\clearpage

\appendix

\part*{Appendix}

\section{Proofs for {\sf NL1} (Section~\ref{subsec:NL1})}

Let $W^k \eqdef (h_1^k, \dots, h_n^k, x^k)$.

\subsection{Lemma}

We first establish a lemma.

\begin{lemma}\label{lm:calHk} Let $S^2\eqdef \sum_{i=1}^n \sum_{j=1}^m \norm{a_{ij}}^2$. For $\eta \leq \frac{1}{\omega+1}$ and all $k\geq 0$, we have 
$$
\ExpBr{ {\cal H}^{k+1} \;|\; W^k} \leq  (1-\eta) {\cal H}^k +  \eta \nu^2 S^2 \norm{x^k - x^*}^2.
$$
\end{lemma}

\begin{proof}
First, recall that
\begin{equation}\label{eq:bu98fg9d_087f98dhf}{\cal H}^{k} =   \sum_{i=1}^n \norm{h_i^{k} - h_i(x^*)}^2 .\end{equation}Since $h_{i}(x^*) \in \R^m_+$ due to convexity of $f_{ij}$, we have \begin{equation} \label{ew:ub98g9fdu_09u9fdf}\norm{[x]_+ - h_i(x^*)}^2 \leq \norm{x - h_i(x^*)}^2, \qquad \text{for all} \qquad x \in \R^m.\end{equation} Then as long as $\eta \leq \frac{1}{\omega+1}$, we have 
\begin{eqnarray}
\ExpBr{{\cal H}^{k+1} \;|\; W^k}	&=& \ExpBr{  \sum_{i=1}^n \norm{h_i^{k+1} - h_i(x^*)}^2  \;|\; W^k} \notag \\ 
	&=&  \sum_{i=1}^n\ExpBr{  \norm{ [h^k_i + \eta \cC_i^k(h_i(x^k) - h^k_i)]_+ - h_i(x^*) }^2  \;|\; W^k} \notag  \\ 
	&\overset{\eqref{ew:ub98g9fdu_09u9fdf}}{\leq} &    \sum_{i=1}^n \ExpBr{\norm{ h^k_i + \eta \cC_i^k(h_i(x^k) - h^k_i) - h_i(x^*) }^2 \;|\; W^k }\notag  \\ 
	&=& \sum_{i=1}^n  \norm{ h_i^k - h_i(x^*) }^2   + 2 \eta\sum_{i=1}^n \ExpBr{  \langle  \cC_i^k(h_i(x^k) - h_i^k),  h_i^k - h_i(x^*) \rangle \;|\; W^k}\notag  \\
	&& \qquad + \eta^2 \sum_{i=1}^n \ExpBr{ \norm{\cC_i^k(h_i(x^k) - h_i^k) }^2 \;|\; W^k} \notag \\ 
	&\overset{\eqref{eq:bu98fg9d_087f98dhf}+\eqref{eq:unbiased} + \eqref{eq:omega-variance}}{\leq} &  {\cal H}^k  + 2\eta \sum_{i=1}^n    \langle  h_i(x^k) - h_i^k, h_i^k - h_i(x^*) \rangle  + \eta^2 \sum_{i=1}^n  (\omega+1) \|h_i(x^k) - h_i^k\|^2 .\label{eq:nbu98fd-09ud9f8hufd-fd}
	\end{eqnarray}
	
Using the stepsize restriction $\eta \leq \frac{1}{\omega+1}$, we can bound $\eta^2 (\omega+1) \leq \eta$. Plugging this back in  \eqref{eq:nbu98fd-09ud9f8hufd-fd}, we get
\begin{eqnarray*}
\ExpBr{{\cal H}^{k+1} \;|\; W^k}
	&\leq &  {\cal H}^k  + \eta  \sum_{i=1}^n   \langle h_i(x^k) - h_i^k, h_i(x^k) + h_i^k - 2h_i(x^*) \rangle  \\ 
	&=&  {\cal H}^k  + \eta \sum_{i=1}^n  \left( \norm{h_i(x^k) - h_i(x^*)}^2 - \norm{h_i^k - h_i(x^*)}^2\right)   \\ 
	&=& (1-\eta) {\cal H}^k  + \eta \sum_{i=1}^n  \norm{h_i(x^k) - h_i(x^*)}^2  \\ 
		&\overset{\eqref{eq:8f0d8hfd}}{=} & (1-\eta) {\cal H}^k  + \eta \sum_{i=1}^n  \sum_{j=1}^m (h_{ij}(x^k) - h_{ij}(x^*))^2  \\ 
	&\overset{(\ref{eq:alphaijL})}{\leq}&  (1-\eta) {\cal H}^k + \eta\sum_{i=1}^n \sum_{j=1}^m \nu^2 \norm{a_{ij}}^2 \norm{x^k - x^*}^2 \\ 
	&\leq&  (1-\eta) {\cal H}^k +  \eta \nu^2 S^2 \norm{x^k - x^*}^2. 
\end{eqnarray*}

\end{proof}

\subsection{Proof of Theorem~\ref{th:lambda>0}}

It is easy to see that \begin{equation}\label{eq:b97gfd-ujo_09op}\mH^k = \frac{1}{n} \sum_{i=1}^n \mH^k_i, \qquad \mH^k_i = \frac{1}{m} \sum_{j=1}^{m} h^k_{ij} a_{ij}a_{ij}^\top .\end{equation}
By the first order necessary optimality conditions, we have
\begin{equation}\label{eq:optx}
\nabla f(x^*) + \lambda x^*  = \nabla P(x^*)= 0. 
\end{equation}
Furthermore, since we maintain $\mH^k \succeq 0$ for all $k$, we have $\mH^k + \lambda \mI \succeq \lambda \mI$, whence  \begin{equation}\label{eq:89f8gd9=08yh98fd}\norm{\left(\mH^k + \lambda \mI \right)^{-1}} \leq \frac{1}{\lambda}.\end{equation}
Then we can write
\begin{eqnarray}
	\norm{ x^{k+1} - x^* } &\overset{\eqref{eq:Newton-learn}}{=}& \norm{ x^k - x^* - (\mH^k + \lambda \mI)^{-1} (\nabla f(x^k) + \lambda x^k) }\notag  \\
	&=& \norm{ (\mH^k + \lambda \mI)^{-1} \left(  (\mH^k + \lambda \mI)(x^k - x^*) - \nabla f(x^k) - \lambda x^k  \right) } \notag \\ 
	&\overset{\eqref{eq:89f8gd9=08yh98fd}}{\leq}& \frac{1}{\lambda} \norm{  (\mH^k + \lambda \mI)(x^k - x^*) - \nabla f(x^k) - \lambda x^k  } \notag \\ 
	&\overset{(\ref{eq:optx})}{=}& \frac{1}{\lambda} \norm{ \mH^k (x^k - x^*) - (\nabla f(x^k) - \nabla f(x^*)) } \notag \\ 
	&=& \frac{1}{\lambda} \norm{ \frac{1}{n} \sum_{i=1}^n \left[\mH_i^k (x^k - x^*) - (\nabla f_i(x^k) - \nabla f_i(x^*)) \right]}\notag  \\ 
	&\leq& \frac{1}{n \lambda} \sum_{i=1}^n \norm{ \mH_i^k (x^k - x^*) - (\nabla f_i(x^k) - \nabla f_i(x^*)) } \notag  \\
	&\overset{\eqref{eq:b97gfd-ujo_09op}+\eqref{eq:f_and_f_i} }{=}& \frac{1}{n \lambda} \sum_{i=1}^n \norm{ \frac{1}{m} \sum_{j=1}^m  \left[ h_{ij}^k a_{ij} a_{ij}^\top (x^k - x^*) - (\nabla f_{ij}(x^k) - \nabla f_{ij}(x^*)) \right] },\label{eq:pop_us_lop_989f}
	\end{eqnarray}
	where the last step follows from applying Jensen's inequality to he convex function $x\mapsto \norm{x}$. 

We can now express the difference of the gradients in integral form via the fundamental theorem of calculus, obtaining 
\begin{eqnarray*} \nabla f_{ij}(x^k) - \nabla f_{ij}(x^*) &=&   \int_{0}^1 \mH_{ij} (x^* + \tau(x^k - x^*)) (x^k - x^*) \; d\tau \\
&\overset{\eqref{eq:87ybfd0fd}}{=} & \int_{0}^1 h_{ij}(x^* + \tau(x^k - x^*)) a_{ij}a_{ij}^\top  (x^k - x^*)\; d\tau .\end{eqnarray*}
We can plug this back into \eqref{eq:pop_us_lop_989f}, which gives
\begin{eqnarray}	
\norm{ x^{k+1} - x^* }	&\leq & 
 \frac{1}{n \lambda} \sum_{i=1}^n \frac{1}{m} \norm{  \sum_{j=1}^{m} \left(  h_{ij}^k a_{ij}a_{ij}^\top (x^k - x^*)  -  \int_{0}^1 h_{ij} (x^* + \tau(x^k - x^*))a_{ij}a_{ij}^\top (x^k - x^*)d\tau  \right)   } \notag \\ 
	&=& \frac{1}{n \lambda} \sum_{i=1}^n \frac{1}{m} \norm{  \sum_{j=1}^{m}  a_{ij}a_{ij}^\top (x^k - x^*) \left( h_{ij}^k - \int_{0}^1 h_{ij}(x^* + \tau(x^k - x^*))  d\tau \right) }\notag \\ 
	&\leq& \frac{ \norm{x^k - x^*}}{ \lambda} \frac{1}{nm} \sum_{i=1}^n \sum_{j=1}^{m} \norm{a_{ij}}^2 \left|   \int_{0}^1 h_{ij}^k - h_{ij}(x^* + \tau(x^k - x^*))  d\tau  \right|. \label{petrol-09809fd}
\end{eqnarray}

From (\ref{eq:alphaijL}), we have 
\begin{eqnarray}
	\left|   \int_{0}^1 h_{ij}^k - h_{ij}(x^* + \tau(x^k - x^*))  d\tau  \right| &\leq& \int_{0}^1 \left|  h_{ij}^k - h_{ij}(x^* + \tau(x^k - x^*))   \right| d\tau \notag \\
	&\leq& |h_{ij}^k - h_{ij}(x^*)| + \int_{0}^1 \left|  h_{ij}(x^*) -  h_{ij}(x^* + \tau(x^k - x^*))   \right| d \tau \notag \\ 
	&\overset{(\ref{eq:alphaijL})}{\leq}&  |h_{ij}^k - h_{ij}(x^*)| + \int_{0}^1 \tau \nu \|a_{ij}\| \cdot \|x^k - x^*\| d\tau  \notag \\ 
	&=&  |h_{ij}^k - h_{ij}(x^*)| + \frac{\nu \|a_{ij}\|}{2} \|x^k-x^*\|. \label{eq:olive-09u0hfdih}
\end{eqnarray}

By squaring both sides of \eqref{petrol-09809fd},  applying Jensen's inequality to the function $t\mapsto t^2$ in the form $\left(\frac{1}{nm}\sum_i \sum_j t_{ij}\right)^2 \leq \frac{1}{nm}\sum_i \sum_j t^2_{ij}$, and plugging in the bound \eqref{eq:olive-09u0hfdih}, we get
\begin{eqnarray*}	
\norm{ x^{k+1} - x^* }^2	&\overset{}\leq &  \frac{ \norm{x^k - x^*}^2}{ \lambda^2} \left(\frac{1}{nm} \sum_{i=1}^n \sum_{j=1}^{m} \norm{a_{ij}}^2 \left|   \int_{0}^1 h_{ij}^k - h_{ij}(x^* + \tau(x^k - x^*))  d\tau  \right| \right)^2 \\
&\leq &\frac{ \norm{x^k - x^*}^2}{ \lambda^2} \frac{1}{nm} \sum_{i=1}^n \sum_{j=1}^{m} \left( \norm{a_{ij}}^2  \left| \int_{0}^1 h_{ij}^k - h_{ij}(x^* + \tau(x^k - x^*))  d\tau  \right|\right)^2 \\
&\overset{\eqref{eq:olive-09u0hfdih}}{\leq}&  \frac{\|x^k-x^*\|^2}{ \lambda^2} \frac{1}{nm}\sum_{i=1}^n \sum_{j=1}^{m} \|a_{ij}\|^4 \left(   |h_{ij}^k - h_{ij}(x^*)| + \frac{\nu \|a_{ij}\|}{2} \|x^k-x^*\|  \right)^2 \\ 
&\leq&  \frac{\|x^k-x^*\|^2}{ \lambda^2} \frac{1}{nm}\sum_{i=1}^n \sum_{j=1}^{m} \|a_{ij}\|^4 \left(   2|h_{ij}^k - h_{ij}(x^*)|^2 + \frac{\nu^2 \|a_{ij}\|^2}{2} \|x^k-x^*\|^2  \right) \\
&\leq & \frac{\|x^k-x^*\|^2}{ \lambda^2} \frac{1}{nm}\sum_{i=1}^n \sum_{j=1}^{m} R^4 \left(   2|h_{ij}^k - h_{ij}(x^*)|^2 + \frac{\nu^2 R^2}{2} \|x^k-x^*\|^2  \right).
\end{eqnarray*}
In the last step we have used Young's inequality.\footnote{$(a+b)^2 \leq 2a^2 + 2b^2$.}

Since $R\eqdef \max_{i, j} \{ \|a_{ij}\| \}$, we can further bound
\begin{equation}\label{eq:xk+1iter}
\|x^{k+1} - x^*\|^2 \leq \frac{2 R^4}{nm\lambda^2}  \cH^k \|x^k-x^*\|^2 + \frac{ \nu^2 R^6}{2\lambda^2}\|x^k-x^*\|^4. 
\end{equation}

Assume $\|x^k - x^*\|^2 \leq \frac{\lambda^2}{12\nu^2R^6}$ for all $k\geq 0$. Then from (\ref{eq:xk+1iter}), we have 
\begin{eqnarray*}
	\|x^{k+1} - x^*\|^2 &\leq& \frac{\lambda^2}{12\nu^2 R^6} \cdot \frac{2R^4}{nm\lambda^2} \cH^k + \frac{\lambda^2}{12\nu^2 R^6} \cdot \frac{ \nu^2 R^6\|x^k-x^*\|^2}{2\lambda^2} \\ 
	&\leq& \frac{1}{6nm \nu^2 R^2} {\cal H}^k + \frac{1}{24} \|x^k - x^*\|^2, 
\end{eqnarray*}
and by taking expectation, we have 
\begin{equation}\label{eq:expxk+1}
\ExpBr{ \|x^{k+1} - x^*\|^2 \;|\; W^k } \leq  \frac{1}{6nm \nu^2 R^2}  \cH^k + \frac{1}{24} \|x^k - x^*\|^2. 
\end{equation}
Next, Lemma~\ref{lm:calHk} implies that
\begin{equation}\label{eq:oklah090pop}
\ExpBr{ {\cal H}^{k+1} \;|\; W^k} \leq  (1-\eta) {\cal H}^k +  \eta nm \nu^2 R^2 \norm{x^k - x^*}^2.
\end{equation}

Recall that $\Phi_1^{k+1} \eqdef \|x^{k+1} - x^*\|^2 + \frac{1}{3\eta nm  \nu^2 R^2} {\cal H}^{k+1}$. Combining \eqref{eq:expxk+1} and \eqref{eq:oklah090pop}, we get
\begin{eqnarray}
	\ExpBr{ \Phi_1^{k+1} \;|\; W^k} &=& \ExpBr{ \|x^{k+1} - x^*\|^2 \;|\; W^k} + \frac{1}{3\eta nm  \nu^2 R^2}  \ExpBr{ {\cal H}^{k+1} \;|\; W^k}  \notag\\ 
	&\overset{\eqref{eq:expxk+1}}\leq&  \frac{1}{3\eta nm  \nu^2 R^2} \left(  1 - \eta + \frac{\eta}{2}  \right) {\cal H}^k  + \left(  \frac{1}{24} + \frac{1}{3}  \right)  \|x^k - x^*\|^2  \notag \\ 
	&\leq& \left(  1 - \min \left\{  \frac{\eta}{2}, \frac{5}{8}  \right\}  \right)  \Phi_1^k = \theta_1^k \Phi_1^k. \label{eq:ieopllup-878s}
\end{eqnarray}
By applying the tower property, we get $$\ExpBr{\Phi_1^{k+1}} = \ExpBr{\ExpBr{ \Phi_1^{k+1} \;|\; W^k}} \overset{\eqref{eq:ieopllup-878s}}{\leq} \theta_1 \ExpBr{\Phi_1^k}.$$ Unrolling the recursion, we get   $\ExpBr{\Phi_1^k } \leq  
\theta_1^k  \Phi_1^0$, and the first claim is proved.

We further have $\ExpBr{ \|x^k - x^*\|^2 } \leq \theta_1^k \Phi_1^0$ and $\ExpBr{{\cal H}^k} \leq  \theta_1^k 3\eta nm \nu^2 R^2 \Phi_1^0$. Assume $x^k \neq x^*$ for all $k$. Then from (\ref{eq:xk+1iter}), we have 
$$
\frac{\|x^{k+1} - x^*\|^2}{\|x^k - x^*\|^2 } \leq \frac{2R^4}{nm \lambda^2} {\cal H}^k + \frac{\nu^2 R^6}{2\lambda^2}\|x^k - x^*\|^2, 
$$
and by taking expectation, we obtain 
\begin{eqnarray*}
	\ExpBr{  \frac{\|x^{k+1} - x^*\|^2}{\|x^k - x^*\|^2 }  } &\leq& \frac{2R^4}{mn \lambda^2} \ExpBr{{\cal H}^k } + \frac{\nu^2 R^6}{2\lambda^2} \ExpBr{ \|x^k - x^*\|^2 } \\ 
	&\leq&\theta_1^k  \left(  {6\eta} + \frac{1}{2}  \right) \frac{\nu^2 R^6}{\lambda^2} \Phi_1^0. 
\end{eqnarray*}

\subsection{Proof of Lemma \ref{lm:initial-1}}

We prove this by induction. First, we have $\|x^0 - x^*\|^2 \leq \frac{\lambda^2}{12\nu^2R^6}$ by the assumption. We assume $\|x^k - x^*\|^2 \leq \frac{\lambda^2}{12\nu^2R^6}$ holds for all $k \leq K$. For $k\leq K$, since $h_{ij}^k$ is a convex combination of $\{  \newalpha_{ij}(x^0), \newalpha_{ij}(x^1), ..., \newalpha_{ij}(x^k)  \}$ for all $i,j$, we have 
$$
h_{ij}^k = \sum_{t=0}^k \rho_t \newalpha_{ij}(x^t), \quad {\rm with}\quad  \sum_{t=0}^k\rho_t = 1, \quad {\rm and} \quad \rho_t \geq 0, 
$$
which implies that 
\begin{eqnarray*}
	\sum_{i=1}^n\|h_i^k - \newalpha_i(x^*)\|^2 &=& \sum_{i=1}^n \sum_{j=1}^m |h_{ij}^k - \newalpha_{ij}(x^*)|^2 \\
	&=&  \sum_{i=1}^n \sum_{j=1}^m \left|  \sum_{t=0}^k \rho_t (\newalpha_{ij}(x^t) - \newalpha_{ij}(x^*))  \right|^2 \\ 
	&\leq&  \sum_{i=1}^n \sum_{j=1}^m  \sum_{t=0}^k \rho_t | \newalpha_{ij}(x^t) - \newalpha_{ij}(x^*) |^2 \\ 
	&\overset{(\ref{eq:alphaijL})}{\leq}&  \sum_{i=1}^n \sum_{j=1}^m  \sum_{t=0}^k \rho_t \nu^2 \|a_{ij}\|^2 \|x^t - x^*\|^2 \\ 
	&\overset{\text{Assumption}~\ref{as:learning-1}}{\leq}& \nu^2 R^2 \sum_{i=1}^n \sum_{j=1}^m  \sum_{t=0}^k \rho_t \cdot \frac{\lambda^2}{12\nu^2R^6} \\
	&=& \frac{mn \lambda^2}{12R^4}, 
\end{eqnarray*}
for $k\leq K$. 

Combining the above inequality and (\ref{eq:xk+1iter}), we arrive at 
\begin{eqnarray*}
	\|x^{K+1} - x^*\|^2 &\leq& \frac{2\|x^K-x^*\|^2 R^4}{mn\lambda^2} \sum_{i=1}^n \|h_i^K - \newalpha_i(x^*)\|^2 + \frac{ \nu^2 R^6\|x^K-x^*\|^4}{2\lambda^2} \\
	&\leq&  \frac{2\|x^K-x^*\|^2 R^4}{mn\lambda^2} \cdot  \frac{mn \lambda^2}{12R^4} +  \frac{ \nu^2 R^6\|x^K-x^*\|^4}{2\lambda^2} \\
	&\leq& \frac{1}{6}\|x^K - x^*\|^2 +  \frac{ \nu^2 R^6\|x^K-x^*\|^4}{2\lambda^2} \\
	&\leq& \frac{1}{6} \cdot \frac{\lambda^2}{12\nu^2R^6} + \frac{ \nu^2 R^6}{2\lambda^2} \cdot \left(\frac{\lambda^2}{12\nu^2R^6} \right)^2 \\ 
	&\leq& \frac{\lambda^2}{12\nu^2R^6}. 
\end{eqnarray*}

\clearpage
\section{Proofs for {\sf NL2} (Section~\ref{subsec:NL2})}

For ${\cal H}^k \eqdef \sum_{i=1}^n \|h_i^k - \newalpha_i(x^*)\|^2$, even though the update of $h_i^k$ in Algorithm~\ref{alg:NL2} is slightly different from that of Algorithm~\ref{alg:NL1}, we also have the following lemma. The proof is almost the same as that of Lemma~\ref{lm:calHk}, hence we omit it. 

\begin{lemma}\label{lm:calHk-2} For $\eta \leq \frac{1}{\omega+1}$, we have 
	$$
	\mathbb{E}[{\cal H}^{k+1}] \leq  (1-\eta) \ExpBr{{\cal H}^k} + m n \eta \nu^2 R^2 \ExpBr{\|x^k - x^*\|^2}, 
	$$
	for $k\geq 0$. 
\end{lemma}

\subsection{Proof of Theorem~\ref{th:general}}

First, we have 

\begin{eqnarray*}
	\|x^{k+1} - x^*\| &=& \| x^k - x^* - (\mH^k + \lambda \mI)^{-1} (\nabla f(x^k) + \lambda x^k) \| \\
	&=& \|(\mH^k + \lambda \mI)^{-1} \left(  (\mH^k + \lambda \mI)(x^k - x^*) - \nabla f(x^k) - \lambda x^k  \right) \| \\ 
	&\leq& \frac{1}{\mu} \|  (\mH^k + \lambda \mI)(x^k - x^*) - \nabla f(x^k) - \lambda x^k  \| \\ 
	&\overset{(\ref{eq:optx})}{=}& \frac{1}{\mu} \| \mH^k (x^k - x^*) - (\nabla f(x^k) - \nabla f(x^*)) \| \\ 
	&\leq& \frac{1}{n \mu} \sum_{i=1}^n \| \mH_i^k (x^k - x^*) - (\nabla f_i(x^k) - \nabla f_i(x^*)) \| \\
	&=& \frac{1}{n \mu} \sum_{i=1}^n \frac{1}{m} \left\|  \sum_{j=1}^{m} \left( \left( \beta^k (h_{ij}^k + 2\gamma) - 2\gamma \right)a_{ij}a_{ij}^\top (x^k - x^*) -  (\nabla f_{ij}(x^k) - \nabla f_{ij}(x^*)) \right)   \right\| 
\end{eqnarray*}

By using 	\eqref{eq:difnablaf}, we further get
\begin{eqnarray*}
	\|x^{k+1} - x^*\|
	&\overset{(\ref{eq:difnablaf})}{=}& \frac{1}{n \mu} \sum_{i=1}^n \frac{1}{m} \left\|  \sum_{j=1}^{m} \left( \beta^k (h_{ij}^k + 2\gamma)a_{ij}a_{ij}^\top (x^k - x^*) - 2\gamma a_{ij}a_{ij}^\top (x^k - x^*) \right. \right. \\ 
	&& \left. \left. -  \int_{0}^1 \mH_{ij} (x^* + \tau(x^k - x^*)) (x^k - x^*)d\tau  \right)   \right\| \\ 
	&=& \frac{1}{n \mu} \sum_{i=1}^n \frac{1}{m} \left\|  \sum_{j=1}^{m} \left( \beta^k (h_{ij}^k + 2\gamma)a_{ij}a_{ij}^\top (x^k - x^*) - 2\gamma a_{ij}a_{ij}^\top (x^k - x^*) \right. \right. \\ 
	&& \left. \left. -  \int_{0}^1 \newalpha_{ij} (x^* + \tau(x^k - x^*))a_{ij}a_{ij}^\top (x^k - x^*)d\tau  \right)   \right\| \\ 
	&=& \frac{1}{n \mu} \sum_{i=1}^n \frac{1}{m} \left\|  \sum_{j=1}^{m}  (h_{ij}^k + 2\gamma)a_{ij}a_{ij}^\top (x^k - x^*) \left( \beta^k - \int_{0}^1 \frac{\newalpha_{ij}(x^* + \tau(x^k - x^*)) + 2\gamma}{ h_{ij}^k + 2\gamma} d\tau \right) \right\| \\ 
	&\leq& \frac{3\gamma\|x^k - x^*\|}{n \mu} \sum_{i=1}^n \sum_{j=1}^{m} \frac{\|a_{ij}\|^2}{m} \left|   \beta^k - \int_{0}^1 \frac{\newalpha_{ij}(x^* + \tau(x^k - x^*)) + 2\gamma}{ h_{ij}^k + 2\gamma} d\tau   \right|, 
\end{eqnarray*}
where we use $|h_{ij}^k| \leq \gamma$ in the last inequality. Next, we estimate $|\beta^k - 1|$:
\begin{eqnarray*}
	\left|  \frac{\newalpha_{ij}(x^k) + 2\gamma}{h_{ij}^k + 2\gamma}  - 1  \right| &=& \left|  \frac{\newalpha_{ij}(x^k) - h_{ij}^k}{h_{ij}^k + 2\gamma}  \right| \\ 
	&\leq& \frac{1}{\gamma} |\newalpha_{ij}(x^k) - \newalpha_{ij}(x^*) + \newalpha_{ij}(x^*) - h_{ij}^k| \\ 
	&\leq& \frac{1}{\gamma} |\newalpha_{ij}(x^k) - \newalpha_{ij}(x^*)| + \frac{1}{\gamma} |h_{ij}^k - \newalpha_{ij}(x^*)| \\ 
	&\overset{(\ref{eq:alphaijL})}{\leq}& \frac{\nu \|a_{ij}\|}{\gamma} \|x^k - x^*\| + \frac{1}{\gamma} |h_{ij}^k - \newalpha_{ij}(x^*)|. 
\end{eqnarray*}

Let $\{i^k, j^k\} = {\rm argmax}_{i, j} \left\{ \frac{\newalpha_{ij}(x^k) + 2\gamma}{h_{ij}^k + 2\gamma}  \right\}$. Then we have 
\begin{eqnarray*}
	|\beta^k - 1| &=& \left| \frac{\newalpha_{i^kj^k}(x^k) + 2\gamma}{h_{i^kj^k}^k + 2\gamma}   - 1  \right| \\ 
	&\leq& \max_{i, j} \left\{  \frac{\nu \|a_{ij}\|}{\gamma} \|x^k - x^*\| + \frac{1}{\gamma} |h_{ij}^k - \newalpha_{ij}(x^*)|  \right\} \\ 
	&\leq& \frac{\nu R}{\gamma} \|x^k - x^*\| + \frac{1}{\gamma} \max_{i, j} \{|h_{ij}^k - \newalpha_{ij}(x^*)| \}.  
\end{eqnarray*}

For $\left| \int_{0}^1 \frac{\newalpha_{ij}(x^* + \tau(x^k - x^*)) + 2\gamma}{ h_{ij}^k + 2\gamma} d\tau - 1   \right|$, we have 

\begin{eqnarray*}
	&& \left| \int_{0}^1 \frac{\newalpha_{ij}(x^* + \tau(x^k - x^*)) + 2\gamma}{ h_{ij}^k + 2\gamma} d\tau - 1   \right| \\ 
	&=& \left|  \int_{0}^1  \frac{\newalpha_{ij}(x^* + \tau(x^k - x^*)) - h_{ij}^k}{ h_{ij}^k + 2\gamma} d\tau  \right| \\ 
	&\leq& \int_{0}^1 \left|   \frac{\newalpha_{ij}(x^* + \tau(x^k - x^*)) - h_{ij}^k}{ h_{ij}^k + 2\gamma}  \right| d\tau \\ 
	&\leq& \frac{1}{\gamma} \int_{0}^1 \left(  |\newalpha_{ij}(x^* + \tau(x^k - x^*)) - \newalpha_{ij}(x^*)| + |h_{ij}^k - \newalpha_{ij}(x^*)|  \right) d\tau \\ 
	&\overset{(\ref{eq:alphaijL})}{\leq}& \frac{1}{\gamma} |h_{ij}^k - \newalpha_{ij}(x^*)| + \frac{1}{\gamma} \int_{0}^1 \tau \nu \|a_{ij}\| \cdot \|x^k - x^*\| d\tau \\ 
	&=& \frac{1}{\gamma} |h_{ij}^k - \newalpha_{ij}(x^*)|  + \frac{\nu \|a_{ij}\|}{2\gamma} \|x^k - x^*\|. 
\end{eqnarray*}

Combining the two above inequalities, we can obtain 

\begin{eqnarray}
&& \left|   \beta^k - \int_{0}^1 \frac{\newalpha_{ij}(x^* + \tau(x^k - x^*)) + 2\gamma}{ h_{ij}^k + 2\gamma} d\tau   \right| \nonumber \\ 
&\leq& \frac{\nu R}{\gamma} \|x^k - x^*\| + \frac{1}{\gamma} \max_{i, j} \{|h_{ij}^k - \newalpha_{ij}(x^*)| \} + \frac{1}{\gamma} |h_{ij}^k - \newalpha_{ij}(x^*)|  + \frac{\nu \|a_{ij}\|}{2\gamma} \|x^k - x^*\| \nonumber \\ 
&\leq&  \frac{2\nu R}{\gamma} \|x^k - x^*\| + \frac{2}{\gamma} \max_{i, j} \{|h_{ij}^k - \newalpha_{ij}(x^*)| \}. \label{eq:betak-2}
\end{eqnarray}

Thus, we have 
\begin{eqnarray}
\|x^{k+1} - x^*\|^2 &\leq& \frac{9\gamma^2\|x^k - x^*\|^2}{n \mu^2} \sum_{i=1}^n \sum_{j=1}^{m} \frac{\|a_{ij}\|^4}{m} \left|   \beta^k - \int_{0}^1 \frac{\newalpha_{ij}(x^* + \tau(x^k - x^*)) + 2\gamma}{ h_{ij}^k + 2\gamma} d\tau   \right|^2 \nonumber \\ 
&\leq& \frac{9\gamma^2\|x^k - x^*\|^2}{n \mu^2} \sum_{i=1}^n \sum_{j=1}^{m} \frac{\|a_{ij}\|^4}{m} \left(  \frac{8\nu^2 R^2}{\gamma^2} \|x^k - x^*\|^2 + \frac{8}{\gamma^2} \max_{i, j} \{|h_{ij}^k - \newalpha_{ij}(x^*)|^2 \}    \right) \nonumber \\ 
&\leq& \frac{72\|x^k - x^*\|^2 R^4}{\mu^2} \max_{i, j} \{|h_{ij}^k - \newalpha_{ij}(x^*)|^2 \}  +   \frac{72\nu^2 R^6\|x^k - x^*\|^4 }{\mu^2} \label{eq:forinitiallm} \\ 
&\leq&  \frac{72\|x^k - x^*\|^2 R^4}{\mu^2} \sum_{i=1}^n \|h_i^k - \newalpha_i(x^*)\|^2 +   \frac{72\nu^2 R^6\|x^k - x^*\|^4 }{\mu^2}, \nonumber 
\end{eqnarray}
where we use the convexity of $x\mapsto \|x\|^2$ in the first inequality. From the definition of ${\cal H}^k$, we can write 
\begin{equation}\label{eq:xk+1-2}
\|x^{k+1} - x^*\|^2 \leq  \frac{72\|x^k - x^*\|^2 R^4}{\mu^2} {\cal H}^k +   \frac{72\nu^2 R^6\|x^k - x^*\|^4 }{\mu^2}. 
\end{equation}

Assume $\|x^k - x^*\|^2 \leq \frac{\mu^2}{432m n \nu^2R^6}$ for all $k\geq 0$. Then from (\ref{eq:xk+1-2}), we have 
\begin{eqnarray*}
	\|x^{k+1} - x^*\|^2 &\leq& \frac{\mu^2}{432 mn\nu^2 R^6} \cdot \frac{72R^4}{\mu^2}{\cal H}^k + \frac{\mu^2}{432mn\nu^2 R^6} \cdot \frac{ 72\nu^2 R^6\|x^k-x^*\|^2}{\mu^2} \\ 
	&\leq& \frac{1}{6mn \nu^2 R^2} {\cal H}^k + \frac{1}{6mn} \|x^k - x^*\|^2, 
\end{eqnarray*}
and by taking expectation, we have 
\begin{equation}\label{eq:expxk+1-2}
\ExpBr{\|x^{k+1} - x^*\|^2} \leq  \frac{1}{6mn \nu^2 R^2} \ExpBr{{\cal H}^k} + \frac{1}{6mn} \ExpBr{ \|x^k - x^*\|^2}. 
\end{equation}

Let $\Phi_2^k \eqdef \|x^{k} - x^*\|^2 + \frac{1}{3mn\eta  \nu^2 R^2} {\cal H}^{k}$. Combining (\ref{eq:expxk+1-2}) and the evolution of ${\cal H}^k$ in Lemma~\ref{lm:calHk-2}, we arrive at 
\begin{eqnarray*}
	\mathbb{E}[\Phi_2^{k+1}] &=& \ExpBr{\|x^{k+1} - x^*\|^2} + \frac{1}{3mn\eta  \nu^2 R^2}  \ExpBr{{\cal H}^{k+1}} \\ 
	&\leq&  \frac{1}{3mn \eta \nu^2 R^2} \left(  1 - \eta + \frac{\eta}{2}  \right) \mathbb{E}[{\cal H}^k] + \left(  \frac{1}{6mn} + \frac{1}{3}  \right)\ExpBr{\|x^k - x^*\|^2 }\\ 
	&\leq& \left(  1 - \min \left\{  \frac{\eta}{2}, \frac{1}{2}  \right\}  \right) \ExpBr{\Phi_2^k}, 
\end{eqnarray*}
which implies that $\ExpBr{\Phi_2^k}\leq  \left(  1 - \min \left\{  \frac{\eta}{2}, \frac{1}{2}  \right\}  \right)^k \Phi_2^0$. Then we further have $\ExpBr{\|x^k - x^*\|^2} \leq  \left(  1 - \min \left\{  \frac{\eta}{2}, \frac{1}{2}  \right\}  \right)^k \Phi_2^0$ and $\mathbb{E}[{\cal H}^k] \leq  \left(  1 - \min \left\{  \frac{\eta}{2}, \frac{1}{2}  \right\}  \right)^k 3mn\eta \nu^2 R^2 \Phi_2^0$. Assume $x^k \neq x^*$ for all $k$. Then from (\ref{eq:xk+1-2}), we have 
$$
\frac{\|x^{k+1} - x^*\|^2}{\|x^k - x^*\|^2 } \leq \frac{72R^4}{ \mu^2} {\cal H}^k + \frac{72 \nu^2 R^6}{\mu^2}\|x^k - x^*\|^2, 
$$
and by taking expectation, we can get 
\begin{eqnarray*}
	\mathbb{E} \left[  \frac{\|x^{k+1} - x^*\|^2}{\|x^k - x^*\|^2 }  \right] &\leq& \frac{72R^4}{ \mu^2}\ExpBr{{\cal H}^k} + \frac{72 \nu^2 R^6}{\mu^2} \ExpBr{\|x^k - x^*\|^2} \\ 
	&\leq& \left(  1 - \min \left\{  \frac{\eta}{2}, \frac{1}{2}  \right\}  \right)^k  \left(  {3mn\eta} + 1  \right) \frac{72\nu^2 R^6}{\mu^2} \Phi_2^0. 
\end{eqnarray*}

\subsection{Proof of Lemma \ref{lm:initial-2}}

We prove this by induction. First, we have $\|x^0 - x^*\|^2 \leq \frac{\mu^2}{432mn \nu^2R^6}$ by the assumption. We assume $\|x^k - x^*\|^2 \leq \frac{\mu^2}{432mn \nu^2R^6}$ holds for all $k \leq K$. For $k\leq K$, since $h_{ij}^k$ is a convex combination of $\{  \newalpha_{ij}(x^0), \newalpha_{ij}(x^1), ..., \newalpha_{ij}(x^k)  \}$ for all $i,j$, we have 
$$
h_{ij}^k = \sum_{t=0}^k \rho_t \newalpha_{ij}(x^t), \quad {\rm with}\quad  \sum_{t=0}^k\rho_t = 1, \quad {\rm and} \quad \rho_t \geq 0, 
$$
which implies that 
\begin{eqnarray*}
	|h_{ij}^k - \newalpha_{ij}(x^*)|^2 &=&   \left|  \sum_{t=0}^k \rho_t (\newalpha_{ij}(x^t) - \newalpha_{ij}(x^*))  \right|^2 \\ 
	&\leq&   \sum_{t=0}^k \rho_t | \newalpha_{ij}(x^t) - \newalpha_{ij}(x^*) |^2 \\ 
	&\overset{(\ref{eq:alphaijL})}{\leq}&   \sum_{t=0}^k \rho_t \nu^2 \|a_{ij}\|^2 \|x^t - x^*\|^2 \\ 
	&\overset{Assumption \ref{as:learning-2}}{\leq}& \nu^2 R^2  \sum_{t=0}^k \rho_t \cdot \frac{\mu^2}{432mn \nu^2R^6} \\
	&=& \frac{ \mu^2}{432mnR^4}, 
\end{eqnarray*}
for $k\leq K$. 

Combining the above inequality and (\ref{eq:forinitiallm}), we arrive at 
\begin{eqnarray*}
	\|x^{K+1} - x^*\|^2 &\leq& \frac{72\|x^K-x^*\|^2 R^4}{\mu^2} \max_{i, j} \{  | h_{ij}^K - \newalpha_{ij}(x^*) |^2  \}  + \frac{ 72\nu^2 R^6\|x^K-x^*\|^4}{\mu^2} \\
	&\leq& \frac{72\|x^K-x^*\|^2 R^4}{\mu^2} \cdot  \frac{ \mu^2}{432mnR^4} +  \frac{ 72\nu^2 R^6\|x^K-x^*\|^4}{\mu^2} \\
	&\leq& \frac{1}{6mn}\|x^K - x^*\|^2 + \frac{ 72\nu^2 R^6\|x^K-x^*\|^4}{\mu^2} \\ \\
	&\leq& \frac{1}{6mn} \cdot \frac{\mu^2}{432n \nu^2R^6} + \frac{72 \nu^2 R^6}{\mu^2} \cdot \left(\frac{\mu^2}{432mn \nu^2R^6} \right)^2 \\ 
	&\leq& \frac{\mu^2}{432mn \nu^2R^6}. 
\end{eqnarray*}

\clearpage

\section{Proofs for {\sf CNL} (Section~\ref{sec:CUBIC-NEWTON-LEARN})}

%%%%%%%%%%%%%%%%%%%%%%

\subsection{Solving the Subproblem}\label{sec:subproblem}

In {\sf CUBIC-NEWTON-LEARN}, we need to minimize $T(x^k, s)$ at each step. The optimality condition for this subproblem is 
\begin{equation}\label{eq:subproblem}
\nabla P(x^k) + (\mH^k + \lambda \mI)s + \frac{M\|s\|}{2}s = 0. 
\end{equation}

Let $\mU^T \Lambda \mU$ be the eigenvalue decomposition of $\mH^k + \lambda \mI$. Then we can transform the above equality to 
$$
\Lambda \mU s + \frac{M\|s\|}{2} \mU s = - \mU \nabla P(x^k). 
$$
Noticing that $\| \mU s\| = \|s\|$, define $y = \mU s$, then we can get 
$$
y = - \left(  \Lambda + \frac{M}{2}\|y\|\mI  \right)^{-1} \mU \nabla P(x^k), 
$$
and by taking the norm, we arrive at 
$$
\|y\|^2 = \sum_{i=1}^d \frac{\left(  \mU\nabla P(x^k)  \right)_i^2 }{(\Lambda_i + \frac{M}{2}\|y\|)^2}, 
$$
which is actually equivalent to a one-dimensional nonlinear equation. By solving this one-dimensional nonlinear equation, we can get $\|s\| = \|\mU s\| = \|y\|$, and then from (\ref{eq:subproblem}), we can obtain the solution 
$$
s = - \left(  \mH^k + \lambda \mI + \frac{M\|s\|\mI}{2}  \right)^{-1} \nabla P(x^k) = - \left(  \mH^k + \lambda \mI + \frac{M\|y\|\mI}{2}  \right)^{-1} \nabla P(x^k). 
$$

\subsection{Proof of Lemma \ref{lm:global}}

Since $x^{k+1} = x^k + s^k$, from Lemma 1 in \citep{PN2006-cubic}, we have 
\begin{eqnarray}
P(x^{k+1}) &\leq& P(x^k) + \langle \nabla f(x^k) + \lambda x^k, s^k \rangle + \frac{1}{2} \langle (\nabla^2 f(x^k) + \lambda \mI)s^k, s^k \rangle + \frac{M}{6}\|s^k\|^3 \nonumber \\ 
&\leq&  P(x^k) + \langle \nabla f(x^k) + \lambda x^k, s^k \rangle + \frac{1}{2} \langle (\mH^k + \lambda \mI)s^k, s^k \rangle + \frac{M}{6}\|s^k\|^3 \nonumber \\ 
&=& P(x^k) + T(x^k, s^k). \label{eq:decreasingP}
\end{eqnarray}
Since $T(x^k, s^k) = \min_{s}\{  T(x^k, s)  \} \leq T(x^k, 0)=0$, we have $P(x^{k+1}) \leq P(x^k)$.  From (\ref{eq:decreasingP}), we have 
\begin{eqnarray*}
	P(x^{k+1}) &\leq& P(x^k) + T(x^k, s^k)\\
	&=& P(x^k) + \min_{s}\{  T(x^k, s)  \}  \\ 
	&\leq& P(x^k) + T(x^k, y - x^k) \\ 
	&=& P(x^k) + \langle \nabla P(x^k), y-x^k \rangle + \frac{1}{2} \langle (\mH^k + \lambda \mI)(y-x^k), y-x^k \rangle + \frac{M}{6}\|y-x^k\|^2 \\ 
	&\leq& P(x^k) + \langle \nabla P(x^k), y-x^k \rangle + \frac{1}{2} \langle \nabla^2 P(x^k)(y-x^k), y-x^k \rangle + \frac{M}{6}\|y-x^k\|^2\\ 
	&& + \frac{1}{2} \|\mH^k + \lambda \mI - \nabla^2 P(x^k)\| \cdot \|y-x^k\|^2. 
\end{eqnarray*}

First, by Lemma 1 in \citep{PN2006-cubic}, we can obtain 
$$
\langle \nabla P(x^k), y-x^k \rangle + \frac{1}{2} \langle \nabla^2 P(x^k)(y-x^k), y-x^k \rangle \leq P(y) - P(x^k) + \frac{M}{6}\|y-x^k\|^3. 
$$

For $ \|\mH^k + \lambda \mI - \nabla^2 P(x^k)\| $, we have 
\begin{eqnarray*}
	\|\mH^k + \lambda \mI - \nabla^2 P(x^k)\| &=& \|\mH^k - \nabla^2 f(x^k)\| \\
	&=& \left\| \frac{1}{n} \sum_{i=1}^n \frac{1}{m} \sum_{j=1}^m \left(  \beta^k (h_{ij}^k + 2\gamma) - 2\gamma - \newalpha_{ij}(x^k)  \right)a_{ij}a_{ij}^\top \right\| \\ 
	&\leq&  \frac{1}{n} \sum_{i=1}^n \frac{1}{m} \sum_{j=1}^m R^2 \left| \beta^k (h_{ij}^k + 2\gamma) - 2\gamma - \newalpha_{ij}(x^k) \right| \\ 
	&\leq&  \frac{1}{n} \sum_{i=1}^n \frac{1}{m} \sum_{j=1}^m R^2 \left| h_{ij}^k + 2\gamma   \right| \cdot \left| \beta^k  - \frac{\newalpha_{ij}(x^k) + 2\gamma}{h_{ij}^k + 2\gamma} \right|. 
\end{eqnarray*}
Since $|h_{ij}^k| \leq \gamma$ and $|\newalpha_{ij}(x^k)| \leq \gamma$, we have $|h_{ij}^k + 2\gamma| \leq 3\gamma$ and $\left| \frac{\newalpha_{ij}(x^k) + 2\gamma}{h_{ij}^k + 2\gamma} \right| \leq 3$ for any $i,j$.  Thus we can get 
$$
\|\mH^k + \lambda \mI - \nabla^2 P(x^k)\| \leq 18\gamma R^2. 
$$

Then we arrive at 
$$
P(x^{k+1}) \leq P(y) + 9\gamma R^2 \|y-x^k\|^2 + \frac{M}{3}\|y-x^k\|^3. 
$$

\subsection{Proof of Theorem \ref{th:concubic}}

Recall that $R_x =\sup_{x\in \R^d} \{  \|x-x^*\| : P(x) \leq P(x^0)  \}$. Since $P(x^k) \leq P(x^0)$ for all $k\geq 0$, we have $\|x^k - x^*\| \leq R_x$ for $k\geq 0$.  Let $c_k \eqdef k^2$, $C_k = C_0 + \sum_{i=1}^kc_i$ with $C_0=\frac{4}{3}$ for $k\geq 1$. Then we can obtain  
\begin{equation}\label{eq:Ck}
C_k = C_0 + \sum_{i=1}^k i^2 \geq C_0 + \int_{0}^k x^2 dx = C_0 + \frac{k^3}{3}. 
\end{equation}

Let $\sigma_k \eqdef \frac{c_{k+1}}{C_{k+1}}$ for $k\geq 0$. Then $1 - \sigma_k = \frac{C_k}{C_{k+1}}$. We also have $\sigma_k \leq 1$ since 
$$
\max_{\xi\geq 0} \frac{\xi^2}{C_0 + \frac{\xi^3}{3}} =1. 
$$

Choose $y = \sigma_k x^* + (1-\sigma_k)x^k$. Then by the convexity of $P$, we have 
\begin{eqnarray*}
	P(x^{k+1}) &\leq& \sigma_k P(x^*) + (1-\sigma_k) P(x^k) + 9\gamma R^2 \sigma_k^2 \|x^k - x^*\|^2 + \frac{M \sigma_k^3}{3}\|x^k-x^*\|^3 \\ 
	&\leq& \frac{c_{k+1}}{C_{k+1}}P(x^*) + \frac{C_k}{C_{k+1}}P(x^k) + 9\gamma R^2R_x^2 \left(  \frac{c_{k+1}}{C_{k+1}}  \right)^2 + \frac{MR_x^3}{3} \left(  \frac{c_{k+1}}{C_{k+1}}  \right)^3, 
\end{eqnarray*}
which implies that 
$$
C_{k+1} [P(x^{k+1}) - P(x^*)] \leq C_k [P(x^k) - P(x^*)] + 9\gamma R^2R_x^2 \frac{c_{k+1}^2}{C_{k+1}} + \frac{MR_x^3}{3} \frac{c_{k+1}^3}{C_{k+1}^2}. 
$$

Hence, 
$$
C_k[P(x^k) - P(x^*)] \leq C_0[P(x^0) - P(x^*)] + 9\gamma R^2R_x^2 \sum_{i=1}^k \frac{c_i^2}{C_i} + \frac{MR_x^3}{3} \sum_{i=1}^k \frac{c_i^3}{C_i^2}. 
$$

Notice that 
$$
\sum_{i=1}^k \frac{c_i^2}{C_i} \overset{(\ref{eq:Ck})}{\leq} \sum_{i=1}^k \frac{i^4}{C_0 + \frac{1}{3}i^3}\leq 3 \sum_{i=1}^k i \leq 3k^2, 
$$
and 
$$
\sum_{i=1}^k \frac{c_i^3}{C_i^2}  \overset{(\ref{eq:Ck})}{\leq}  \sum_{i=1}^k \frac{i^6}{(C_0 + \frac{1}{3}i^3)^2} \leq 9k. 
$$

Therefore, we arrive at 
\begin{eqnarray*}
	P(x^k) - P(x^*) &\leq& \frac{1}{C_k} \left(  C_0[P(x^0) - P(x^*)] + 27\gamma R^2R_x^2 k^2 + 3MR_x^3 k \right) \\ 
	&\overset{(\ref{eq:Ck})}{\leq}& \frac{81\gamma R^2 R_x^2}{k} + \frac{9MR_x^3}{k^2} + \frac{4[P(x^0) - P(x^*)]}{k^3}. 
\end{eqnarray*}

Then in order to guarantee $P(x^k) - P(x^*) \leq \epsilon$, we only need to let 
$$
k = O\left( \frac{81\gamma R^2 R_x^2}{\epsilon}  + \sqrt{\frac{MR_x^3}{\epsilon}} + \left(  \frac{P(x^0) - P(x^*)}{\epsilon}  \right)^{1/3}   \right). 
$$

\subsection{Proof of Theorem \ref{th:stronglyconcubic}}

Since $P$ is $\mu$-strongly convex, we have 
$
\frac{\mu}{2} \|x-x^*\|^2 \leq P(x) - P(x^*), 
$
and thus 
$$
R_x^2 \leq \frac{2}{\mu} (P(x^0) - P(x^*)). 
$$

Combining the above inequality and Theorem~\ref{th:concubic}, we can obtain 
$$
P(x^k) - P(x^*) \leq \left(  \frac{162\gamma R^2}{\mu k} +  \frac{18MR_x}{\mu k^2} + \frac{4}{ k^3}  \right) \cdot \left(  P(x^0) - P(x^*)  \right)
$$

If we let 
$$
\frac{162\gamma R^2}{\mu k} \leq \frac{1}{6}, \quad  \frac{18MR_x}{\mu k^2} \leq \frac{1}{6}, \quad {\rm and} \frac{4}{ k^3}  \leq \frac{1}{6}, 
$$
which is equivalent to 
$$
k \geq O\left(  \frac{\gamma R^2}{\mu} + \sqrt{\frac{MR_x}{\mu}} + 1 \right), 
$$
then we can obtain $P(x^k) - P(x^*) \leq \frac{1}{2} (P(x^0) - P(x^*))$. This means that we can reduce the residual $P(x^k)-P(x^*)$ by half after $O\left(  \frac{\gamma R^2}{\mu} + \sqrt{\frac{MR_x}{\mu}} + 1 \right)$ iterations. Thus, we have $P(x^k) - P(x^*) \leq \epsilon$ as long as 
$$
k = O\left( \left(  \frac{\gamma R^2}{\mu} + \sqrt{\frac{MR_x}{\mu}} + 1 \right) \log \left(  \frac{P(x^0) - P(x^*)}{\epsilon}  \right)   \right). 
$$

\subsection{Proof of Theorem~\ref{th:cubicsup}}

Since $s^k = {\rm argmin}_{s\in\R^d}T(x^k, s)$, from the optimality condition, we have 

$$
\nabla f(x^k) + \lambda x^k + (\mH^k + \lambda \mI)s^k  + \frac{M}{2}\|s^k\| s^k = 0, 
$$
which indicates that 
\begin{equation}\label{eq:xk+1im}
x^{k+1} = x^k - \left(  \mH^k + \lambda \mI + \frac{M}{2}\|x^{k+1} - x^k\| \right)^{-1} \left(  \nabla f(x^k) + \lambda x^k  \right). 
\end{equation}

Then we can obtain 
\begin{eqnarray*}
	\|x^{k+1} - x^*\| &=& \left \| x^k - x^* - \left(\mH^k + \lambda \mI +  \frac{M}{2}\|x^{k+1} - x^k\| \right)^{-1} (\nabla f(x^k) + \lambda x^k) \right\| \\
	&=& \left\| \left(\mH^k + \lambda \mI + \frac{M}{2}\|x^{k+1} - x^k\| \right)^{-1} \left(  \left( \mH^k + \lambda \mI + \frac{M}{2}\|x^{k+1} - x^k\| \right) (x^k - x^*) - \nabla f(x^k) - \lambda x^k  \right) \right\| \\ 
	&\leq& \frac{1}{\mu}  \left\|  \left(\mH^k + \lambda \mI +  \frac{M}{2}\|x^{k+1} - x^k\| \right)(x^k - x^*) - \nabla f(x^k) - \lambda x^k  \right \| \\ 
	&\overset{(\ref{eq:optx})}{=}& \frac{1}{\mu} \left\| \mH^k (x^k - x^*) - (\nabla f(x^k) - \nabla f(x^*)) +  \frac{M}{2}\|x^{k+1} - x^k\|(x^k - x^*) \right\| \\ 
	&\leq& \frac{1}{\mu} \left\| \mH^k (x^k - x^*) - (\nabla f(x^k) - \nabla f(x^*)) \right\|  + \frac{M\|x^{k+1} - x^k\|}{2\mu} \| x^k - x^* \|. 
\end{eqnarray*}

Then same as the analysis in the proof of Theorem~\ref{th:general}, we can obtain 
\begin{eqnarray*}
	\|x^{k+1} - x^*\| &\leq& \frac{3\gamma\|x^k - x^*\|}{n \mu} \sum_{i=1}^n \sum_{j=1}^{m} \frac{\|a_{ij}\|^2}{m} \left|   \beta^k - \int_{0}^1 \frac{\newalpha_{ij}(x^* + \tau(x^k - x^*)) + 2\gamma}{ h_{ij}^k + 2\gamma} d\tau   \right| \\ 
	&& +  \frac{M\|x^{k+1} - x^k\|}{2\mu} \| x^k - x^* \|.
\end{eqnarray*}

Combing the above inequality and (\ref{eq:betak-2}) yields 
\begin{eqnarray*}
	\|x^{k+1} - x^*\|^2 &\leq& \left(  1 + \frac{1}{3}  \right) \frac{9\gamma^2 \|x^k - x^*\|^2}{n \mu^2} \sum_{i=1}^n \sum_{j=1}^{m} \frac{\|a_{ij}\|^4}{m} \left|   \beta^k - \int_{0}^1 \frac{\newalpha_{ij}(x^* + \tau(x^k - x^*)) + 2\gamma}{ h_{ij}^k + 2\gamma} d\tau   \right|^2 \\
	&& + (1 + 3) \cdot \frac{M^2\|x^{k+1} - x^k\|^2}{4\mu^2} \| x^k - x^* \|^2 \\
	&\overset{(\ref{eq:betak-2})}{\leq}& \frac{12 \gamma^2\|x^k - x^*\|^2}{n \mu^2} \sum_{i=1}^n \sum_{j=1}^{m} \frac{\|a_{ij}\|^4}{m} \left(  \frac{8\nu^2 R^2}{\gamma^2} \|x^k - x^*\|^2 + \frac{8}{\gamma^2} \max_{i} \{\|h_i^k - \newalpha_i(x^*)\|^2 \}    \right) \\ 
	&& +  \frac{M^2\|x^{k+1} - x^k\|^2}{\mu^2} \| x^k - x^* \|^2 \\ 
	&\leq&  \frac{96\|x^k - x^*\|^2 R^4}{\mu^2} \sum_{i=1}^n \|h_i^k - \newalpha_i(x^*)\|^2 +   \frac{96\nu^2 R^6\|x^k - x^*\|^4 }{\mu^2} +  \frac{M^2\|x^{k+1} - x^k\|^2}{\mu^2} \| x^k - x^* \|^2 \\ 
	&=&  \frac{96\|x^k - x^*\|^2 R^4}{\mu^2}{\cal H}^k  +   \frac{96\nu^2 R^6\|x^k - x^*\|^4 }{\mu^2} +  \frac{M^2\|x^{k+1} - x^k\|^2}{\mu^2} \| x^k - x^* \|^2. 
\end{eqnarray*}

Then from $\|x^{k+1} - x^k\|^2 \leq 2\|x^{k+1} - x^*\|^2 + 2\|x^k - x^*\|^2$, we can get 
\begin{equation}\label{eq:xk+1cubic}
\|x^{k+1} - x^*\|^2 \leq \frac{96\|x^k - x^*\|^2 R^4}{\mu^2}{\cal H}^k + \frac{(96\nu^2 R^6 + 2M^2) \|x^k - x^*\|^4 }{\mu^2} + \frac{2M^2\|x^{k+1} - x^*\|^2}{\mu^2} \| x^k - x^* \|^2. 
\end{equation}

Assume $\|x^k - x^*\|^2 \leq \frac{\mu^2}{432m n \nu^2R^6}$ for all $k\geq 0$. Then from (\ref{eq:xk+1cubic}) and $M=\nu R^3$, we have 
\begin{eqnarray*}
	\|x^{k+1} - x^*\|^2 &\leq& \frac{\mu^2}{432mn\nu^2 R^6} \cdot \frac{96R^4}{\mu^2}{\cal H}^k + \frac{\mu^2}{432mn\nu^2 R^6} \cdot \frac{ 100\nu^2 R^6\|x^k-x^*\|^2}{\mu^2} \\ 
	&\leq& \frac{2}{9mn \nu^2 R^2} {\cal H}^k + \frac{1}{4mn} \|x^k - x^*\|^2, 
\end{eqnarray*}
and by taking expectation, we have 
\begin{equation}\label{eq:expxk+1-cubic}
\mathbb{E}\|x^{k+1} - x^*\|^2 \leq  \frac{2}{9mn \nu^2 R^2} \mathbb{E}[{\cal H}^k] + \frac{1}{4mn} \|x^k - x^*\|^2. 
\end{equation}

Let $\Phi_3^k \eqdef \|x^{k} - x^*\|^2 + \frac{4}{9mn\eta  \nu^2 R^2} {\cal H}^{k}$. Combing (\ref{eq:expxk+1-cubic}) and the evolution of ${\cal H}^k$ in Lemma \ref{lm:calHk-2}, we arrive at 
\begin{eqnarray*}
	\mathbb{E}[\Phi_3^{k+1}] &=& \mathbb{E}\|x^{k+1} - x^*\|^2 + \frac{4}{9mn\eta  \nu^2 R^2}  \mathbb{E}[{\cal H}^{k+1}] \\ 
	&\leq&  \frac{4}{9mn \eta \nu^2 R^2} \left(  1 - \eta + \frac{\eta}{2}  \right) \mathbb{E}[{\cal H}^k] + \left(  \frac{1}{4mn} + \frac{4}{9}  \right) \mathbb{E}\|x^k - x^*\|^2 \\ 
	&\leq& \left(  1 - \min \left\{  \frac{\eta}{2}, \frac{1}{4}  \right\}  \right) \mathbb{E}[\Phi_1^k], 
\end{eqnarray*}
which implies that $\mathbb{E}[\Phi_3^k] \leq  \left(  1 - \min \left\{  \frac{\eta}{2}, \frac{1}{4}  \right\}  \right)^k \Phi_3^0$. Then we further have $\mathbb{E}\|x^k - x^*\|^2 \leq  \left(  1 - \min \left\{  \frac{\eta}{2}, \frac{1}{4}  \right\}  \right)^k \Phi_3^0$ and $\mathbb{E}[{\cal H}^k] \leq  \left(  1 - \min \left\{  \frac{\eta}{2}, \frac{1}{4}  \right\}  \right)^k 3mn\eta \nu^2 R^2 \Phi_3^0$. Assume $x^k \neq x^*$ for all $k$. Then from (\ref{eq:xk+1cubic}), we have 
$$
\frac{\|x^{k+1} - x^*\|^2}{\|x^k - x^*\|^2 } \leq \frac{96R^4}{ \mu^2} {\cal H}^k + \frac{98 \nu^2 R^6}{\mu^2}\|x^k - x^*\|^2 + \frac{2\nu^2 R^6}{\mu^2} \|x^{k+1} - x^*\|^2, 
$$
and by taking expectation, we can get 
\begin{eqnarray*}
	\mathbb{E} \left[  \frac{\|x^{k+1} - x^*\|^2}{\|x^k - x^*\|^2 }  \right] &\leq& \frac{96R^4}{ \mu^2} \mathbb{E} [{\cal H}^k] + \frac{98 \nu^2 R^6}{\mu^2} \mathbb{E}\|x^k - x^*\|^2 + \frac{2\nu^2 R^6}{\mu^2}\mathbb{E} \|x^{k+1} - x^*\|^2 \\ 
	&\leq& \left(  1 - \min \left\{  \frac{\eta}{2}, \frac{1}{4}  \right\}  \right)^k  \left(  {3mn\eta} + 1  \right) \frac{100\nu^2 R^6}{\mu^2} \Phi_3^0. 
\end{eqnarray*}

%%%%%%%%%%%%%%%%%%%%%%%%%%%%

\clearpage
\section{Extra Method: {\sf MAX-NEWTON}} \label{sec:maxNewton}

%%%%%%%%%%%%%%%%%%%%%%%%%%%%

In this section we propose and analyze one more method, {\sf MAX-NEWTON (MN)}, which should be seen as a variant of  {\sf NEWTON-STAR (NS)}. Like {\sf NS}, {\sf MS} is not practical and is of theoretical interest only. This was the first method we developed, and all subsequent development that eventually lead to the results in this paper started here.
{\sf MS} differs from {\sf NS} in how we approximate Hessian of $P(x)$. As in {\sf NS}, we also assume that we know all $\newalpha_{ij}(x^*)$ at the optimum. However,  we estimate the Hessian at $x^k$ by $\mH^k = \frac{1}{n}\sum\limits_{i=1}^n\mH_i^k,$ where $\mH_i^k$ is defined differently:
\begin{eqnarray*}
	\mH_i^k = \frac{1}{m}\max_{j \in [m]}\left\{\frac{\newalpha_{ij}(x^k)}{\newalpha_{ij}(x^*)}\right\}\sum\limits_{j=1}^m\newalpha_{ij}(x^*)a_{ij}a_{ij}^\top = \frac{\beta^k_i}{m}\sum\limits_{j=1}^m\newalpha_{ij}(x^*)a_{ij}a_{ij}^\top.
\end{eqnarray*}
Above, we define $\beta_i^k \eqdef \max\limits_{j \in [m]}\frac{\newalpha_{ij}(x^k)}{\newalpha_{ij}(x^*)}$. Subsequently, we perform a Newton-like  step:
\begin{eqnarray*}
	x^{k+1} = x^k - \left(\mH^k + \lambda \mI \right)^{-1}\left(\frac{1}{n}\sum\limits_{i=1}^n \nabla f_i(x^k) +\lambda x^k\right).
\end{eqnarray*}

The method is summarized as Algorithm~\ref{alg:maxnewton}.

\begin{algorithm}[h]
	\caption{{\sf MN: MAX-NEWTON}}
	\label{alg:maxnewton}
	\begin{algorithmic}
		\STATE {\bfseries Initialization:}
		$x^0 \in \R^d$
		\FOR{ $k = 0, 1, 2, \dots$}
		\STATE Broadcast $x^k$ to all workers 
		\FOR{ $i = 1, \dots, n$} 
		\STATE Compute $\nabla f_i(x^k)$ 
		\STATE $\beta_i^k = \max\limits_{j \in [m]}\frac{\newalpha_{ij}(x^k)}{\newalpha_{ij}(x^*)}$ 
		\STATE Send $\nabla f_i(x^k)$ and $\beta_i^k$ to the server 
		\ENDFOR
		\STATE $\mH_i^k = \frac{\beta_i^k}{m}\sum\limits_{j=1}^m\newalpha_{ij}(x^*)a_{ij}a_{ij}^\top$
		\STATE $\mH^k = \frac{1}{n}\sum\limits_{i=1}^n \mH_i^k$
		\STATE $x^{k+1} = x^k - \left(\mH^k + \lambda \mI \right)^{-1}\left(\frac{1}{n}\sum\limits_{i=1}^n \nabla f_i(x^k)+ \lambda x^k\right)$
		\ENDFOR
	\end{algorithmic}
\end{algorithm} 

We now show that like original Newton's method and {\sf NEWTON-STAR}, {\sf MAX-NEWTON} also converges locally quadratically.

\begin{theorem}[Local quadratic convergence]\label{th:maxnewton}
	Assume that $f_{ij}$ is convex for all $i,j$ and that we know  $\newalpha_{ij}(x^*)$ for all $i,j$. Furthermore, assume that $\mH_{ij}(x^*) \succeq \mu^* \mI$ for some $\mu^*>0$ for all $i,j$ (for instance, this holds if $f_{ij}$ is $\mu^*$-strongly convex). Then for any starting point $x^0 \in \R^d$, the iterates of {\sf MAX-NEWTON} for solving problem \eqref{primal} satisfy the following inequality
	$$
	\|x^{k+1}-x^*\| \leq  \left(\frac{\nu}{2\lambda} \frac{1}{nm}\sum\limits_{i=1}^n\sum\limits_{j=1}^m\left\|a_{ij}\right\|^3\left[\frac{\newalpha_{ij}(x^*)R}{\mu^*\|a_{ij}\|}+ 1\right] \right) \|x^k-x^*\|^2,
	$$
	where $R\eqdef \max\limits_{ij}\|a_{ij}\|$.
\end{theorem}

\begin{proof}
	Since $\mH_{ij}(x^*) \succeq \mu^*\mI$, then $\newalpha_{ij}(x^*)$ is positive for all $i,j$. We estimate Hessian at $x^k$ by $\mH^k = \frac{1}{n}\sum\limits_{i=1}^n\mH_i^k,$ where 
	\begin{eqnarray*}
		\mH_i^k = \frac{1}{m}\max_{j \in [m]}\left\{\frac{\newalpha_{ij}(x^k)}{\newalpha_{ij}(x^*)}\right\}\sum\limits_{j=1}^m\newalpha_{ij}(x^*)a_{ij}a_{ij}^\top = \frac{\beta^k_i}{m}\sum\limits_{j=1}^m\newalpha_{ij}(x^*)a_{ij}a_{ij}^\top. 
	\end{eqnarray*}
	Notice that $h_{ij}(x) \geq 0$, that's why $\beta_i^k \geq 0$ too. In consequence, $\mH_i^k \succeq \mathbf{0}, $and $\mH^k+\lambda\mI  \succeq \lambda\mI$. Then we have
	\begin{eqnarray*}
		\|x^{k+1}-x^*\| & = & \left\|x^k-x^*\left(\mH^k+\lambda\mI\right)^{-1}\left(\nabla f(x^k) + \lambda x^k\right)\right\|\\
		&=&\left\|\left(\mH^k+\lambda\mI\right)^{-1}\left(\left(\mathbf{B}^k+\lambda\mathbf{I}\right)(x^k-x^*)-\nabla f(x^k) - \lambda x^k\right)\right\|\\
		&\leq & \left\|\left(\mH^k+\lambda\mI\right)^{-1}\right\|\left\|\left(\mH^k+\lambda\mI\right)(x^k-x^*)-\nabla f(x^k) - \lambda x^k\right\|\\
		&\leq & \frac{1}{\lambda}\left\|\left(\mH^k+\lambda\mI\right)(x^k-x^*)-\nabla f(x^k) - \lambda x^k\right)\|\\
		&\overset{(\ref{eq:optx})}{=}& \frac{1}{\lambda}\left\|\mH^k(x^k-x^*) - (\nabla f(x^k) -\nabla f(x^*))\right\| \\
		&\leq&\frac{1}{n\lambda}\sum\limits_{i=1}^n\left\|\mH_i^k(x^k-x^*) -(\nabla f_i(x^k) - \nabla f_i(x^*))\right\|\\
		&\leq&\frac{1}{nm\lambda}\sum\limits_{i=1}^n\sum\limits_{j=1}^m\left\|\beta_i^k\newalpha_{ij}(x^*)a_{ij}a_{ij}^\top(x^k-x^*) - \int\limits_0^1\newalpha_{ij}(x^*+\tau(x^k-x^*))a_{ij}a_{ij}^\top(x^k-x^*) d\tau\right\|.
	\end{eqnarray*}
	Finally we obtain
	\begin{equation}
		\label{originalmaxnewton1}
		\|x^{k+1}-x^*\| \leq  \frac{\|x^k-x^*\|}{nm\lambda}\sum\limits_{i=1}^n\sum\limits_{j=1}^m\left\|a_{ij}\right\|^2\newalpha_{ij}(x^*)\left|\beta_i^k-\int\limits_0^1\frac{\newalpha_{ij}(x^*+\tau(x^k-x^*))}{\newalpha_{ij}(x^*)}\right|.
	\end{equation}
	Now we want to upper bound the last term. We know that $\newalpha_{ij}(x)$ is $\nu\|a_{ij}\|$-Lipschitz function, then \begin{eqnarray*}
		\newalpha_{ij}(x^*) - \nu\|a_{ij}\|\tau\|x^k-x^*\| &\leq& \newalpha_{ij}(x^*+\tau(x^k-x^*)) \leq \newalpha_{ij}(x^*)+\nu\|a_{ij}\|\tau\|x^k-x^*\|\\
		1 - \frac{\nu\|a_{ij}\|}{\newalpha_{ij}(x^*)}\tau\|x^k-x^*\| &\leq& \frac{\newalpha_{ij}(x^*+\tau(x^k-x^*))}{\newalpha_{ij}(x^*)} \leq 1 + \frac{\nu\|a_{ij}\|}{\newalpha_{ij}(x^*)}\tau\|x^k-x^*\|.
	\end{eqnarray*} 
	We integrate these two inequalities and obtain
	\begin{eqnarray*}
		1 - \frac{\nu\|a_{ij}\|}{2\newalpha_{ij}(x^*)}\|x^k-x^*\| &\leq& \int\limits_0^1\frac{\newalpha_{ij}(x^*+\tau(x^k-x^*))}{\newalpha_{ij}(x^*)}d\tau \leq 1 + \frac{\nu\|a_{ij}\|}{2\newalpha_{ij}(x^*)}\|x^k-x^*\|.
	\end{eqnarray*}
	Let's denote $j_i^k =\arg \max\limits_{j\in[m]}\left\{\frac{\newalpha_{ij}(x^k)}{\newalpha_{ij}(x^*)}\right\}$, then $\beta_i^k = \frac{\newalpha_{ij^k_i}(x^k)}{\newalpha_{ij^k_i}(x^*)}.$ It means that for $\beta_i^k$ we can write the same two inequalities as above
	\begin{eqnarray*}
		1 - \frac{\nu\|a_{ij_i^k}\|}{2\newalpha_{ij_i^k}(x^*)}\|x^k-x^*\| \leq \beta_i^k \leq 1+ \frac{\nu\|a_{ij_i^k}\|}{2\newalpha_{ij_i^k}(x^*)}\|x^k-x^*\|.
	\end{eqnarray*}
	Taking minimum and maximum in two parts we get
	\begin{eqnarray*}
		1 - \frac{\nu\max\limits_{ij}\|a_{ij}\|}{2\min\limits_{ij}\newalpha_{ij}(x^*)}\|x^k-x^*\| \leq \beta_i^k \leq 1 + \frac{\nu\max\limits_{ij}\|a_{ij}\|}{2\min\limits_{ij}\newalpha_{ij}(x^*)}\|x^k-x^*\|.
	\end{eqnarray*}
	Using inequalities for $\beta_i^k$ and for the integral we obtain
	\begin{eqnarray*}
		\left|\beta_i^k-\newalpha_{ij}(x^k)\right| \leq \frac{\nu\|x^k-x^*\|}{2}\left[\frac{\max\limits_{ij}\|a_{ij}\|}{\min\limits_{ij}\newalpha_{ij}(x^*)}+ \frac{\|a_{ij}\|}{\newalpha_{ij}(x^*)}\right].
	\end{eqnarray*}
	Combining above and (\ref{originalmaxnewton1}) we get
	\begin{eqnarray*}
		\|x^{k+1}-x^*\| &\leq&  \frac{\nu\|x^k-x^*\|^2}{2nm\lambda}\sum\limits_{i=1}^n\sum\limits_{j=1}^m\left\|a_{ij}\right\|^2\newalpha_{ij}(x^*)\left[\frac{\max\limits_{ij}\|a_{ij}\|}{\min\limits_{ij}\newalpha_{ij}(x^*)}+ \frac{\|a_{ij}\|}{\newalpha_{ij}(x^*)}\right]\\
		&\leq& \frac{\nu\|x^k-x^*\|^2}{2nm\lambda}\sum\limits_{i=1}^n\sum\limits_{j=1}^m\left\|a_{ij}\right\|^3\left[\frac{\newalpha_{ij}(x^*)\max\limits_{ij}\|a_{ij}\|}{\mu^*\|a_{ij}\|}+ 1\right],
	\end{eqnarray*}
	where we use inequality $\newalpha_{ij}(x^*) \geq \mu^*$.
\end{proof}

As we can see, the size of convergence area for {\sf NS} is larger than for {\sf MN}, because {\sf NS} depends on $\|a_{ij}\|^3$, while {\sf MN} depends on $\|a_{ij}\|^2\max\limits_{i,j}\|a_{ij}\|$.

\end{document}